\PassOptionsToPackage{dvipsnames}{xcolor}
\documentclass[final]{alt2027halarxiv}

\usepackage{enumitem}
\usepackage{thm-restate}
\usepackage{notation}

\renewcommand{\BlackBox}{\rule{1.5ex}{1.5ex}}  
\ifdefined\proof
    \renewenvironment{proof}{\par\noindent{\bf Proof\ }}{\hfill\BlackBox\\[-2mm]}
\else
    \newenvironment{proof}{\par\noindent{\bf Proof\ }}{\hfill\BlackBox\\[-2mm]}
\fi
\newtheorem{theoremsketch}[theorem]{Theorem sketch}
\newtheorem{assumption}[theorem]{Assumption}
\newtheorem{property}[theorem]{Property}

\newenvironment{proofsketch}{\par\noindent{\bf Proof sketch\ }}{\hfill\BlackBox\\[-2mm]}

\title{A general condition for fast rates in multi-group learning}
\altauthor{%
\Name{Hanqi Sun} \Email{hanqi.sun@universite-paris-saclay.fr}\\
\addr Université Paris-Saclay, CNRS, Laboratoire de mathématiques d'Orsay, 91405, Orsay, France
\AND
\Name{Evgenii Chzhen} \Email{hanqi.sun@universite-paris-saclay.fr}\\
\addr Université Paris-Saclay, CNRS, Laboratoire de mathématiques d'Orsay, 91405, Orsay, France
\AND
\Name{Luca Ganassali} \Email{hanqi.sun@universite-paris-saclay.fr}\\
\addr Université Paris-Saclay, CNRS, Laboratoire de mathématiques d'Orsay, 91405, Orsay, France
\AND
\Name{Gilles Stoltz} \Email{hanqi.sun@universite-paris-saclay.fr}\\
\addr Université Paris-Saclay, CNRS, Laboratoire de mathématiques d'Orsay, 91405, Orsay, France%
}

\editor{}

\begin{document}
\maketitle

\begin{abstract}
In multi-group learning, a single predictor must control its excess risk simultaneously over overlapping groups. \citet{pmlr-v162-tosh22a} obtained $1/\sqrt{n}$ slow rates via a reduction to online learning with sleeping experts, and $1/n$ rates only under realizability (and for a different strategy). We identify a group-wise Bernstein condition under which the same adaptive predictor attains the fast rate $1/n$ without knowing whether the condition holds, while retaining the slow rates on the groups where it fails. The condition is implied by some natural assumptions weaker than realizability. Lower bounds show that weakened versions leave every learner with slow rates. Intermediate rates under relaxed conditions and an extension to VC classes are also obtained.
\end{abstract}


\section{Introduction}
\label{sec:intro}

A prediction model with good overall accuracy may behave poorly on specific subpopulations, as the good global performance can hide bad local behavior. This issue appeared prominently in the literature on algorithmic fairness~\citep{kearns2018preventing, kim2019multiaccuracy, hebert2018multicalibration}, which already contains two features that are central to the present paper: the family of groups may be large and overlapping, and the target is not a single global risk but a collection of conditional guarantees. Multi-group learning is a theoretical framework designed to address this issue: following the formalization of multi-group agnostic PAC learnability by \citet{pmlr-v139-rothblum21a}, the goal is to learn a single predictor whose risk is controlled simultaneously over all groups of a rich collection. The groups can be any meaningful subsets of the data domain that one wants to monitor, e.g., related to social, demographic, or economic characteristics.

In multi-group learning, on each group the statistician must compete with the best benchmark predictor for this group. These per-group comparators typically disagree: the best predictor for one group may perform poorly on another. Yet, the learner must output a single prediction rule, common to all groups. The target rule is therefore generally not a member of the benchmark class. It is often a data-dependent aggregation of benchmark predictors, and the learning problem is necessarily improper.
\citet{pmlr-v162-tosh22a} describes two ways to build such aggregated predictors: decision lists and a reduction to online learning with sleeping experts.
Related objectives, e.g., online multi-group learning~\citep{blum_et_al:LIPIcs.ITCS.2020.55, deng2024group, rittler2023agnostic}, multi-distribution learning~\citep{haghtalab2022demand,zhang2024optimal}, or minimax group fairness~\citep{diana2021minimax}, face the same tension between a single prediction and several data sources.

There is, however, a simpler approach: partition the feature space into regions of fixed group membership, perform an empirical risk minimization on each, and aggregate them into a global predictor. A lower bound in Appendix~\ref{subsec:lower_ERM} shows that this method performs poorly when the number of groups is large, while the reduction of~\citet{pmlr-v162-tosh22a} only pays a logarithmic dependency on this number.

In the classical, single-group setting, the slow rate of order $1/\sqrt{n}$ is unavoidable in general. It can be improved to the fast rate of order $1/n$ under additional stochastic structure, e.g., under the low-noise condition of \citet{MN06} or under Bernstein conditions~\citep{bartlett2006empirical}; see \citet{vanerven2015fast} for a survey.

Whether an analogous improvement is possible in the multi-group setting is a delicate question. So far, fast rates were established by \citet{pmlr-v162-tosh22a} and \citet{pmlr-v313-ardeshir26a} under realizability assumptions, but the agnostic case was left open. Identifying the stochastic structure under which a group-wise slow rate can be turned into a group-wise fast rate is the subject of this paper.

\subsection{The problem}
\label{sec:thepb}

Let $(X,Y) \in \cX \times \cY$ be a random couple with joint distribution denoted by $\Prob^{(X,Y)}$,
the pushforward measure of the ambient probability $\P$ under $(X,Y)$.
Let $\cF$ be a finite benchmark class of predictors $f : \cX \to \cY$, and let $\cG$ be a finite collection of measurable groups $g \subset \cX$; these finiteness assumptions are maintained throughout, the only exception being Appendix~\ref{sec:infinite}, where classes of finite VC dimension are studied. We consider a bounded loss $\ell : \cY \times \cY \to [0,1]$ and assume, with no loss of generality, that $\Prob(X \in g) > 0$ for all $g \in \cG$. For $g \in \cG$, we write
\[
\risk_g(f)=\E\bigl[\ell\bigl(f(X),Y\bigr)\bigm| X\in g\bigr]
\]
for the risk
of a predictor $f : \cX \to \cY$ on group $g$. Introducing $f^\star_g \in \smash{\displaystyle{\argmin_{f\in\cF}\risk_g(f)}}$, a best benchmark predictor on group $g$, we denote by
\[
\cE_g(f)=\risk_g(f)-\risk_g(f^\star_g)\,,
\]
the excess risk of $f : \cX \to \cY$.

Given a sample $D_n = \bigl( (X_1,Y_1),\, \ldots,\,(X_n,Y_n) \bigr)$ of $n \geq 3$ i.i.d.\ copies from $(X,Y)$, our goal is to construct a predictor, possibly randomized and not necessarily in $\cF$, whose excess risk is simultaneously small on every group. Since group $g$ contributes about $n\,\Prob(X\in g)$ observations, we seek guarantees that adapt to this effective sample size: groups of smaller probability are therefore allowed a larger error.
The key difficulty in this setting is that the comparator predictor $f^\star_g$ can and does differ from group to group, while the learned predictor is a single rule.

While rates for $\cE_g(f)$ of order $\sd{\sqrt{\frac{\log\bigl(\card{\cF}\card{\cG}\bigr)}{n\,\Prob(X\in g)}}}$ have already been obtained~\citep{pmlr-v162-tosh22a}, our central question is: \\[-.55cm]
\begin{center}
    Under which conditions can this dependence be improved to $\sd{\frac{\log(\card{\cF}\card{\cG})}{n\,\Prob(X\in g)}}$? \\[.5cm]
\end{center}

\noindent
A partial answer is already presented in~\citet{pmlr-v162-tosh22a}, who claimed this improved rate under a realizability assumption in the binary classification setting: that is, whenever the class $\cF$ contains, for each group, a classifier that is perfect on this group.
They however had to resort to a different algorithm than for obtaining the general slower rates of order $1/\sqrt{n}$.
Our goal is to extend this result to a relaxed version of realizability that we identify and to provide a single algorithm which is adaptive to the assumption or its absence, i.e., which does not need to know whether the realizability assumption holds.

\paragraph{Notation.}
Throughout the document, $\E$ denotes a generic expectation (with respect to all randomness, unless there is a conditioning), and $\ind{\,\cdot\,}$ denotes the indicator function. We denote by $[K] = \{1,\ldots,\,K\}$ the set of the first $K$ positive integers.
For a finite set $A$, we denote
the probability simplex over $A$ by $\Delta(A)$, and the cardinality of $A$ by $\card{A}$. We write $\log(\,\cdot\,)$ to denote the natural logarithm and $\log_2(\,\cdot\,)$
for the base-2 logarithm.

\paragraph{Risks of randomized predictors.}
A randomized predictor is a measurable mapping $\bq : \cX \to \DeltaF$. Given a feature $x \in \cX$, it draws a predictor $F \in \cF$ at random according to the distribution $\bq(x) = \bigl(q(f \mid x)\bigr)_{f \in \cF}$ and outputs the prediction $F(x)$.
The risk of this randomized predictor $\bq$ on group $g$ equals
the expected loss of $\ell\bigl(F(X),Y\bigr)$ over group $g$, which, by the tower rule, rewrites as
\begin{equation}
\label{eq:def_randomized_risk}
\risk_g(\bq) =
\smash{\Exp \Bigl[\ell\bigl(F(X), Y\bigr) \Bigm| X \in g \Bigr]
= \E\Biggl[\sum_{f \in \cF} q(f \mid X)\,\ell\bigl(f(X),Y\bigr) \Biggm| X\in g\Biggr]\,.} \vspace{.15cm}
\end{equation}
The reduction studied in this paper outputs a data-dependent randomized predictor $\widehat{\bq}_n$ of this form, constructed from the sample; that is, $\widehat{\bq}_n: (\cX \times \cY)^{n} \to \bigl(\cX \to \DeltaF\bigr)$ consists of a $D_{n}$--measurable
choice of such a randomized predictor.

\subsection{Outline and contributions}

Section~\ref{sec:reminder} first revisits the strategy of \citet[Section~4]{pmlr-v162-tosh22a}
based on a reduction of the multi-group learning problem to a different learning framework, namely, online adversarial learning with sleeping experts.
We propose a slightly more generic and modular analysis, based on suitably adaptive algorithms for sleeping experts, thus not requiring
to allocate a fraction $n/2$ of the sample data for tuning the online adversarial learning subroutine.
This section achieves the known $1/\sqrt{n}$ bounds on excess risks and merely lays down the foundations for the rest of the work.

\paragraph{Same algorithm for fast and slow rates.}
Section~\ref{sec:FR} states our main results, holding for the same predictor considered in Section~\ref{sec:reminder}.
When the distribution $\P^{(X,Y)}$ and a group $g$ satisfy some natural stochastic conditions discussed below and a more regret-efficient
sleeping expert subroutine is used (satisfying a second-order regret bound as in \citealp{pmlr-v35-gaillard14}),
then the excess risk of that group decreases as $1/n$. For groups $g$ not satisfying this condition,
the previously stated $1/\sqrt{n}$ bound of course remains, as we consider the very same predictor as in Section~\ref{sec:reminder}.
This is in sharp contrast with the work of \citet[Section~5]{pmlr-v162-tosh22a} which required a specifically designed version of their algorithm for their assumption of realizability to ensure the fast $1/n$ rate (see also \citealp{pmlr-v313-ardeshir26a}).

\paragraph{A Master stochastic condition, and natural instantiations.}
We formulate in Definition~\ref{Condition:Bernstein} a Master condition,
of the form of a self-bounding condition on loss differences, to get the mentioned fast rates.
This condition is satisfied under some natural and classic conditions, a local margin condition {\`a} la
Massart (see \citealp{MN06}) or a local risk-gap condition (as in online learning,
see, e.g., \citealp{pmlr-v35-gaillard14}). More precisely, the Master condition holds:
\begin{itemize}[itemsep=1pt, topsep=1pt]
\item In binary classification, when some classifier $f \in \cF$ coincides with the Bayes classifier
$\ind{\eta \geq 1/2}$ over $g$, where $\eta(X)$ denotes the regression function $\Prob(Y=1\,|\,X)$,
and when a margin condition $\bigl|\eta(X)-1/2\bigr|\geq \beta_g > 0$ is a.s.\ satisfied over $\{X \in g\}$---this case includes the group-realizability assumption of~\citet{pmlr-v162-tosh22a};
\item For general losses, when for all $f \in \cF$ differing from $f^\star_g$ on $g$, one has,
for all so-called group-stable subregions $\cR$ of $g$ with positive mass,
\begin{equation}
\label{eq:intro-rg}
\E\Bigl[ \ell\bigl(f(X),Y\bigr) - \ell\bigl(f^\star_g(X),Y\bigr)\,\Bigm|\,X\in \mathcal{R} \Bigr]\geq \delta_g > 0\,;
\end{equation}
Group-stable subregions are regions of $\cX$ where group memberships are the same, see Section~\ref{sec:riskgapcdt}.
\end{itemize}

Besides the theory of fast rates, we further provide in Appendix~\ref{sec:interm-rates} a theory of intermediate rates covering the range $n^{-\alpha}$,
where $\alpha \in [1/2,\, 1]$. This theory builds upon a relaxation of the Master condition which is satisfied, for instance, in classification under Mammen-Tsybakov-type margin conditions.

\paragraph{Necessity of the assumptions through a lower bound.}
A central final contribution to the theory of fast rates is a justification
of the necessity of the conditions stated above.
A precise discussion is offered in Section~\ref{sec:sharp}.
We show, for instance, that requiring gaps~\eqref{eq:intro-rg} on the entire group~$g$ only, rather than on
its subgroups~$\cR$, may lead to inevitable $1/\sqrt{n}$ excess risks; we do so even by imposing other conditions
on the distribution $\P^{(X,Y)}$ that correspond to a Dirac relaxation
of the Master condition of Definition~\ref{Condition:Bernstein}.

\section{Reminder on the reduction by \citet{pmlr-v162-tosh22a}, with $1/\sqrt{n}$~rates}
\label{sec:reminder}

\citet{pmlr-v162-tosh22a} proposed a strategy to reduce the multi-group learning problem to a different learning framework, namely, online learning with sleeping experts. In this section, we describe this reduction and the associated excess-risk guarantees.
To do so, we first briefly recall the setting of sleeping experts, which should be seen as an auxiliary setting.

\subsection{Auxiliary setting: online adversarial learning with sleeping experts}
\label{sec:aux-sleeping}

The sleeping-experts model, also known as specialist-experts model, was introduced by \citealp{Blu97} and further studied by \citealp{freund1997using}. The model is specified by a finite set of experts \(\cA\).
At each round $t \geq 1$, a learner interacts with an opponent as follows.

Simultaneously\footnote{\label{fn:sleeping} We are issuing a slight
restriction compared to the original setting, see Appendix~\ref{app:Sleeping-experts liteerature review}
for more details.}, the opponent chooses
an activity vector $\vec{I}_t=(I_{t,a})_{a\in\cA}\in\{0,1\}^{\card{\cA}}$
and a loss vector $\vec{\ell}_t=(\ell_{t,a})_{a\in\cA}\in[0,1]^{\card{\cA}}$,
while the learner chooses a
distribution $\vec{p}_t\in \Delta(\cA)$ with positive weights on each expert
(based on past values $\vec{I}_\tau$ and $\vec{\ell}_\tau$, where $\tau \leq t-1$).

The coordinate \(I_{t,a}\) indicates whether
expert \(a\) is available at round \(t\): experts with \(I_{t,a}=1\) are active, \(I_{t,a}=0\)
experts are sleeping. We denote by $\cA_t\eqdef\{a\in \cA:I_{t,a}=1\}$ the $t$--stage active expert set, and with no loss of generality,
due to the form of the regret recalled below, we will assume that $\cA_t$ is always not empty.

The learner finally gets to see $\vec{I}_t$ and $\vec{\ell}_t$. The choice of $\vec{p}_t$ is evaluated by first considering
the distribution induced over the active expert set $\cA_t$ by reallocating mass,
\[
\bp_{t|\cA_t} \eqdef \activ(\vec{p}_t,\cA_t) \in \Delta(\cA_t)
\qquad \mbox{defined component-wise by} \qquad
p_{t,a|\cA_t} \eqdef \frac{p_{t,a} I_{t,a}}{\displaystyle{\sum_{a' \in \cA_t} p_{t,a'}}}\,, \vspace{-.8cm}
\]
and then, by suffering the linear loss $\displaystyle{\sum_{a'\in \cA_t} p_{t,a'|\cA_t}\ell_{t,a'}}$.

Since expert $a$ is a meaningful comparator only on the rounds where it is active, the associated cumulative regret at stage $T$
against $a$ is computed only on these active rounds:
\[
\regret_{T,a} \eqdef \sum_{t \in [T]} I_{t,a}\biggl(\,\sum_{a'\in \cA_t}p_{t,a'|\cA_t}\ell_{t,a'}-\ell_{t,a}\biggr).
\]
Appendix~\ref{app:Sleeping-experts liteerature review} provides more background on online adversarial
learning with sleeping experts, including strategies satisfying the following property.

\begin{lemma}
\label{lm:regr-bd}
There exist strategies such that, for some numerical constant
$C \geq 1$, against all opponents, \\[-.5cm]
\[
\quad\qquad \forall T \geq 3, \quad \forall a\in\cA, \qquad\quad \regret_{T,a} \leq C_{T} \smash{\sqrt{\sum_{t \in [T]} I_{t,a}}}\,, \\[.2cm]
\]
where the factors $C_{T}$ satisfy $C_T \leq C\sqrt{\log\bigl(\card{\cA}\bigr) + \ln \ln (T)}$.
\end{lemma}

\subsection{Reduction of multi-group learning to online learning with sleeping experts}
\label{sec:reduc}

We now recall how a regret guarantee for sleeping experts can lead to a multi-group learning guarantee via a slightly simplified\footnote{The simplification
	consists of not having to divide the sample data into two subsets, due to considering the
	data-adaptive bounds of Lemma~\ref{lm:regr-bd}; see comments around Corollary~\ref{cor:simple-reduc}.} version of the reduction from \citet[Section~4]{pmlr-v162-tosh22a}.
A related earlier use of sleeping experts for group-wise guarantees appears in the online subgroup-fairness setting of \citet{blum_et_al:LIPIcs.ITCS.2020.55}.

\paragraph{The reduction itself.}
Starting from an instance of the multi-group learning problem, we first construct a corresponding instance of online learning with sleeping experts, consisting of a set of experts and a sequence of online activity and loss vectors. This construction is carried out in Step 1 below.
We then aggregate the distributions issued to obtain a single, randomized predictor. This is the objective of Step 2, which is called an online-to-batch conversion in the online learning literature (see, e.g., \citealp{CBCG01}).

\emph{Step 1: Constructing an instance of sleeping experts.}
Experts are indexed by $\cA = \cF \times \cG$ and consist of pairs $a = (f,g)$.
Data is ordered in some arbitrary way into a sequence $(X_t,Y_t)_{1 \leq t \leq n}$.
At round $t$, expert $(f,g)$ is active if $X_t \in g$
and if so, outputs the prediction $f(X_t)$; the associated
components of the activity and loss vectors are
\[
I_{t,(f,g)} = \ind{X_t\in g} \qquad \mbox{and} \qquad
\ell_{t,(f,g)} = \ell\bigl(f(X_t),Y_t\bigr) \, .
\]
We consider a sleeping-experts strategy: when processing
data as described above, it returns a sequence
$(\bp_t)_{1 \leq t \leq n}$ of distributions over $\cF \times \cG$, that should then be used to induce distributions over
sets of active experts. To make dependencies clear, we rather write $\bp_t(D_{t-1})$ where $D_{t-1} = (X_\tau,Y_\tau)_{1 \leq \tau \leq t-1}$.

\emph{Step 2: Online-to-batch reduction.}
To output a prediction for a new feature $x \in \cX$,
we first consider the subset $\cG_x = \{ g \in \cG : x \in g\}$
of groups to which $x$ belongs.
The active experts are given by $\cF \times \cG_x$.
We compute the average
\[
\frac{1}{n} \sum_{t=1}^n \bp_{t|\cF \times \cG_x}(D_{t-1})
= \frac{1}{n} \sum_{t=1}^n \activ\bigl(\bp_t(D_{t-1}),\cF \times \cG_x)\,,
\]
which is a distribution over $\cF \times \cG$, and even over $\cF \times \cG_x$ only,
and finally consider the marginal $\hat{\bq}_n(x)$ over $\cF$
of this distribution; it rewrites, for all $f \in \cF$, as
\begin{align}
\label{eq:hat_q_def}
\hat{q}_n(f \mid x) =
 \frac{1}{n}\sum_{t = 1}^n \sum_{g \in \cG} p_{t,(f,g)|\cF \times \cG_x}(D_{t-1}) =
 \frac{1}{n}\sum_{t = 1}^n \sum_{g \in \cG_x} p_{t,(f,g)|\cF \times \cG_x}(D_{t-1})\,.
 \end{align}
The equation above defines the randomized predictor $\hat{\bq}_n$ to be used for multi-group learning.

\paragraph{Excess risk guarantee for the reduction.}
We explain how the regret bound of Lemma~\ref{lm:regr-bd} satisfied by the underlying
online strategy transfers to the excess risk of the randomized predictor $\hat{\bq}_n$ defined in~\eqref{eq:hat_q_def}.
We do so in two steps: Theorem~\ref{tm:reduction risk guarantee} explains the transfer in generality,
while Corollary~\ref{cor:simple-reduc} states the specific guarantee achieved with Lemma~\ref{lm:regr-bd}.

\begin{restatable}[Generic excess risk guarantee]{theorem}{thredriskg}
\label{tm:reduction risk guarantee}
Assume that both the benchmark class $\cF$ and the collection of groups $\cG$ are finite.
Let $\widehat \bq_n$ be the randomized predictor produced by the reduction
described in~\eqref{eq:hat_q_def}. Then, for all $\delta\in(0,1)$, with probability at least
$1-\delta$, \\[-.3cm]
\[
\forall g \in \cG, \qquad\quad
\cE_g(\widehat \bq_n) \leq \frac{1}{n \, \Prob(X \in g)} \regret_{n,(f^{\star}_g,g)}
+ \sqrt{\frac{2 \log\bigl(\card{\cG}/\delta\bigr)}{n \, \Prob(X \in g)}}
+ \frac{\log\bigl(\card{\cG}/\delta\bigr)}{n\, \Prob(X \in g)}\,,
\]
where $\smash{f^\star_g \in \displaystyle{\argmin_{f\in\cF}\risk_g(f)}}$ and
$\regret_{n,(f^\star_g,g)}$ is the sleeping regret against expert~$(f^\star_g,g)$.
\end{restatable}
\ \\[-.8cm]

\begin{proofsketch} Write
$\smash{\displaystyle{n \, \Prob(X \in g) \,\, \cE_g(\hat{\bq}_n) = \sum_{t \in [n]} \E[W_{t, g} \mid D_{t-1}]}}$
and $\regret_{n, (f^\star_g, g)} = \smash{\displaystyle{\sum_{t \in[n]} W_{t, g}}}$, where
\[
W_{t,g} \eqdef \ind{X_t\in g}\biggl(\sum_{f\in \cF}\sum_{g'\in \cG_{X_t}}p_{t,(f,g')|\cF\times \cG_{X_t}}(D_{t-1}) \, \ell \bigl(f(X_t),Y_t\bigr)-  \ell \bigl(f^\star_g(X_t),Y_t\bigr) \biggr)\,,
\]
so that the bound of Theorem~\ref{tm:reduction risk guarantee} reduces to an application of the Bernstein-Freedman inequality
(recalled as Lemma~\ref{lm:constant freedman} in Appendix~\ref{app:conc}), relying on the deterministic bound
\begin{equation}
\label{eq:V-detbd}
V_{n,g}
\smash{\leq \sum_{t \in [n]} \E\bigl[ W_{t,g}^2 \,\big|\, D_{t-1} \bigr]
\leq \sum_{t \in [n]} \E\bigl[ \ind{X_t\in g} \,\big|\, D_{t-1} \bigr] = n\,\P(X \in g),} \\[.2cm]
\end{equation}
where $V_{n,g}$ is the sum of conditional variances of the $W_{t,g}$.
The straightforward details may be found in Appendix~\ref{app:proofs-reduc}.
\end{proofsketch}

A more concrete bound on the excess risk follows by substituting the (general form of) sleeping-experts regret bound given in Lemma~\ref{lm:regr-bd}.
(We recall that we assumed $n \geq 3$.)

\begin{restatable}[Specific excess risk guarantee]{corollarry}{corsimplereduc}
\label{cor:simple-reduc}
    Under the same setting and assumptions as in Theorem~\ref{tm:reduction risk guarantee},
    if the sleeping-experts algorithm satisfies the bound of Lemma~\ref{lm:regr-bd}
    with constant $C\geq 1$, then, for all $\delta\in(0,1)$,
    with probability at least $1-\delta$, for all $g\in \cG,$
    \[
    \cE_g(\hat{\bq}_n) \leq
    3C \, \frac{\sqrt{\log\bigl(2|\cF|\card{\cG}/\delta\bigr) + \log\log n}}{\sqrt{n\, \Prob(X \in g)}} + 3C \, \frac{\log\bigl(2|\cF|\card{\cG}/\delta\bigr) + \sqrt{\log\log n}}{n\, \Prob(X \in g)}\,.
    \]
\end{restatable}

\begin{proofsketch}
The bound of Lemma~\ref{lm:regr-bd} rewrites
$\regret_{n,(f^\star_g,g)} \leq C_n\sqrt{N_g}$,
while the Bernstein-Freedman inequality bounds $N_g$ in terms of its expectation $n\,\P(X \in g)$
up to an additive deviation term of order $\sqrt{n}$.
Combining both bounds and substituting them in the bound of Theorem~\ref{tm:reduction risk guarantee}
concludes the proof. The straightforward details may be found in Appendix~\ref{app:proofs-reduc}.
\end{proofsketch}

\paragraph{Discussion.}
The reduction presented above allows for any sleeping-experts subroutine and is therefore more generic than the one by
\citet[Section~4]{pmlr-v162-tosh22a}, which relied on the sleeping-expert algorithm
by~\citet[Section 6]{blum2007external}. This algorithm further discussed in Appendix~\ref{app:Sleeping-experts liteerature review} requires tuning of per-expert learning parameters, basically as functions of $n \P(X \in g)$. To perform this tuning,
\citet[Section~4]{pmlr-v162-tosh22a} separates the sample into two subsets of equal size, one to determine
the tuning, i.e., to estimate the $\P(X \in g)$, and one to actually perform the reduction.
The guarantees associated with the latter are in terms of $n/2$ instead of $n$ (plus some price for
the correct estimations of probability masses).
We do not require this separation because we consider online strategies that are fully
adaptive; this comes at a slight cost of a $\ln \ln n$ factor in the regret bound
(see Appendix~\ref{app:Sleeping-experts liteerature review} for details), a
mild price compared to dividing the sample size by
$2$ as in \citet[Section~4]{pmlr-v162-tosh22a}.
Finally, the $\sqrt{\log |\cG|}$ dependency on the number of groups in the guarantees above is the correct one: as we show in Appendix~\ref{subsec:lower_log}, it cannot be improved in general.

\section{Fast rates for multi-group learning: upper bounds on excess risk}
\label{sec:FR}

The slow-rate guarantee of Corollary~\ref{cor:simple-reduc} builds upon a so called first-order sleeping expert regret bound (that of Lemma~\ref{lm:regr-bd}), whose leading term scales as the square root of the number of active rounds.
Under some additional assumptions, we are able to obtain a faster rate than $(n \, \Prob(X \in g))^{-1/2}$ for multi-group learning via the reduction to sleeping experts. A first condition is algorithmic: the sleeping-experts algorithm must satisfy a sharper, second-order regret bound in terms of squared excess losses.
The second condition is stochastic: the data-generating process $\Prob$ must satisfy a Bernstein-type condition, that is provide a control on the squared excess
losses in terms of the excess losses.
In Section~\ref{sec:fast rate}, we describe these two conditions as well as the fast rate result that they yield.
Section~\ref{sec:ex-Bern-cdt} then provides examples (under separation conditions, like
margin conditions in classification or risk-gap conditions in more general problems) where the
stochastic condition is met.

\subsection{The two additional conditions and the general fast-rate result}\label{sec:fast rate}

We state the two conditions in a general sleeping expert setting. The condition on the sleeping expert algorithm is a strengthening
of the first-order regret bound of Lemma~\ref{lm:regr-bd}
(as used by \citealp{pmlr-v162-tosh22a})
in the form of a second-order bound taking into account loss differences (lying in $[-1,1]$,
hence the strengthening).
An example of an algorithm satisfying the bound below is
Adapt-ML-Prod by \citet[Algorithm~2 and Corollary~8]{pmlr-v35-gaillard14},
see Appendix~\ref{app:Sleeping-experts liteerature review}.

\begin{lemma}\label{lm:second-order regret}
There exist strategies such that, for some numerical constant
$C \geq 1$, against all opponents,
\[
\forall T \geq 3,\quad\forall a\in \cA, \qquad \regret_{T,a}\leq C_{T}\sqrt{\sum_{t \in [T]} I_{t,a}\biggl(\sum_{a'\in \cA_t} p_{t,a'|\cA_t}\ell_{t,a'}-\ell_{t,a}\biggr)^{\! 2} } + C^2_{T}\,,
\]
where the factors $C_T$ satisfy $C_{T} \leq C \sqrt{\log|\cA| + \ln\ln T}$.
\end{lemma}

We next state the stochastic condition.
To do so, we assume that the loss and activity vectors $(\vec{\ell}_t,\vec{I}_t)_{t \geq 1}$ are drawn i.i.d.\ by
the opponent from a fixed distribution $\nu$. We will state an adaptation of a Bernstein-type condition (also known as a self-bounding condition) on $\nu$.
This is a one-round condition, imposed only on the distribution $\nu$ and not on any underlying algorithm.
We use notation of Section~\ref{sec:aux-sleeping} to state it.

\begin{definition}\label{Condition:Bernstein}
Distribution $\nu$ and expert $a \in \cA$ satisfy the Bernstein condition with constant $\gamma_a>0$ if, for any
weight vector $\bp = (p_{a})_{a \in \cA} \in \Delta(\cA)$,
we have, letting $(\vec{\ell},\vec{I}) \sim \nu$ under $\P$
and introducing the active set $\cA' = \{ a \in \cA : I_a = 1\}$:
    \[
    \E\Biggl[I_{a}\biggl(\sum_{a'\in \cA'}p_{a'|\cA'}\ell_{a'}-\ell_{a}\biggr)^{\!\!2}\Biggr]\leq \gamma_a \,\, \E\Biggr[I_{a} \biggl(\sum_{a'\in \cA'}p_{a'|\cA'}\ell_{a'}-\ell_{a} \biggr) \Biggr].
    \]
\end{definition}

We translate the definition above in the context of the
reduction described in Section~\ref{sec:reduc}.
We recall that $f^\star_g$ denotes an element in $\cF$ with smallest risk over group~$g$,
see Section~\ref{sec:thepb}.

\begin{definition}\label{Condition:group bernstein}
    Distribution $\Prob^{(X,Y)}$ and group $g \in \cG$ satisfy the Bernstein condition with constant $\gamma_g>0$ if, for any weight vector $\bq \in \Delta(\cF\times \cG)$, we have
    \begin{multline*}
        \Exp\!\left[\ind{X\in g}\Biggl(\sum_{(f,g')\in \cF\times \cG_X}q_{(f,g')|\cF\times \cG_X}\ell\bigl(f(X),Y\bigr)
        -\ell\bigl(f^\star_g(X),Y\bigr)\Biggr)^{\!\! 2\,}\right]\\
        \leq \gamma_g \,\, \Exp\!\left[\ind{X\in g}\Biggl(\sum_{(f,g')\in \cF\times \cG_X}q_{(f,g')|\cF\times \cG_X}\ell\bigl(f(X),Y\bigr)
        -\ell\bigl(f^\star_g(X),Y\bigr)\Biggr)\right].
    \end{multline*}
\end{definition}

We may now state and prove our main fast-rate result.

\begin{restatable}{theorem}{thmfastrateregret}\label{tm:fast rate fixed point}
The randomized predictor $\widehat \bq_n$ produced by the reduction
described in~\eqref{eq:hat_q_def} is such that for all distributions $\Prob^{(X,Y)}$
and groups $g$ satisfying the Bernstein condition of Definition~\ref{Condition:group bernstein} with constant $\gamma_g > 0$,
the excess risk for this group $g$ is controlled,
for all $\delta\in(0,1)$, with probability at least~$1-\delta$, by
\[
\cE_g(\hat{\bq}_n) \leq \frac{2}{n \,\Prob(X\in g)} \Biggl( \regret_{n,(f^\star_g,g)} + (8\gamma_g+4) \log\biggl(\frac{\log_2 n}{\delta}\biggr) \Biggr) \,,
\]
where $\regret_{n,(f^\star_g,g)}$ is the sleeping regret against~$(f^\star_g,g)$, as in the reduction of Section~\ref{sec:reduc}.
\end{restatable}

\begin{remark} Theorem~\ref{tm:fast rate fixed point} and Corollary~\ref{cor:main fast rate} below are stated for a fixed group $g$ satisfying the Bernstein condition. For groups that do not satisfy this condition, a fast rate is not guaranteed; nevertheless, the slow-rate guarantee of Corollary~\ref{cor:simple-reduc} continues to hold. \end{remark}

\begin{proofsketch}
We adapt the proof of Theorem~\ref{tm:reduction risk guarantee} and bound the
sum $V_{n,g}$ of conditional variances in a sharper way than in~\eqref{eq:V-detbd}, namely, by
\[
V_{n,g}
\leq \sum_{t \in [n]} \E\bigl[ W_{t,g}^2 \,\big|\, D_{t-1} \bigr]
\leq \sum_{t \in [n]} \gamma_g \,\, \E[ W_{t,g} \,|\, D_{t-1} ]\,,
\]
where the second inequality follows exactly from the Bernstein condition
of Definition~\ref{Condition:group bernstein}. An extension of the Bernstein-Freedman inequality
directly in terms of $V_{n,g}$, and not of a deterministic upper bound thereof,
may then be called to get the inequality
\[
\underbrace{\sum_{t=1}^n \E[W_{t, g} \mid D_{t-1}]}_{= \, n \, \Prob(X \in g) \,\, \cE_g(\hat{\bq}_n)}
\leq \underbrace{\sum_{t=1}^n W_{t, g}}_{= \, \regret_{n, (f^\star_g, g)}}
+ 4 \sqrt{\gamma_g\sum_{t=1}^n \E[W_{t,g}\mid D_{t-1}]\log\left(\frac{\log_2 n}{\delta}\right)}+4\log\left(\frac{\log_2 n}{\delta}\right),
\]
which we may ``solve'' for the quantity $n \, \Prob(X \in g) \,\, \cE_g(\hat{\bq}_n)$, yielding the claimed bound.
The details may be found in Appendix~\ref{app:proofs-FR}.
\end{proofsketch}

Similarly as in the case of ordinary rates
(i.e., similarly as Theorem~\ref{tm:reduction risk guarantee}
entails Corollary~\ref{cor:simple-reduc}),
Theorem~\ref{tm:fast rate fixed point} implies
the following corollary, by substituting a suitable second-order regret bound on the sleeping expert algorithm.

\begin{restatable}{corollarry}{cormainfastrate}
\label{cor:main fast rate}
The randomized predictor $\widehat \bq_n$ produced by the reduction
described in~\eqref{eq:hat_q_def}, based on a
sleeping-experts algorithm satisfying the second-order regret bound of Lemma~\ref{lm:second-order regret}
with constant $C \geq 1$,
is such that for all distributions $\Prob^{(X,Y)}$
and groups $g$ satisfying the Bernstein condition of Definition~\ref{Condition:group bernstein} with constant $\gamma_g > 0$,
the following fast-rate control on the excess risk for this group $g$ is achieved:
for all $\delta\in(0,1)$, with probability at least $1-\delta$,
\[
\cE_g(\hat{\bq}_n) \leq
\frac{ 8\,C^2(\gamma_g + 1)\Bigl(\log\bigl(\card{\cF}\card{\cG}\bigr) + \log\log n\Bigr) +
32(\gamma_g + 1)\log\biggl( \displaystyle{\frac{2\log_2 n}{\delta}} \biggr) }{ n\,\Prob(X\in g) }\,.
\]
\end{restatable}

We thus obtained an upper bound on the excess risk for $g$ uniformly over all distributions satisfying
the Bernstein condition with some parameter $\gamma_g$ and with $\Prob(X\in g)$ larger than some threshold: this observation is important with respect to the
lower bounds presented in Section~\ref{sec:sharp}, which will be of the same form. \medskip

\begin{proofsketch}
We want to substitute the regret bound stemming from Lemma \ref{lm:second-order regret}, which reads
\[
\regret_{n,(f^*_g,g)} \leq C_n \sqrt{\sum_{t=1}^n W_{t,g}^2}+C_n^2\,,
\]
into the bound of Theorem~\ref{tm:fast rate fixed point}. To do so, by the Bernstein-Freedman inequality
and the Bernstein condition, we relate
\[
\sum_{t=1}^n W_{t,g}^2 \quad \mbox{to} \quad
S_{n,g} \eqdef \sum_{t = 1}^n \E \bigl[ W_{t, g}^2 \,\bigm|\, D_{t-1} \bigr] \leq \gamma_g\sum_{t = 1}^n\E\!\bigl[W_{t, g}\,\bigm|\,D_{t-1}\bigr]
= \gamma_g \,\, n\, \Prob(X\in g)\,\, \cE_{g}(\hat{\bq}_n)
\]
up to deviations equal to $\sqrt{S_{n,g}}$ times a logarithmic term.
Thanks to the inequality $2\sqrt{uv} \leq u+v$,
we may further bound these deviations by $S_{n,g}$ plus a logarithmic term.
Therefore, up to multiplicative and additive logarithmic terms,
\[
\regret_{n,(f^*_g,g)} \lesssim C_n \sqrt{2 \gamma_g \,\, n\, \Prob(X\in g)\,\, \cE_{g}(\hat{\bq}_n)} \,,
\]
which we may substitute in the bound Theorem~\ref{tm:fast rate fixed point}
and ``solve'' for $n\, \Prob(X\in g)\,\, \cE_{g}(\hat{\bq}_n)$
to get the claimed bound.
The details may be found in Appendix~\ref{app:proofs-FR}.
\end{proofsketch}

\subsection{Examples where the Bernstein condition of Definition~\ref{Condition:group bernstein} is satisfied}
\label{sec:ex-Bern-cdt}

We now illustrate why the Bernstein condition of Definition~\ref{Condition:group bernstein} is a natural condition:
it is satisfied under assumptions that are classically
considered in the literature.

\paragraph{Group-realizability (\citealp{pmlr-v162-tosh22a}).}
This corresponds to $\ell\bigl(f^\star_g(X),Y\bigr) = 0$ over $\{X \in g\}$ a.s.\ and
entails the Bernstein condition with $\gamma_g = 1$ (due to $x^2 \leq x$ for $x \in [0,1]$).

\subsubsection{Group-wise Bayes--Massart condition for binary classification}
\label{sec:Massart}

In this subsection (and in Section~\ref{sec:combine}) only, we consider the special case of binary classification, for which
the label space is
$\cY=\{0,1\}$ and the loss function is $\ell(y,y')=\ind{y\neq y'}$. A regression function, denoted by $\eta$, is such that $\eta(X)=\Prob(Y=1\,|\,X)$.
We extend the classic notion of a Bayes classifier to the group-wise setting as in the following definition.
Therein, $\Prob^X$ denotes the distribution of $X$, and $f_g$ is defined in an arbitrary way outside~$g$.

\begin{definition}\label{Def:Group-wise Bayes}
    A group-wise Bayes classifier for group $g$ is any function $f_g : \cX \to \{0,1\}$
    such that $f_g = \ind{\eta \geq 1/2}$ over $g$, $\Prob^X$--almost surely.
    If such a classifier $f_g$ exists within the class $\cF$, it is necessarily the oracle $f^\star_g$ for that group $g$.
\end{definition}

We next introduce a group-wise version of the classical Bayes--Massart condition.

\begin{definition}\label{Con:Bayes--Massart}
    We say that group $g \in \cG$ satisfies the Bayes--Massart condition if, first,
    there exists a classifier in $\cF$ that is a group-wise Bayes classifier for $g$,
    and second, there exists $\beta_g>0$ such that
    \[
    \Prob\Bigl(\bigl|\eta(X)-1/2\bigr|\geq \beta_g \,\Bigm|\, X\in g \Bigr)=1\,.
    \]
\end{definition}

The following theorem is the main result of this section.

\begin{restatable}{theorem}{thmBM}
\label{tm:Bayes--Massart implies Bernstein}
    For binary classification,
    if group $g \in \cG$ satisfies the Bayes--Massart condition of Definition~\ref{Con:Bayes--Massart}
    with coefficient $\beta_g > 0$, then it also satisfies the Bernstein condition of Definition~\ref{Condition:group bernstein} with constant $\gamma_g= 1 /(2\beta_g)$.
\end{restatable}

\begin{proofsketch}
For any $\bq \in\Delta(\cF\times\cG)$, on the one hand, by Jensen's inequality and the form of the loss,
\begin{multline*}
\smash{\ind{X\in g} \Biggl(\sum_{(f,g')\in \cF\times \cG_X}q_{(f,g')|\cF\times \cG_X} \, \ell\bigl(f(X),Y\bigr)-\ell\bigl(f^\star_g(X),Y\bigr) \Biggr)^{\! 2}} \\
\leq \ind{X\in g}\,\sum_{(f,g')\in \cF\times \cG_X}q_{(f,g')|\cF\times \cG_X}\Bigl( \underbrace{\ell\bigl(f(X),Y\bigr)-\ell\bigl(f^\star_g(X),Y\bigr)}_{= \, \ind{f(X)\neq f^*_g(X)}} \Bigr)^{\! 2}\,,
\end{multline*}
while, by a classical result in classification (see, e.g., \citealp[Section~2.1]{DGL96}),
\begin{align*}
& \ind{X\in g} \sum_{(f,g')\in \cF\times \cG_X}q_{(f,g')|\cF\times \cG_X} \,\, \underbrace{\Exp\Bigl[\ell\bigl(f(X),Y\bigr)-\ell\bigl(f^\star_g(X),Y\bigr)
       \,\Bigm|\,X\Bigr]} \\ = \ & \ind{X\in g} \sum_{(f,g')\in \cF\times \cG_X}q_{(f,g')|\cF\times \cG_X} \,\, \bigl| 2\eta(X) - 1 \bigr| \,\,
    \indb{f(X)\neq f^\star_g(X)}\,.
\end{align*}
Combining both inequalities together with the lower bound $\bigl| 2\eta(X) - 1 \bigr| \geq 2\beta_g$ on $\{X \in g\}$
concludes. The straightforward details are provided in
Appendix~\ref{app:proofs-FR}.
\end{proofsketch}

\subsubsection{Risk-gap condition on group-stable regions}
\label{sec:riskgapcdt}

Expectation-gap assumptions provide a classical route to fast rates in online learning with experts (see, e.g., \citealp{pmlr-v35-gaillard14} or \citealp{koolen2016combining} and references therein).
We propose below a risk-gap assumption; however, to get fast rates, it is somewhat necessary
(see the lower bound discussed in Section~\ref{sec:sharp} and stated in Appendix~\ref{app:sharp}) to assume risk gaps
not only over groups but over smaller subsets,
regions of $\cX$ where group memberships are the same, which we call group-stable regions.
They are handy because all their elements
induce the same normalizations
in terms of activation functions.

\begin{definition} \label{def:eqv}
We say that $x,x' \in \cX$ are $\cG$--equivalent if
$x \in g \Leftrightarrow  x' \in g$ for all $g \in \cG$.
\end{definition}

\begin{property}
\label{prop:R}
If $x$ and $x'$ are $\cG$--equivalent,
then for all distributions $\bq$ over $\cF \times \cG$, the distributions $\activ(\bq,\cF \times \cG_x)$
and $\activ(\bq,\cF \times \cG_{x'})$ induced over active pairs
are equal.
\end{property}

\begin{property}
The equivalence classes of the $\cG$--equivalence relationship defined above
define a partition $\mathfrak{R}$ of $\cX$. An element
$\mathcal{R} \in \mathfrak{R}$ is called a group-stable region. Note that since $\cG$ is finite, $\mathfrak{R}$ is finite.
We even have that the group-stable regions $\mathcal{R} \in \mathfrak{R}$ such that $\cR \subseteq g$ form a partition of $g$.
\end{property}

We are now ready to state the main
assumption considered in this subsection.

\begin{definition}[Risk gap over group-stable subregions]
\label{Con:primal-risk gap}
    Distribution $\Prob^{(X,Y)}$ and group $g \in \cG$
    satisfy the risk gap condition over group-stable regions if there exists a positive gap $\delta_g>0$ such that, for any $f \in \cF$ that differs from
    $f^\star_g$ on a $\Prob^X$--non-negligible subset of $g$,
    for any group-stable region $\mathcal{R} \in \mathfrak{R}$ such that $\cR \subseteq g$ and $\P(X \in \mathcal{R}) > 0$,
    \[
   \E\Bigl[ \ell\bigl(f(X),Y\bigr) -\ell\bigl(f^\star_g(X),Y\bigr)\,\Bigm|\,X\in \mathcal{R} \Bigr]\geq \delta_g\,.
    \]
\end{definition}

Property~\ref{prop:R} is handy to combine
the lower bounds of Definition~\ref{Con:primal-risk gap}
into a lower bound holding on event $\{ X \in g \}$:
this is exactly how the proof of the following
results works.

\begin{restatable}{theorem}{thriskgapBernst}
\label{tm:primal risk gap implies Bernstein}
    If distribution $\Prob^{(X,Y)}$ and group $g$ satisfy the risk-gap condition of Definition~\ref{Con:primal-risk gap} with~$\delta_g > 0$, they also satisfy the Bernstein condition of Definition~\ref{Condition:group bernstein} with~$\gamma_g=1/\delta_g$.
\end{restatable}

\begin{proofsketch}
We denote by $\mathfrak{R}_g$ the collection of group-stable regions considered in the statement and have,
by a crude bounding and then by the law of total probabilities that
\begin{multline*}
\E\Biggl[ \biggl(\sum_{(f,g')\in \cF\times \cG_X}q_{(f,g')|\cF\times \cG_X}
    \smash{\overbrace{\ell\bigl(f(X),Y\bigr)-\ell\bigl(f^\star_g(X),Y\bigr)}^{\in [-1,1]}} \biggr)^{\!\!2}\,\Biggm|\,X\in g\Biggr] \\
\leq \frac{1}{\P(X\in g)}\sum_{\mathcal{R} \in \mathfrak{R}_g} \P(X\in \cR)\,\,\E\!\!\left[\sum_{(f,g')\in \cF_g\times \cG_X}q_{(f,g')|\cF\times \cG_X} \,\middle|\,X\in \cR\right]
\end{multline*}
By Definition~\ref{def:eqv} and Property~\ref{prop:R}, we may take out the $q_{(f,g')|\cF\times \cG_X}$ from inside
the conditional expectations, and then artificially introduce an upper bound of the form $1/\delta_g$ times
the risk difference considered in Definition~\ref{Con:primal-risk gap}. A final application
of the law of total probabilities concludes. All details are provided in
Appendix~\ref{app:proofs-FR}.
\end{proofsketch}

\begin{remark}[Combining conditions]
\label{rk:combine}
In binary classification, when some group-stable
subregions are controlled through a positive risk-gap condition as in this section, while the remaining
group-stable subregions satisfy a suitable Bayes--Massart margin condition as in Section~\ref{sec:Massart}, then
the Bernstein condition of Definition~\ref{Condition:group bernstein}
is still satisfied; see Appendix~\ref{sec:combine}.
\end{remark}

\section{Multi-group learning: $1/\sqrt{n}$ lower bounds on excess risk}
\label{sec:sharp}

The Master condition considered to get fast rates, namely the Bernstein condition of Definition~\ref{Condition:group bernstein}, is implied by the risk-gap condition over group-stable subregions of
Definition~\ref{Con:primal-risk gap}.
In this section, we show that natural relaxations
of these two assumptions are not enough to get
fast rates. We consider
\begin{itemize}[itemsep=1pt, topsep=1pt]
	\item[(i)] a weaker Bernstein condition \underline{holding only for (Dirac masses at) functions $f \in \cF$ and} \underline{groups $g$} (not on all distributions $\bq$ over $\cF \times \cG$ as needed in Definition~\ref{Condition:group bernstein}),
	\item[(ii)] a weaker risk-gap condition \underline{only over groups} (but not over all group-stable subregions as stated in Definition~\ref{Con:primal-risk gap}).
\end{itemize}
We prove that the combination of (i)+(ii) is insufficient in general to guarantee fast rates. We present a theorem sketch, with formal statement and proof provided in Appendix~\ref{app:sharp}. The notation $\P_{\mathbf{P}}$ below means that the underlying $\P$
induces the distribution $\P^{(X,Y)} = \mathbf{P}$.

\begin{theoremsketch}[Slow-rate lower bound]
\label{tm:sketch:sharpness}
Consider the setting of
binary classification $\cY = \{0,1\}$ with binary loss. There exists a context space $\cX$, a classifier class $\cF$, a group structure $\cG$ and a class $\cP$ of distributions on $\cX \times \cY$
such that the following properties hold.
For each $\mathbf{P} \in \cP$ and every group $g\in\cG$, the pair
$(\mathbf{P},g)$ satisfies both the weaker Bernstein
condition \emph{(i)} with $\gamma_g = 6$,
and the weaker risk-gap condition \emph{(ii)} with gap $\delta_g >0$ on group $g$.
Nevertheless, for any constant $\beta \geq 1$,
\[
\forall n \geq 1, \qquad
    \inf_{\hat f_n}\ \max_{\mathbf{P} \in \cP} \ \P_{\mathbf{P}} \!\left(\exists g \in \cG \ \text{\rm s.t. } \ \cE_{\mathbf{P},g} \bigl(\hat f_n\bigr)
    \geq \frac{1}{14\beta\sqrt{n \,\P_{\mathbf{P}}(X \in g)}} \right) \geq \frac{1}{2} - \frac{0.22}{\beta}\,,
\]
where
$\cE_{\mathbf{P},g}$ denotes the excess risk over group $g$ under when $(X,Y) \sim \mathbf{P}$,
the infimum is taken over all, possibly randomized, classifiers
$\hat{f}_n : \cX \to [0,1]$ that are measurable functions of a sample
$D_n = \bigl((X_1,Y_1),\ldots,(X_n,Y_n)\bigr)$ i.i.d.\ according to $\mathbf{P}$.
\end{theoremsketch}

\section{Extensions}
\label{sec:ext}

\paragraph{Intermediate rates.}
The hard margin condition discussed in Section~\ref{sec:Massart}
may be relaxed into Mammen-Tsybakov margin conditions, leading to
a family of upper bounds on the excess risks of the order $\bigl( n \, \P(X \in g) \bigr)^{-\alpha}$,
where $\alpha \in [1/2,\,1]$; see Appendix~\ref{sec:interm-rates}.

\paragraph{Infinite classes $\cF$ and $\cG$.}
  All previous results were established under the potentially restrictive assumption that $\cF$ and $\cG$ are finite. It is thus natural to wonder how to lift this assumption, even for the slow rates. \citet{pmlr-v162-tosh22a} discuss this issue after their Corollary~9, claiming that finiteness of $\cF$ can be bypassed, for example, by assuming that $\cF$ has a finite VC dimension or some other suitable notion of complexity, while removing the dependency on $|\cG|$ is less straightforward. Actually, we show that, at least in the standard binary classification setup, one can get rid of the finiteness of $\cG$, replacing it by finiteness of its VC dimension.
Details about infinite classes are given in Appendix~\ref{sec:infinite}.

\clearpage
\appendix

\section*{AI Statement}

ChatGPT 5.5 was used to simplify the proofs of Theorem~\ref{tm:fast rate fixed point} and Theorem~\ref{th:31}. GLM 5.3 was used to confirm that the proof of Lemma~\ref{lem:for_recursion} is correct. The authors take full responsibility for the mathematical results and the final manuscript.

\ \\

\section*{Appendices}

\ \\

\paragraph{Important notice.}
The proof sketches in the main body should be
fully sufficient, for readers familiar with the domain,
to follow, validate, and appreciate the results claimed,
except for the lower bound of Theorem~\ref{tm:sketch:sharpness},
for which no such proof sketch fit (see Appendix~\ref{app:sharp} for
its formal statement and proof).
In particular, Appendices~\ref{app:conc}--\ref{app:proofs-reduc}--\ref{app:proofs-FR}
(except perhaps the sub-appendix~\ref{sec:combine} about mixing local Bayes--Massart or local risk-gap conditions)
merely implement the calculations indicated in the main body.

\ \\

\paragraph{Table of contents.}
The supplementary material, whose reading is not essential to evaluating
the submission (as explained above), is organized as follows.
\begin{itemize}[itemsep=3pt, topsep=3pt]
\item Appendix~\ref{app:conc} recalls the Bernstein-Freedman inequality,
which admits two versions, one with a deterministic bound on the sum of
conditional variances, and one involving directly the latter.
\item Appendix~\ref{app:proofs-reduc} states detailed versions of the proofs
of the slow-rate results of Section~\ref{sec:reduc}, implementing the calculations
of all constants.
\item Appendix~\ref{app:proofs-FR} states detailed versions of the proofs
of the fast-rate results of Section~\ref{sec:FR}, implementing the calculations
of all constants.
\item Appendix~\ref{app:LBs} states and proves three lower bounds on excess risks,
the one discussed in Section~\ref{sec:sharp} as well as two other lower bounds
alluded at in the text and that provide more contexts on which type of dependencies
on $\card{\cG}$ are expected or can be achieved in the case of $1/\sqrt{n}$ slow rates.
\item Appendix~\ref{sec:interm-rates} states and proves the intermediate-results
claimed in Section~\ref{sec:ext} as one of the two possible extensions.
\item Appendix~\ref{sec:infinite} does so for the other extension, relative to
multi-group binary classification for infinite classes of benchmark functions $\cF$
and groups $\cG$.
\item Appendix~\ref{app:Sleeping-experts liteerature review} provides
some background and literature review on online adversarial learning with sleeping experts
and, for the sake of completeness and for the sake of slightly improving their regret
bounds, discusses (adaptive) versions of the algorithms by \citet[Section 6]{blum2007external}
and \citealp{pmlr-v35-gaillard14}, corresponding, respectively,
to Lemmas~\ref{lm:regr-bd} and~\ref{lm:second-order regret}.
\end{itemize}

\clearpage
\section{Reminder on the Bernstein-Freedman inequality}
\label{app:conc}

We first restate the simplest form of the inequality,
then apply it to relate the empirical number of times $N_g$
a given group $g$ appears on a data sample of size $n$ to
its expectation $n \, \P(X \in g)$,
and finally state a more empirical version of the
Bernstein-Freedman inequality involving directly the sum of
conditional variances and not a deterministic upper bound thereon.

\begin{lemma}[{original form, see \citealp{freedman1975tail}}]
\label{lm:constant freedman}
Let $X_1,\ldots,X_n$ be a martingale difference sequence
with respect to the filtration $(\cH_0,\cH_1,\ldots,\cH_n)$.
Assume that there exist constants $B$ and $V$ such that,
almost surely,
\[
\forall 1\leq t\leq n, \quad |X_t|\leq B \qquad \mbox{and} \qquad \sum_{t=1}^n\E\bigl[X_t^2\,\big|\,\cH_{t-1}\bigr]\leq V\,.
    \]
Then, for any $\delta \in (0,1)$, with probability at least $1-\delta$,
    \[
    \sum_{t=1}^nX_t \leq
    \sqrt{2 V \log(1/\delta)} + \frac{B}{2} \log(1/\delta)\,.
    \]
\end{lemma}

The Bernstein-Freedman inequality applied to the i.i.d.\ variables
$\ind{X_t \in g}$, together
with $|\cG|$ union bounds, immediately yields the following lemma, used throughout this article
to control the group-wise sample sizes
\[
N_g \eqdef \sum_{k=1}^n \ind{X_k\in g} \qquad \mbox{around their expectations} \qquad n\,\P(X \in g)\,.
\]
We present a bound from above but a symmetric bound from below could also be stated.

\begin{lemma}\label{lm:group mass concentration}
    For all $0<\delta<1$, with
    probability at least $1-\delta$,
    \[
    \forall g\in \cG, \quad \qquad N_g-n\,\Prob(X\in g) \leq
    \ \sqrt{2{n\,\Prob(X \in g)\log\bigl(\card{\cG}/\delta\bigr)}} + \frac{\log\bigl(\card{\cG}/\delta\bigr)}{2} \,.
    \]
\end{lemma}

Finally, we recall a corollary of Lemma~\ref{lm:constant freedman},
proved via a peeling technique, and stating a
deviation inequality adapting to the sum of conditional variances, for which no deterministic bound $V$ is required anymore.

\begin{lemma}[{sharper version, see \citealp[Lemma~2]{bartlett2008high}}]
\label{lm:strong Freedman}
Let $X_1,\ldots,X_n$ be a martingale difference sequence
with respect to the filtration $(\cH_0,\cH_1,\ldots,\cH_n)$.
Assume that there exists a constant $B$ such that,
almost surely,
\[
\forall 1\leq t\leq n, \quad |X_t|\leq B \qquad \mbox{and} \qquad \mbox{denote} \quad V_n = \sum_{t=1}^n\E\bigl[X_t^2\,\big|\,\cH_{t-1}\bigr]\,.
    \]
Then, for any $\delta \in (0,1)$, with probability at least $1-\delta$,
\[
    \sum_{t=1}^n X_t
    \leq
    4
    \sqrt{V_n \log\!\left(\frac{\log_2 n}{\delta}\right)} +2 B \log\left(\frac{\log_2 n}{\delta}\right).
\]
\end{lemma}
\clearpage

\clearpage

\section{Proofs of the slow-rate results of Section~\ref{sec:reduc}}
\label{app:proofs-reduc}

This appendix provides detailed proofs of Theorem~\ref{tm:reduction risk guarantee}
and Corollary~\ref{cor:simple-reduc}, both stating
our simplified view on the slow-rates results by \citet{pmlr-v162-tosh22a}.

\subsection{Proof of Theorem~\ref{tm:reduction risk guarantee}}
\label{sec:tm:reduction risk guarantee}

\thredriskg*

\begin{proof}
    For any $g\in \cG$ and $1 \leq t \leq n$, we introduce
    \begin{equation}\label{def:instantaneous regret}
        W_{t,g}\eqdef \ind{X_t\in g}\Biggl(\sum_{f\in \cF}\sum_{g'\in \cG_{X_t}}p_{t,(f,g')|\cF\times \cG_{X_t}}(D_{t-1}) \,
        \ell \bigl(f(X_t),Y_t\bigr) -\ell \bigl(f^\star_g(X_t),Y_t\bigr) \Biggr) \,.
    \end{equation}
    The rationale behind the introduction of the $W_{t,g}$ is the following, as we prove in Step~2:
    \begin{equation}
    \label{eq:excess_regret0}
    n \, \Prob(X \in g) \,\, \cE_g(\hat{\bq}_n) = \sum_{t = 1}^n \E[W_{t, g} \mid D_{t-1}]\qquad\text{and}\qquad\regret_{n, (f^\star_g, g)} = \sum_{t = 1}^n W_{t, g}\,.
    \end{equation}

    \emph{Step~1.} We apply the Bernstein-Freedman inequality, recalled in Lemma~\ref{lm:constant freedman}
    in Appendix~\ref{app:conc}. Since loss $\ell$ takes values in $[0,1]$,
    the random variables $W_{t,g}$ lie in $[-1,1]$,
    and the martingale increments
    $\Delta_{t,g} \eqdef W_{t,g} - \E[W_{t, g} \mid D_{t-1}]$ satisfy
    \[
    |\Delta_{t,g}| \leq 2 \qquad \mbox{and}
    \qquad V_{n,g} \eqdef \sum_{t=1}^n
    \E\bigl[ \Delta_{t,g}^2 \mid D_{t-1} \bigr] \leq \sum_{t=1}^n \E\bigl[ W_{t,g}^2 \mid D_{t-1} \bigr]
    \leq n \, \P(X \in g)\,.
    \]
    By the Bernstein-Freedman inequality, for each given $g \in \cG$, with probability
    at least $1-\delta$,
    \[
    \sum_{t = 1}^n\bigl(\E[W_{t, g} \mid D_{t-1}] - W_{t, g} \bigr) \leq
    \sqrt{2n \, \Prob(X\in g) \, \log(1/\delta)}+\log(1/\delta)\,.
    \]
    Replacing $\delta$ by
    $\delta/|\cG|$ in the bound above and applying union bound over $\cG$, together with~\eqref{eq:excess_regret0} entails that, with probability $1-\delta$,
    \[
    \forall g \in \cG, \qquad\quad
    n \, \Prob(X \in g) \, \cE_g(\widehat \bq_n) - \regret_{n,(f^{\star}_g,g)}
    \leq
    \sqrt{2 n \, \Prob(X \in g) \log\bigl(\card{\cG}/\delta\bigr)}
    +
    \log\bigl(\card{\cG}/\delta\bigr)\,,
    \]
    yielding the desired inequality. \medskip

    \emph{Step~2.} It remains to prove~\eqref{eq:excess_regret0}.
    The mere definition of the sleeping regret, as introduced in Section~\ref{sec:aux-sleeping},
    justifies the second equality $\regret_{n, (f^\star_g, g)} = \displaystyle{\sum_{t \in [n]} W_{t, g}}$.

    The first equality is justified by the online-to-batch conversion, using that
    $(X_t,Y_t)$ and $(X,Y)$ are both independent from $D_{t-1}$
    and have the same distribution. Thus, on the one hand,
    by definition of $\hat{\bq}_n(x)$ in~\eqref{eq:hat_q_def} and of the risk of a randomized predictor in~\eqref{eq:def_randomized_risk},
    \begin{align*}
    \lefteqn{\Prob(X \in g) \, \risk_g\bigl( \hat{\bq}_n \bigr)} \\
    & = \Exp \! \left[ \ind{X \in g}  \frac{1}{n} \sum_{t=1}^n \Biggl(\sum_{f\in \cF}\sum_{g'\in \cG_{X}}p_{t,(f,g')|\cF\times \cG_{X}}(D_{t-1}) \, \ell \bigl(f(X),Y\bigr)\Biggr) \, \middle|\, D_{n-1} \right] \\
    & = \frac{1}{n} \sum_{t=1}^n
    \E \!\left[ \ind{X\in g}\Biggl(\sum_{f\in \cF}\sum_{g'\in \cG_{X}}p_{t,(f,g')|\cF\times \cG_{X}}(D_{t-1}) \, \ell \bigl(f(X),Y\bigr)\Biggr) \,\Bigg|\, D_{t-1} \right] \\
    & = \frac{1}{n} \sum_{t=1}^n
    \E \!\left[ \ind{X_t\in g}\Biggl(\sum_{f\in \cF}\sum_{g'\in \cG_{X_t}}p_{t,(f,g')|\cF\times \cG_{X_t}}(D_{t-1}) \, \ell \bigl(f(X_t),Y_t\bigr)\Biggr) \,\Bigg|\, D_{t-1} \right] ;
    \end{align*}
    for the second equality, we noted that conditioning with respect to $D_{n-1}$ or $D_{t-1}$ for
    the $t$--th term in the summation was the same, as no variable $(X_t,Y_t), \ldots, (X_n,Y_n)$ is involved,
    while for the third equality, we replaced $(X,Y)$ by $(X_t,Y_t)$ by the mentioned property for
    online-to-batch conversion.
    On the other hand, similar but simpler arguments justify that
    \[
    \Prob(X \in g) \, \risk_g(f^\star_g)
    = \Exp\Bigl[ \ind{X\in g} \, \ell \bigl(f^\star_g(X),Y\bigr)\Bigr]
    = \E\Bigl[ \ind{X_t\in g} \, \ell \bigl(f^\star_g(X_t),Y_t\bigr) \,\Big|\, D_{t-1} \Bigr]\,.
    \]
    Combining both equalities indeed leads to the first equality stated in~\eqref{eq:excess_regret0}.
\end{proof}

\subsection{Proof of Corollary~\ref{cor:simple-reduc}}

\corsimplereduc*

\begin{proof}
    We recall that we assumed $n \geq 3$.
    The bound of Lemma~\ref{lm:regr-bd} and the reduction of Section~\ref{sec:reduc} exactly have that for each $g \in \cG$,
    \[
    \regret_{n,(f^\star_g,g)} \leq C_{n} \smash{\sqrt{N_g}}\,, \qquad \mbox{where} \qquad N_g \eqdef \sum_{k=1}^n \ind{X_k\in g}
    \]
    and $C_n = C\sqrt{\log(|\cF||\cG|) + \log\log n}$ for some numerical $C \geq 1$.
    Now, the deviation inequality of Lemma~\ref{lm:group mass concentration}
    together with the subadditivity of the square root
    ensures that with probability at least $1-\delta/2$, for all $g \in \cG$,
    \begin{align*}
    \sqrt{N_g} & \leq \sqrt{n\,\Prob(X\in g)}
    + \Bigl( 2n\,\Prob(X\in g) \log\bigl(2\card{\cG}/\delta\bigr) \Bigr)^{1/4} + \sqrt{\log\bigl(2\card{\cG}/\delta\bigr)/2} \\
    & \leq \bigl(\underbrace{1+1/2}_{= 1.5}\bigr) \sqrt{n\,\Prob(X\in g)} + \bigl(\underbrace{\sqrt{2}/2 + 1/\sqrt{2}}_{\leq 1.5} \bigr) \sqrt{\log\bigl(2\card{\cG}/\delta\bigr)}\,,
    \end{align*}
    where we used $uv \leq u^2/2 + v^2/2$ in the middle term of the first inequality to get the second inequality.
    We substitute these inequalities into the bound of
    Theorem~\ref{tm:reduction risk guarantee}, used
    with $\delta/2$ (instead of $\delta$), and apply a union bound:
    with probability at least $1-\delta$, for all $g \in \cG$,
    \begin{multline*}
    \cE_g(\hat{\bq}_n)
    \leq \frac{C_{n} \smash{\sqrt{N_g}}}{n \, \Prob(X \in g)}
    + \sqrt{\frac{2 \log\bigl(2\card{\cG}/\delta\bigr)}{n \, \Prob(X \in g)}}
    + \frac{\log\bigl(2\card{\cG}/\delta\bigr)}{n\, \Prob(X \in g)}\\
    \leq
        1.5\,\frac{C_{n}}{\sqrt{n \, \Prob(X \in g)}} + 1.5\,\frac{C_{n}\sqrt{\log\bigl(2\card{\cG}/\delta\bigr)}}{n \, \Prob(X \in g)}
        + \sqrt{\frac{2 \log\bigl(2\card{\cG}/\delta\bigr)}{n \, \Prob(X \in g)}}
+ \frac{\log\bigl(2\card{\cG}/\delta\bigr)}{n\, \Prob(X \in g)}\,.
    \end{multline*}
    We then substitute the bounds guaranteed by Lemma~\ref{lm:regr-bd}
    and perform some rough bounding of logarithms to get
    \begin{multline*}
    \cE_g(\hat{\bq}_n) \leq
        \bigl( 1.5\,C + \sqrt{2} \bigr) \frac{\sqrt{\log\bigl(2|\cF|\card{\cG}/\delta\bigr)+\log\log n}}{\sqrt{n\, \Prob(X \in g)}} \\
        + \bigl(1.5C + 1\bigr)
        \frac{\log\bigl(2|\cF|\card{\cG}/\delta\bigr) + \sqrt{\log\log n}}{n\, \Prob(X \in g)}\,.
    \end{multline*}
    The claimed bound follows from further simplifying the expression
    thanks to $C \geq 1$.
\end{proof}

\clearpage
\section{Proofs of the fast-rate results of Section~\ref{sec:FR}}
\label{app:proofs-FR}

This long appendix proves the core results of this article:
\begin{itemize}
\item first,
the general theory of fast rates for multi-group learning
based on the Bernstein condition
stated in Definition~\ref{Condition:group bernstein}, namely,
Theorem~\ref{tm:fast rate fixed point} in Appendix~\ref{app:tm:fast rate fixed point}
and Corollary~\ref{cor:main fast rate} in Appendix~\ref{app:cor:main fast rate};
\item second, that typical conditions, adapted to our setting, like margin conditions
or risk-gap conditions, ensure that the Bernstein condition is satisfied;
see, respectively, the proofs of Theorem~\ref{tm:Bayes--Massart implies Bernstein}
in Appendix~\ref{app:tm:Bayes--Massart implies Bernstein} for margin conditions
and of Theorem~\ref{tm:primal risk gap implies Bernstein} in Appendix~\ref{app:tm:primal risk gap implies Bernstein}
for risk-gap conditions;
\item third and finally, that the previous conditions may be combined, as alluded
at in Remark~\ref{rk:combine} and detailed in Appendix~\ref{sec:combine}.
\end{itemize}

\subsection{Proof of Theorem~\ref{tm:fast rate fixed point}}
\label{app:tm:fast rate fixed point}

\thmfastrateregret*

\begin{proof}
We adapt the proof of Theorem~\ref{tm:reduction risk guarantee} (located in Appendix~\ref{sec:tm:reduction risk guarantee})
and consider the same notation as therein---in particular, the sequence $W_{t,g}$ as in \eqref{def:instantaneous regret},
the associated martingale increments $\Delta_{t,g} = W_{t,g} - \E[W_{t, g} \mid D_{t-1}]$,
and the sum $V_{n,g}$ of their conditional variances.

We control the sum of the $W_{t,g}$ via a refined version of the Bernstein-Freedman inequality,
stated in Lemma~\ref{lm:strong Freedman} of Appendix~\ref{app:conc}, still providing deviations in terms of
\[
V_{n,g} \eqdef \sum_{t=1}^n \E\Bigl[\bigl(W_{t,g}-\E[W_{t,g} \,\big|\, D_{t-1}]\bigr)^{2} \,\Big|\, D_{t-1} \Bigr] \leq
\sum_{t=1}^n \E\bigl[ W_{t,g}^2 \,\big|\, D_{t-1} \bigr]\,,
\]
which we bound in a sharper, data-dependent, way compared to the previous $n \, \P(X \in g)$ bound.
Conditioning by $D_{t-1}$ fixes the probability $\bp_t(D_{t-1})$ over $\cF\times \cG$,
so that the conditional expectations below correspond to integrating over $(X_t,Y_t)$; we do so by the Bernstein condition
stated in Definition~\ref{Condition:group bernstein}:
\begin{align}
\lefteqn{\E\bigl[W_{t,g}^2 \,\big|\, D_{t-1}\bigr]} \nonumber \\
& = \E \!\! \left[ \ind{X_t\in g} \Biggl( \sum_{(f,g')\in \cF\times\cG_{X_t}} p_{t,(f,g')\mid \cF\times\cG_{X_t}}(D_{t-1}) \,
\ell\bigl(f(X_t),Y_t\bigr) - \ell\bigl(f_g^\star(X_t),Y_t\bigr) \Biggr)^{\! 2} \,\Biggm|\, D_{t-1} \right] \notag \\
& \leq \gamma_g\,\, \E\!\!\left[ \ind{X_t\in g} \Biggl( \sum_{(f,g')\in \cF\times\cG_{X_t}} p_{t,(f,g') \mid \cF\times\cG_{X_t}}(D_{t-1}) \,
\ell\bigl(f(X_t),Y_t\bigr) - \ell\bigl(f_g^\star(X_t),Y_t\bigr) \Biggr) \,\Biggm|\, D_{t-1} \right] \notag \\
&= \gamma_g \,\, \E[W_{t,g}\mid D_{t-1}]. \label{eq:filtered bernstein}
\end{align}
Summing over $t$, we get $V_{n,g} \leq \displaystyle{\gamma_g\sum_{t=1}^n\E\bigl[W_{t,g}\bigm|D_{t-1}\bigr]}$.

The martingale increments $\Delta_{t,g}$ still lie in $[-2,2]$
so that Lemma~\ref{lm:strong Freedman} guarantees that with probability at least $1-\delta$,
\begin{equation}
\label{eq:EWFreed}
\sum_{t=1}^n \bigl(\E[W_{t,g}\mid D_{t-1}]-W_{t,g}\bigr) \leq
4 \sqrt{\gamma_g\sum_{t=1}^n \E[W_{t,g}\mid D_{t-1}]\log\left(\frac{\log_2 n}{\delta}\right)}+4\log\left(\frac{\log_2 n}{\delta}\right).
\end{equation}
We now bound the first term in the right-hand side using $4\sqrt{uv} = 2\sqrt{(u/2)\,8v}\leq u/2 + 8v$
and solve an inequation to get
\[
\frac{1}{2} \sum_{t=1}^n \E[W_{t,g}\mid D_{t-1}]
\leq \sum_{t=1}^n W_{t,g} + 8\gamma_g \log\!\left(\frac{\log_2 n}{\delta}\right) +
4\log\!\left(\frac{\log_2 n}{\delta}\right) .
\]
Substituting the identities~\eqref{eq:excess_regret0}
yields the claimed bound.
\end{proof}

\subsection{Proof of Corollary~\ref{cor:main fast rate}}
\label{app:cor:main fast rate}

\cormainfastrate*

\begin{proof}
    Fix $\delta \in (0,1)$. We use the notation of Theorem \ref{tm:fast rate fixed point}.
    We define an event $\cE_g$ under which the inequality in Theorem \ref{tm:fast rate fixed point} holds for group $g$ and a
    given confidence level $1-\delta$.

    We restate the explicit regret bound of Lemma \ref{lm:second-order regret} with the $W_{t,g}$ piece of notation:
    \begin{equation}\label{eq:regret in main result of fast rate 0}
    \regret_{n,(f^*_g,g)} \leq C_n \sqrt{\sum_{t=1}^n W_{t,g}^2}+C^2_n\,,
    \end{equation}
    where $C_n \leq C\sqrt{\log\bigl(\card{\cF}\card{\cG}\bigr) + \log \log n}$ for some universal constant $C \geq 1$.

    We will first prove a high-probability bound for the regret $\regret_{n,(f^*_g,g)}$, then, under the event $\cE_g$, substitute it in the bound of Theorem~\ref{tm:fast rate fixed point}. The issue with the bound~\eqref{eq:regret in main result of fast rate 0} is
    that we do not have any direct control over the fully empirical sum of $W_{t,g}^2$, and would rather have the sum of conditional expectations thereof. Yet another application of the Bernstein-Freedman inequality deals with that.
    \medskip

    \noindent\emph{Concentration of the sum of $W_{t,g}^2$.} We use again the version of the Bernstesin-Freedman
    inequality stated as Lemma~\ref{lm:strong Freedman}, with the martingale increments $W_{t,g}^2-\E\bigl[W_{t,g}^2\bigm|D_{t-1}\bigr]$,
    lying in the range $[-2,2]$.
    Their conditional variances are bounded by the conditional second moments, which we futher bound using $|W_{t,g}| \leq 1$, to get a sum itself bounded as in~\eqref{eq:filtered bernstein}:
    \begin{align*}
     \sum_{t = 1}^n \E \biggl[ \Bigl( W_{t, g}^2-\E\bigl[W_{t,g}^2\,\bigm|\,D_{t-1}\bigr] \Bigr)^2
     \,\biggm|\, D_{t-1} \biggr]
     & \leq
     \sum_{t = 1}^n \E \bigl[ W_{t, g}^4 \,\bigm|\, D_{t-1} \bigr] \\
     & \leq
     \sum_{t = 1}^n \E \bigl[ W_{t, g}^2 \,\bigm|\, D_{t-1} \bigr]
     \leq
     \gamma_g\sum_{t = 1}^n\E\!\bigl[W_{t, g}\,\bigm|\,D_{t-1}\bigr]\,.
    \end{align*}
Therefore, by Lemma~\ref{lm:strong Freedman}, with probability at least $1-\delta$:
    \[
    \sum_{t=1}^n \bigl(W_{t,g}^2-\E[W_{t,g}^2 \,\bigm|\,D_{t-1}]\bigr)\leq 4\sqrt{\gamma_g \log\biggl(\frac{\log_2 n}{\delta}\biggr)\sum_{t=1}^n \E [W_{t,g}\,|\,D_{t-1}]}+4\log\biggl(\frac{\log_2 n}{\delta}\biggr)\,.
    \]
    We rearrange this inequality and use again~\eqref{eq:filtered bernstein} to get
    \[
            \sum_{t=1}^n W_{t,g}^{2}\leq \gamma_g
            \underbrace{\sum_{t=1}^n\E[W_{t,g}\,|\,D_{t-1}]}_{= n\, \Prob(X\in g)\,\, \cE_{g}(\hat{\bq}_n)}
            +4\sqrt{\gamma_g \log\biggl(\frac{\log_2 n}{\delta}\biggr)\sum_{t=1}^n\,\E[W_{t,g}\,|\,D_{t-1}]}+4\log\biggl(\frac{\log_2 n}{\delta}\biggr)\,,
    \]
where we used the identity \eqref{eq:excess_regret0} in the last substitution.
By the inequality $2uv\leq u^2+v^2$, we may further simplify the bound into
\begin{align} \sum_{t=1}^n W_{t,g}^2 \leq 2\gamma_g \,\, n\,\Prob(X\in g)\,\,\cE_g(\hat{\bq}_n) +8\log\biggl(\frac{\log_2 n}{\delta}\biggr)\,. \label{eq:squared term in main fast rate} \end{align}
\medskip

\noindent\emph{Substitution into the bound of Theorem~\ref{tm:fast rate fixed point}.}
Substituting~\eqref{eq:squared term in main fast rate}
into~\eqref{eq:regret in main result of fast rate 0}, we proved
so far that with probability at least $1-\delta$,
\begin{align}\label{eq:regret in main result of fast rate 1}
\regret_{n,(f^*_g,g)} \leq C_n\sqrt{2\gamma_g \,\, n\,\Prob(X\in g)\,\,\cE_g(\hat{\bq}_n) +8\log\biggl(\frac{\log_2 n}{\delta}\biggr)} +C^2_n\,. \end{align}
Keeping in mind that the bound of Theorem~\ref{tm:fast rate fixed point} holds on the event $\cE_g$,
we thus get, by a union bound, the subadditivity of the square root,
and by multiplying both sides by $n\,\Prob(X\in g)$,
that with probability at least $1-2\delta$,
\begin{align*}
n\,\Prob(X\in g)\,\,\cE_g(\hat{\bq}_n)
\leq & \ 2 \Biggl( \regret_{n,(f^\star_g,g)} + (8\gamma_g+4) \log\biggl(\frac{\log_2 n}{\delta}\biggr) \Biggr) \\
\leq & \ \ \ 2\, \overbrace{C_n \sqrt{ 2\gamma_g}}^{=a} \sqrt{n\,\Prob(X\in g)\,\,\cE_g(\hat{\bq}_n)} \\
& +
\underbrace{2 C_n \sqrt{ 8 \log\biggl(\frac{\log_2 n}{\delta}\biggr)} +
2C^2_n + (16\gamma_g+8) \log\biggl(\frac{\log_2 n}{\delta}\biggr)}_{=b} \\
\leq & \ \frac{n\,\Prob(X\in g)\,\,\cE_g(\hat{\bq}_n)}{2} + 2a^2 + b\,,
\end{align*}
where the penultimate inequality
is of the same form as the one discussed
at the end of Theorem~\ref{tm:fast rate fixed point},
and to which we applied the same $2\sqrt{uv} \leq 2u + v/2$ treatment
to obtain the final inequality.
Bounding logarithms
in a crude way by the same factor and rearranging,
we proved so far that with probability at least $1-2\delta$,
\begin{align*}
    n\,\Prob(X\in g) \,\, \cE_g(\hat{\bq}_n)
& \leq 4a^2 + 2b \\
& = 8\gamma_g C_n^2 + 4\,C_n^2 + (32\gamma_g+16) \log\biggl(\frac{\log_2 n}{\delta}\biggr)
+ \smash{\overbrace{4 C_n \sqrt{ 8 \log\biggl(\frac{\log_2 n}{\delta}\biggr)}}^{\leq 4\,C_n^2 + 8 \log ...}} \\
& \leq 8(\gamma_g+1)C_n^2 + 32(\gamma_g + 1)\log\biggl( \frac{\log_2 n}{\delta}\biggr)\,.
\end{align*}
We substitute $C_n^2 \leq C^2 \Bigl(\log\bigl(\card{\cF}\card{\cG}\bigr) + \log\log n\Bigr)$,
and replace $\delta$ by $\delta/2$ to conclude.
\end{proof}

\subsection{Proof of Theorem~\ref{tm:Bayes--Massart implies Bernstein}}
\label{app:tm:Bayes--Massart implies Bernstein}

\thmBM*

\begin{proof}
    Let $f^*_g$ denote a group-wise Bayes classifier in $\cF$, whose existence is guaranteed by Bayes--Massart condition of Definition~\ref{Con:Bayes--Massart}.
    We recall some basic facts in binary classification, to be used repeatedly in the proof: first, since $\ell\bigl(f(X),Y\bigr) = \ind{f(X) \neq Y}$,
    \begin{equation}
    \label{eq:rem-class-1}
    \Bigl( \ell\bigl(f(X),Y\bigr)-\ell\bigl(f^\star_g(X),Y\bigr) \Bigr)^2
    = \bigl( \ind{f(X)\ne Y} - \indb{f^\star_g(X)\ne Y} \bigr)^2
    = \indb{f(X)\neq f^*_g(X)}\,,
    \end{equation}
    and second (see \citealp[Section~2.1]{DGL96}),
    \begin{align}
    \nonumber
    \E \Bigl[ \ell\bigl(f(X),Y\bigr)-\ell\bigl(f^\star_g(X),Y\bigr) \,\Bigm|\, X \Bigr]
    & = \E \bigl[ \ind{f(X)\ne Y} \,\bigm|\, X \bigr] - \E \bigl[ \ind{f^\star_g(X)\ne Y} \,\bigm|\, X \bigr] \\
    \nonumber
    & = \bigl( 2\eta(X) - 1 \bigr) \Bigl( \indb{f^\star_g(X)=1} - \ind{f(X)=1} \Bigr)\,,
    \end{align}
    so that
    \begin{equation}
    \label{eq:rem-class-2}
    \E \Bigl[ \ell\bigl(f(X),Y\bigr)-\ell\bigl(f^\star_g(X),Y\bigr) \,\Bigm|\, X \Bigr]
    = \bigl| 2\eta(X) - 1 \bigr| \,\,
    \indb{f(X)\neq f^\star_g(X)} \qquad \mbox{ on } \{X \in g\}\,.
    \end{equation}
   Fix a weight vector $\bq\in\Delta(\cF\times\cG)$.
   By Jensen's inequality and \eqref{eq:rem-class-1},
    \begin{align*}
    \ind{X\in g} & \Biggl(\sum_{(f,g')\in \cF\times \cG_X}q_{(f,g')|\cF\times \cG_X} \, \ell\bigl(f(X),Y\bigr)-\ell\bigl(f^\star_g(X),Y\bigr) \Biggr)^{\! 2}\\
     & \leq \ind{X\in g}\,\sum_{(f,g')\in \cF\times \cG_X}q_{(f,g')|\cF\times \cG_X}\Bigl( \ell\bigl(f(X),Y\bigr)-\ell\bigl(f^\star_g(X),Y\bigr) \Bigr)^{\! 2} \\
     & = \ind{X\in g}\,\sum_{(f,g')\in \cF\times \cG_X}q_{(f,g')|\cF\times \cG_X}\,
     \ind{f(X)\neq f^*_g(X)} \eqdef \rho_{\bq,g}\,.
    \end{align*}
    On the other hand, by~\eqref{eq:rem-class-2},
    \begin{multline*}
       \E\!\left[\ind{X\in g}\Biggl(\sum_{(f,g')\in \cF\times \cG_X}q_{(f,g')|\cF\times \cG_X}\ell\bigl(f(X),Y\bigr)-\ell\bigl(f^\star_g(X),Y\bigr)\Biggr)
       \,\,\middle|\,\, X \right]\\
       = \ind{X\in g} \sum_{(f,g')\in \cF\times \cG_X}q_{(f,g')|\cF\times \cG_X} \,\, \Exp\Bigl[\ell\bigl(f(X),Y\bigr)-\ell\bigl(f^\star_g(X),Y\bigr)
       \,\Bigm|\,X\Bigr] =|2\eta(X)-1|\,\rho_{\bq,g}\,.
    \end{multline*}
The tower rule and the fact that
by assumption, $|2\eta(X)-1| \geq 2\beta_g$ on $\{X\in g\}$, conclude the proof.
\end{proof}

\vspace{4cm}

\subsection{Proof of Theorem~\ref{tm:primal risk gap implies Bernstein}}
\label{app:tm:primal risk gap implies Bernstein}

\thriskgapBernst*

\begin{proof}\textbf{(of Theorem~\ref{tm:primal risk gap implies Bernstein})}
We denote by
$\mathfrak{R}_g$ the subset of elements $\mathcal{R} \in \mathfrak{R}$ such that $\cR\subseteq g$ and $\P(X \in \mathcal{R}) > 0$,
and by $\cF_g$ the subset of functions $f \in \cF$
differing from $f^\star_g$ on a $\Prob^X$--non-negligible subset of~$g$.
We start with some crude bounding (taking into account that for functions not in $\cF_g$, the marked differences are null)
and then decompose the expectation in terms of the regions in $\mathfrak{R}_g$: for all distributions $\bq$ over $\cF \times \cG$,
\[
    \Biggl(\sum_{(f,g')\in \cF\times \cG_X}q_{(f,g')|\cF\times \cG_X}
    \overbrace{\ell\bigl(f(X),Y\bigr)-\ell\bigl(f^\star_g(X),Y\bigr)}^{\in [-1,1]} \Biggr)^{\!\!2}
    \leq \sum_{(f,g')\in \cF_g\times \cG_X}q_{(f,g')|\cF\times \cG_X}
\]
then
\begin{multline*}
    \E\!\!\left[\Biggl(\sum_{(f,g')\in \cF\times \cG_X}q_{(f,g')|\cF\times \cG_X}
    \ell\bigl(f(X),Y\bigr)-\ell\bigl(f^\star_g(X),Y\bigr) \Biggr)^{\!\!2}\,\middle|\,X\in g\right] \\
    \leq \frac{1}{\P(X\in g)}\sum_{\mathcal{R} \in \mathfrak{R}_g} \P(X\in \cR)\,\,\E\!\!\left[\sum_{(f,g')\in \cF_g\times \cG_X}q_{(f,g')|\cF\times \cG_X} \,\middle|\,X\in \cR\right].
\end{multline*}
Note that by definition of group-stable regions,
for each $\cR \in \mathfrak{R}$, the groups
$\cG_x$ are the same for all $x \in \cR$: let $\cG(\cR)$ be this common group.
Continuing from the last display and using the condition of Definition~\ref{Con:primal-risk gap},
\begin{align*}
    &\E\!\!\left[\Biggl(\sum_{(f,g')\in \cF\times \cG_X}q_{(f,g')|\cF\times \cG_X}\ell\bigl(f(X),Y\bigr)-\ell\bigl(f^\star_g(X),Y\bigr)
    \Biggr)^{\!\!2}\,\middle|\,X\in g\right]\\
    & \leq \frac{1}{\P(X\in g)}\sum_{\mathcal{R} \in \mathfrak{R}_g} \P(X\in \cR)\,\,\sum_{(f,g')\in \cF_g\times \cG(\cR)}q_{(f,g')|\cF\times \cG(\cR)} \\
    &\leq
    \frac{1}{\P(X\in g)}\sum_{\mathcal{R} \in \mathfrak{R}_g} \P(X\in \cR) \! \sum_{(f,g')\in \cF_g\times \cG(\cR)} \!\!q_{(f,g')|\cF\times \cG(\cR)} \frac{\E\Bigl[\ell\bigl(f(X),Y\bigr)-\ell\bigl(f^\star_g(X),Y\bigr) \Bigm| X\in \cR\Bigr]}{\delta_g}\\
    &=
     \frac{1}{\delta_g} \,\, \E\!\!\left[\sum_{(f,g')\in \cF\times \cG_X}q_{(f,g')|\cF\times \cG_X}\ell\bigl(f(X),Y\bigr)-\ell\bigl(f^\star_g(X),Y\bigr)\,\middle|\,X\in g\right],
\end{align*}
where to get the last equality we performed similar algebraic manipulations (in reverse order) as
in the rest of the proof.
\end{proof}

\subsection{Mixing conditions: local Bayes--Massart or local risk-gap conditions}
\label{sec:combine}

We now show that the group-wise Bernstein condition can be obtained by
combining two different local sufficient conditions across group-stable
regions.

\subsubsection{With a global Bayes assumption}

Throughout this subsection, we fix a group $g\in\cG$ and assume that
there exists a group-wise Bayes classifier
$f_g^\star \in \cF$. In particular, the same classifier
$f_g^*$ is Bayes-optimal on every group-stable region of $g$.
This requirement, formally stated as
Assumption~\ref{ass:gr-Bayes} below,
may be relaxed; see the end of this subsection for a discussion.

The conditions of Definition~\ref{Con:Bayes--Massart} and
Definition~\ref{Con:primal-risk gap} provide two different mechanisms
for establishing a Bernstein inequality. The former controls the noise
around the Bayes decision boundary through the regression function
$\eta$, whereas the latter directly controls the separation in
conditional risk between the group oracle and the competing benchmarks
in $\cF$. In particular, the risk-gap argument does not require its
oracle to be Bayes-optimal, while the Bayes-Massart argument relies
essentially on Bayes optimality. Thus, the two conditions capture
different structural properties of the learning problem.

Also (but this is an ancillary remark), the two mechanisms do not necessarily provide the same control on every
group-stable region. Some regions may admit a favorable Massart margin,
whereas on other regions a direct risk gap may provide the more useful
bound. When both
conditions are available on the same region, we show that we may actually use the sharper of the
two resulting Bernstein constants.

We first introduce the corresponding local versions of the two
conditions, i.e., versions
that hold for group-stable subregions $\cR \subseteq g$ and not necessarily over the entire group $g$. Definition~\ref{Con:Bayes--Massart}
is now cut into two separate conditions,
Assumption~\ref{ass:gr-Bayes} and
Definition~\ref{Con:Group-stable-region-wise Bayes--Massart}.

\begin{assumption}
\label{ass:gr-Bayes}
There exists a classifier in $\cF$,
denoted by $f^\star_g$, that is a group-wise Bayes classifier for $g$.
\end{assumption}

We recall (cf.\ proof of
Theorem~\ref{tm:primal risk gap implies Bernstein}) that we denote by
$\mathfrak{R}_g$ the subset of elements $\mathcal{R} \in \mathfrak{R}$ such that $\cR\subseteq g$ and $\Prob(X \in \mathcal{R}) > 0$.

\begin{definition}
    \label{Con:Group-stable-region-wise Bayes--Massart}
    A group-stable region $\cR\in\mathfrak{R}_g$ satisfies the local Bayes--Massart condition if
    there exists a coefficient $\beta_{\cR}>0$ such that
    \[
    \Prob\Bigl( \bigl|\eta(X)-1/2\bigr| \geq \beta_{\cR} \Bigm| X\in\cR \Bigr) = 1\,.
    \]
\end{definition}

The risk-gap condition of Definition~\ref{Con:primal-risk gap}
is made slightly more local as follows.
The three modifications are first,
that the lower bound now explicitly depends on the group-stable region
$\cR$ considered, and second that the order of quantifiers
between $f$ and $\cR$ is reversed:
we look at a local condition on $f$,
namely, that it differs from the group-wise
Bayes classifier
$f^\star_g$ only locally on $\cR$.
The third modification is that
the condition is stated with an explicit
reference to $f^\star_g$ being the
group-wise Bayes classifier.

\begin{definition}
    \label{Con:Group-stable wise risk gap}
    Under Assumption~\ref{ass:gr-Bayes},
    a group-stable region $\cR\in\mathfrak{R}_g$ satisfies the local risk-gap condition if there exists a gap $\delta_{\cR}>0$ such that, for
    every $f\in\cF$ that differs from $f_g^\star$ on a subset of $\cR$ of
    positive $\Prob^X$--measure,
    \[
    \Exp\!\left[ \ell\bigl(f(X),Y\bigr) - \ell\bigl(f_g^\star(X),Y\bigr) \bigm| X\in\cR \right] \geq \delta_{\cR}\,.
    \]
\end{definition}

We can now combine the two local conditions into the following result.

\begin{theorem}
    \label{tm:mixed condition implies Bernstein}
    Under Assumption~\ref{ass:gr-Bayes} (the existence of a group-wise Bayes classifier),
    suppose that for each group-stable region $\cR\in\mathfrak{R}_g$, either the condition of
    Definition~\ref{Con:Group-stable-region-wise Bayes--Massart} holds with margin $\beta_{\cR}>0$
    (otherwise, set $\beta_{\cR} = 0$), or the condition of Definition~\ref{Con:Group-stable wise risk gap}
    holds with gap $\delta_{\cR}>0$ (otherwise, set $\delta_{\cR} = 0$).
    Then the group-wise Bernstein condition of Definition~\ref{Condition:group bernstein} holds for $g$, with
    constant
    \[
    \gamma_g = \max_{\cR\in\mathfrak{R}_g}  \ \min\!\left\{ \frac{1}{2\beta_\cR}, \, \frac{1}{\delta_\cR} \right\} > 0\,.
    \]
\end{theorem}

\begin{proof}\textbf{(sketched)}
    We slightly adapt the proof of Theorem~\ref{tm:primal risk gap implies Bernstein}
    and replace the first crude bounding therein by an application of Jensen's inequality
    followed by a total-probabilities argument:
    \begin{align*}
    & \E\!\left[\Biggl(\sum_{(f,g')\in \cF\times \cG_X}q_{(f,g')|\cF\times \cG_X}\ell\bigl(f(X),Y\bigr)-
        \ell\bigl(f^\star_g(X),Y\bigr)\Biggr)^{\!\!2}\,\middle|\,X\in g\right]\\
    & \leq \E\!\left[\sum_{(f,g')\in \cF\times \cG_X}q_{(f,g')|\cF\times \cG_X} \Bigl( \ell\bigl(f(X),Y\bigr)-
        \ell\bigl(f^\star_g(X),Y\bigr)\Bigr)^{2}\,\middle|\,X\in g\right] \\
    & = \frac{1}{\P(X\in g)}\sum_{\mathcal{R} \in \mathfrak{R}_g} \P(X\in \cR)\,\,\E\!\!\left[\sum_{(f,g')\in \cF\times \cG_X} \!\!\! q_{(f,g')|\cF\times \cG_X}
    \Bigl( \ell\bigl(f(X),Y\bigr)- \ell\bigl(f^\star_g(X),Y\bigr) \Bigr)^2 \,\middle|\,X\in \cR\right].
    \end{align*}
    Via another total-probabilities argument, it thus suffices to establish
    a Bernstein inequality on each group-stable region $\cR \in \mathfrak{R}_g$, of the following form,
    where we recall that $\cG(\cR)$ denotes the common value of the groups $\cG_x$ for $x \in \cR$:
    \begin{align}
    \nonumber
    \lefteqn{\, \hspace{-1.5cm} \forall \bq' \in \Delta\bigl(\cF\times\cG(\cR)\bigr),} \\
    \nonumber
    & \Exp\!\left[ \ind{X\in \cR} \Biggl( \sum_{(f,g')\in\cF\times\cG(\cR)} q'_{(f,g')} \ell\bigl(f(X),Y\bigr)-
    \ell\bigl(f^\star_g(X),Y\bigr) \Biggr)^{\!\! 2\,} \right] \\
    \label{eq:Mixed condition_aggregate}
    \leq \ & \gamma_{\cR} \,\, \Exp\!\left[ \ind{X\in \cR} \Biggl( \sum_{(f,g')\in\cF\times\cG(\cR)} q'_{(f,g')}
    \ell\bigl(f(X),Y\bigr)-\ell\bigl(f^\star_g(X),Y\bigr) \Biggr) \right],
    \end{align}
    where $\gamma_\cR = \min\bigl\{1/(2\beta_\cR),\,1/\delta_\cR\bigr\}$.

To do so, we note that
the proof of Theorem~\ref{tm:Bayes--Massart implies Bernstein} provides
an inequality that holds $\Prob^X$--almost surely
and is then integrated over $\{X \in g \}$ by
the tower rule; with the local condition,
we rather integrate over $\{X \in \cR \}$.
We then obtain~\eqref{eq:Mixed condition_aggregate} with a constant $1/(2\beta_\cR)$.

The local risk-gap condition of Definition~\ref{Con:Group-stable wise risk gap}
does not differ from
Definition~\ref{Con:primal-risk gap}
used in the
    proof of Theorem~\ref{tm:primal risk gap implies Bernstein}
in the final inequality on expectation gaps
and can therefore be used similarly;
the main difference between the two definitions
rather lies in the order of quantifiers,
i.e., fewer functions $f$ are considered
in Definition~\ref{Con:Group-stable wise risk gap} but their treatment is similar
in the proof.
We then obtain~\eqref{eq:Mixed condition_aggregate} with a constant $1/\delta_\cR$.
\end{proof}

\subsubsection{Beyond the global Bayes assumption}

The existence of a group-wise Bayes classifier $f_g^\star\in\cF$ allows
the whole group $g$ to be dealt with by a single expert pair $(f_g^*,g)$ in the reduction
described in Section~\ref{sec:reduc}.

When no such global Bayes classifier exists, but each group-stable region
$\cR\in\mathfrak{R}_g$ admits a Bayes classifier
$f_{\cR}^\star\in\cF$, one could instead introduce region-wise expert pairs $(f,\cR)$,
with $(f,\cR)$ active for $x \in \cX$ only when $x\in\cR$. The relevant comparator on
$\cR$ would then be $(f_{\cR}^\star,\cR)$, and the corresponding region-wise
fast-rate guarantees could be aggregated over
$\cR\in\mathfrak{R}_g$ to recover a guarantee for $g$. We omit the details.

The price of this refinement would be dealing with a larger experts class. Instead of
$\card{\cF}\card{\cG}$ group-wise expert pairs
as in Section~\ref{sec:reduc}, the more complex reduction involving pairs $(f,\cR)$
may involve up to $\card{\cF}\card{\mathfrak{R}}$ experts.
Since the group-stable regions are determined by the membership patterns
induced by $\cG$, in the worst case,
the cardinality of $\mathfrak{R}$
is of order~$2^{\card{\cG}}$.
Therefore, the logarithmic expert-complexity term could be bound to increase from
$\log\bigl( \card{\cF}\card{\cG} \bigr)$ to $\card{\cG}+\log\bigl( \card{\cF} \bigr)$.

In a nutshell, the global Bayes assumption is not essential for achieving fast rates, but
it offers a simpler treatment and a better
dependency on $\card{\cG}$.

\clearpage
\section{Lower bounds}
\label{app:LBs}

Several lower bounds are stated or alluded at in the main body of this article.
\begin{itemize}[itemsep=1pt, topsep=1pt]
\item The main one was sketched in Section~\ref{sec:sharp} and is formally stated and proved below in Appendix~\ref{app:sharp}.
\item Appendix~\ref{subsec:lower_ERM} shows formally that the simple approach consisting of partitioning the feature space into
group-stable regions, performing an empirical risk minimization on each such region, and aggregating them into a global predictor
performs poorly in the sense that its excess risk, while of the order of $1/\sqrt{n}$, suffers from a dependency in $\sqrt{\card{\cG}}$
instead of the logarithmic dependency achieved by \citet{pmlr-v162-tosh22a}.
\item Appendix~\ref{subsec:lower_log} proves that the thus achieved $\sqrt{\log |\cG|}$ dependency for the excess risk
cannot be improved in general.
\end{itemize}

\paragraph{Notation.}
In this appendix, it will be crucial to construct families of distributions for $(X,Y)$.
To do so, we will assume that the random vector
\[
\bigl( (X,Y), \, (X_1,Y_1), \, \ldots, \, (X_n,Y_n) \bigr) \qquad \mbox{is the identity function over} \quad (\cX \times \cY)^{n+1}\,,
\]
so that the underlying probability $\P$ on some probability space $(\Omega,\cF)$ and distributions $\mathbf{P}$
on $\cX \times \cY$ or $\mathbf{P}^{\otimes (n+1)}$ on $(\cX \times \cY)^{n+1}$ may now be identified.

For instance, in Section~\ref{sec:sharp}, we denoted by $\P_{\mathbf{P}}$ an underlying probability such that
the resulting distribution for $(X,Y)$ is $\mathbf{P}$. With the identification above,
we may rewrite the quantity of interest directly in terms of $\mathbf{P}$, with no reference to
any underlying probability, as
\begin{multline*}
\P_{\mathbf{P}} \!\left(\exists g \in \cG \ \text{\rm s.t. } \ \cE_{\mathbf{P},g} \bigl(\hat f_n\bigr)
    \geq \frac{1}{14\beta\sqrt{n \,\P_{\mathbf{P}}(X \in g)}} \right) \\
= \mathbf{P}^{\otimes n} \!\left(\exists g \in \cG \ \text{\rm s.t. } \ \cE_{\mathbf{P},g} \bigl(\hat f_n\bigr)
    \geq \frac{1}{14\beta\sqrt{n \,\mathbf{P}(X \in g)}} \right).
\end{multline*}
We denote by $\mathbf{E}$ the expectation associated with $\mathbf{P}$.

\subsection{Multi-group learning: $1/\sqrt{n}$ lower bounds on excess risk}
\label{app:sharp}
We recall that the main assumption considered to get fast rates
is the Bernstein condition of Definition~\ref{Condition:group bernstein},
which is implied by the risk-gap
condition over group-stable subregions of
Definition~\ref{Con:primal-risk gap}.

In this section, we show that (natural) relaxations
of these assumptions are not enough to get
fast rates. Namely, we consider a weaker Bernstein condition only holding for (Dirac masses at) functions $f \in \cF$ and groups $g$, not all distributions $\bq$ over $\cF \times \cG$, and complement it with risk-gap conditions over groups (but not over group-stable subregions thereof). The combination of both is insufficient in general to guarantee fast rates. We now formally state the two relaxed
conditions.

\begin{definition}[Expert-wise Bernstein condition]\label{Con:Expert-wise Bernstein}
    Distribution $\mathbf{P}$ and group $g$ satisfy the expert-wise Bernstein condition with constant $\gamma_g > 0$ if
    \begin{multline*}
    \forall f\in \cF,\qquad \mathbf{E}\biggl[\ind{X\in g}\Bigl( \ell\bigl(f(X),Y\bigr)-\ell\bigl(f^\star_g(X),Y\bigr) \Bigr)^{\! 2}\biggr] \\
    \leq \gamma_g \,\, \mathbf{E}\biggl[\ind{X\in g}\Bigl(\ell\bigl(f(X),Y\bigr)-\ell\bigl(f^\star_g(X),Y\bigr)\Bigr)\biggr]\,.
    \end{multline*}
\end{definition}

\begin{definition}[Risk gap over a group] \label{Con:Marginal risk-gap condition}
    Distribution $\mathbf{P}$ and group $g$ satisfy the risk gap condition if there exists a positive gap $\delta_g>0$ such that, for any $f\in \cF$
    that differs from $f^\star_g$ on a $\mathbf{P}$--non-negligible subset of $g$,
    \[
   \mathbf{E}\Bigl[ \ell\bigl(f(X),Y\bigr) -\ell\bigl(f^\star_g(X),Y\bigr)\,\Bigm|\,X\in g \Bigr]\geq \delta_g\,.
   \]
\end{definition}

The group-stable regions $g$ considered in Definition~\ref{Con:primal-risk gap}
form a partition of $g$ up to some $\mathbf{P}$--negligible subset, hence, by the law of total probabilities, the condition
of Definition~\ref{Con:primal-risk gap}
implies the one of Definition~\ref{Con:Marginal risk-gap condition}.

We are ready to state a slow-rate lower bound. The constant $\beta \geq 1$ considered therein
is only here for cosmetic reasons---to get the probability lower bound
as close as desired to~$1/2$.

\begin{theorem}[Slow-rate lower bound]
\label{tm:sharpness}
In binary classification $\cY = \{0,1\}$ with the binary loss,
with contexts $\cX = \{a,b,c\}$,
with a class $\cF = \{f_0,f_1\}$ containing two
constants functions $f_0 \equiv 0$ and $f_1 \equiv 1$, and groups $\cG = \bigl\{ \{a,c\}, \, \{b,c\} \bigr\}$,
there exists a subset $\cP$ of the distributions over $\cX \times \cY$
such that the following properties hold.

For each $\mathbf{P} \in \cP$ and every group $g\in\cG$, the pair
$(\mathbf{P},g)$ satisfies both the expert-wise Bernstein
condition of Definition~\ref{Con:Expert-wise Bernstein} with $\gamma_g = 6$ and the marginal risk-gap
condition of Definition~\ref{Con:Marginal risk-gap condition} with $\delta_g = 1/6$.

Nevertheless, for any constant $\beta \geq 1$,
\[
\forall n \geq 1, \qquad
    \inf_{\hat f_n}\ \max_{\mathbf{P} \in \cP} \ \mathbf{P}^{\otimes n}\!\left(\exists g \in \cG \ \text{\rm s.t. } \ \cE_{\mathbf{P},g} \bigl(\hat f_n\bigr)
    \geq \frac{1}{14\beta\sqrt{n \,\mathbf{P}(X \in g)}} \right) \geq \frac{1}{2} - \frac{0.22}{\beta}\,,
\]
where the infimum is taken over all, possibly randomized, classifiers
$\hat{f}_n : \cX \to [0,1]$ that are measurable functions of the training sample $D_n$.
In the inequality above, $\mathbf{P}^{\otimes n}$ denotes the
probability with respect to the training
sample $D_n = \bigl((X_1,Y_1),\ldots,(X_n,Y_n)\bigr)$ i.i.d. according to $\mathbf{P}$, and $\cE_{\mathbf{P},g}$ denotes the excess risk over group $g$ under $\mathbf{P}$.
\end{theorem}

A randomized classifier $f : \cX \to [0,1]$ draws independently at random
the label of a context $x \in \cX$ according to a Bernoulli distribution
with parameter $f(x)$. \\

\begin{proof}
    The class of distributions $\cP$ considered are indexed by $\epsilon \in [-1/8,\,1/8]$, and may be read
    in the following table, providing $p_{\epsilon}(x,y) = \mathbf{P}_\epsilon(Y=y, X=x)$
    in the first lines and the marginal distribution over $\cX$ in the final line:
    \newcolumntype{C}[1]{>{\centering\arraybackslash}m{#1}}
    \begin{center}
    \begin{tabular}{ccC{1.6cm}C{1.6cm}C{1.9cm}}
    \toprule
    $p_{\epsilon}(x,y)$ & & \multicolumn{3}{c}{$x$} \\
    \cmidrule(l){3-5}
    $y$ & & $a$ & $b$ & $c$ \\
    \midrule
    $0$ & & $1/4$ & $0$ & $(1/2 - \epsilon)/2$ \\
    $1$ & & $0$ & $1/4$ & $(1/2 + \epsilon)/2$ \\
    \midrule
    $\mathbf{P}_\epsilon(X=x)$ & & $1/4$ & $1/4$ & $1/2$ \\
    \bottomrule
    \end{tabular}
    \end{center}
    In particular, the conditional distributions with respect to
    single values in $\cX$ equal
    \[
    \mathbf{P}_{\epsilon}(Y = 1 \mid X = a) = 1 - \mathbf{P}_{\epsilon}(Y = 1 \mid X = b) = 0 \qquad\text{and}\qquad
    \mathbf{P}_{\epsilon}(Y = 1 \mid X = c) = \frac{1}{2} + \epsilon\,.
    \]
    Thus, all distributions $\mathbf{P}_{\epsilon}$ have the same marginal distributions over $X$ and the same conditional distributions, except
    for $\mathbf{P}_{\epsilon}(Y = \,\cdot\, \mid X = c)$.

    We denote by $g_{a,c} = \{a,c\}$ and $g_{b,c} = \{b,c\}$ the two groups considered in $\cG$.
    The conditional distributions with respect to these two groups
    are characterized by
    \begin{align*}
    \mathbf{P}_{\epsilon}(Y = 1 \mid X \in g_{a,c}) & =
    \frac{(1/2+\epsilon)/2}{1/4 + 0 + (1/2-\epsilon)/2 + (1/2+\epsilon)/2} =
    \frac{1+2\epsilon}{3} \\
    \mbox{and} \qquad
    \mathbf{P}_{\epsilon}(Y = 1 \mid X \in g_{b,c}) & = \frac{1/4 + (1/2+\epsilon)/2}{0 + 1/4 + (1/2-\epsilon)/2 + (1/2+\epsilon)/2}
    = \frac{2 + 2\epsilon}{3}\,.
    \end{align*}

    \noindent\emph{Group-wise risks and optimal classifiers.}
    Let $\risk_{\epsilon,g}(f)$ stand for the risk of a prediction function $f$ in a group $g$ under the distribution $(X, Y) \sim \mathbf{P}_{\epsilon}$.
    We have
    \[
    \risk_{\epsilon,g_{a,c}}(f_1)
    = \mathbf{P}_{\epsilon}\bigl(Y \neq f_1(X) \bigm| X \in g_{a,c}\bigr) = \mathbf{P}_{\epsilon}(Y = 0 \,|\, X \in g_{a,c}) = \frac{2 - 2\epsilon}{3}\,,
    \]
    and similarly,
    \[
    \risk_{\epsilon,g_{a,c}}(f_0) =
    \frac{1+2\epsilon}{3},\qquad
    \risk_{\epsilon,g_{b,c}}(f_1) = \frac{1-2\epsilon}{3}\,, \qquad
    \risk_{\epsilon,g_{b,c}}(f_0)
    = \frac{2 + 2\epsilon}{3}\,.
    \]
    Therefore, for all $\epsilon \in [-1/8, \, 1/8]$,
    \begin{align*}
    \risk_{\epsilon,g_{a,c}}(f_1) - \risk_{\epsilon,g_{a,c}}(f_0)
     & = \frac{1- 4\epsilon}{3} \geq \frac{1}{6} \\ \mbox{and} \qquad
    \risk_{\epsilon,g_{b,c}}(f_0) - \risk_{\epsilon,g_{b,c}}(f_1)
     & = \frac{1+ 4\epsilon}{3} \geq \frac{1}{6}\,.
    \end{align*}
    In particular, under all distributions $\mathbf{P}_\epsilon$,
    the group-wise optimal predictors are the same, namely,
    $f_{g_{a,c}}^\star = f_0$ and
    $f_{g_{b,c}}^\star = f_1$.
    \medskip

    \noindent\emph{Verification of the conditions.}
    As is clear from the excess risks computed above,
    the marginal risk-gap
    condition of Definition~\ref{Con:Marginal risk-gap condition} is satisfied for both
    distributions and both groups, with risk gaps larger
    than $1/6$.

    For the expert-wise Bernstein
    condition of Definition~\ref{Con:Expert-wise Bernstein},
    we proceed as follows.
    Almost surely,
    \[
    \Bigl( \ell\bigl(f_0(X),Y\bigr)-\ell\bigl(f_1(X),Y\bigr) \Bigr)^{\! 2} =
    \Bigl( \indb{f_0(X) \ne Y} - \indb{f_1(X) \ne Y} \Bigr)^2 = 1
    \]
    and for all $\epsilon \in [-1/8, \, 1/8]$,
    \begin{equation}
    \label{eq:prob34}
    \mathbf{P}_\epsilon(X \in g_{a,c}) =
    \mathbf{P}_\epsilon(X \in g_{b,c}) = 3/4\,.
    \end{equation}
    These facts, together with the $\delta_g \geq 1/6$ marginal risk gaps
    yield for $f\neq f^\star_g$:
    \begin{align*}
    & \Exp_\epsilon\biggl[ \ind{X\in g} \Bigl( \ell\bigl(f(X),Y\bigr)- \ell\bigl(f^\star_g(X),Y\bigr)\Bigr)^{\! 2} \biggr]
    = \frac{3}{4} \\
    \mbox{and} \qquad
    & \Exp_{\epsilon}\!\left[ \ind{X\in g} \Bigl( \ell\bigl(f(X),Y\bigr)-
    \ell\bigl(f^\star_g(X),Y\bigr) \Bigr) \right] \geq \frac{3}{4} \delta_g \,,
    \end{align*}
    i.e., the expert-wise Bernstein condition of Definition~\ref{Con:Expert-wise Bernstein} holds
    under both distributions and for both groups with $\gamma_{g_{a,c}} = \gamma_{g_{b,c}}=1/\delta_g=6$.
    \medskip

    \noindent\emph{Lower bounds on excess risks.}
    Consider a randomized predictor
    $\hat f : \cX \to [0,1]$ built on training data.
    Given $x \in \cX$, the prediction is drawn independently at
    random according to a Bernoulli distribution
    with parameter $\hat f(x)$.

    We first focus on group $g_{b,c}$. The oracle $f_1$ incurs $0$ loss when $x=b$, so that
    the excess risk of $\hat f$ may be lower bounded
    by what happens for $x=c$, with proper weight:
    \begin{align*}
        \cE_{\epsilon,g_{b,c}}\bigl(\hat f\bigr)
        \geq
        \underbrace{\mathbf{P}_{\epsilon}(X=c\,|\,X\in g_{b,c})}_{=2/3}\left(\risk_{\epsilon,\{c\}}\bigl(\hat{f}\bigr)-\risk_{\epsilon,\{c\}}(f_1)\right).
    \end{align*}
    Similarly to the risk calculations
    above,
    \smash{$\risk_ {\epsilon,\{c\}}(f_0)=1/2+\epsilon$
    and $\risk_{\epsilon,\{c\}}(f_1)=1/2-\epsilon$,}
    and, for the randomized predictor $\hat{f}$, by definition,
    \[
    \risk_{\epsilon,\{c\}}\bigl(\hat{f}\bigr)
    = \hat{f}(c)\, \risk_{\epsilon,\{c\}}(f_1)
    + \bigl( 1 - \hat{f}(c) \bigr)\, \risk_{\epsilon,\{c\}}(f_0)\,,
    \]
so that, collecting all (in)equalities,
we obtain the following lower bound,
of course only of interest for $\epsilon \in (0,\,1/8]$:
\[
\cE_{\epsilon,g_{b,c}}\bigl(\hat{f}\bigr)
\geq \frac{2}{3}\,
 \bigl( 1 - \hat{f}(c) \bigr) \, \Bigl( \risk_{\epsilon,\{c\}}(f_0)
 - \risk_{\epsilon,\{c\}}(f_1) \Bigr) \\
= \frac{4\epsilon}{3}\, \bigl( 1 - \hat{f}(c) \bigr) \,.
\]
A similar proof for group $g_{a,c}$ leads to a lower bound
only of interest for $\epsilon \in [-1/8,\,0)$:
    \[
        \cE_{\epsilon,g_{a,c}}\bigl(\hat f\bigr)
        \geq
        -\frac{4\epsilon}{3}\,\hat{f}(c).
    \]
The event $C = \bigl\{ \hat{f}(c) \geq 1/2 \bigr\}$
and its complement $C^c$ satisfy, in particular, \\ for all $\epsilon^+ \in (0,\,1/8]$
\begin{align*}
C = \bigl\{ \hat{f}(c) \geq 1/2 \bigr\} & \subseteq \left\{ \max_{g \in \cG} \cE_{-\epsilon^+,g}\bigl(\hat f\bigr) \geq \frac{2 \eps^+}{3} \right\} \\
\mbox{and} \qquad
C^c = \bigl\{ \hat{f}(c) < 1/2 \bigr\} & \subseteq \left\{ \max_{g \in \cG} \cE_{\epsilon^+,g}\bigl(\hat f\bigr) \geq \frac{2 \eps^+}{3} \right\}.
\end{align*}

    \noindent\emph{Conclusion by Pinsker's inequality.}
    We now fix $n \geq 1$ and consider a randomized predictor
    $\hat f_n$ built based on $D_n$.
    By taking expectations and using that the maximum over two points
    is lower bounded by the average, we established so far that for all
    $\epsilon^+ \in (0,\,1/8]$,
    \begin{align*}\label{eq:lower bound with TV}
\max_{\sigma\in\{-1,+1\}}\mathbf{P}^{\otimes n}_{\sigma\epsilon^+}\!\left(
            \max_{g\in\cG}
            \cE_{\sigma\epsilon^+,g}\bigl(\hat f_n\bigr) \geq \frac{2 \eps^+}{3}
        \right)
        \geq \frac{1}{2} \Bigl( \mathbf{P}^{\otimes n}_{-\epsilon^+}\bigl( \hat{f}_n(c) \geq 1/2 \bigr) +
        \mathbf{P}^{\otimes n}_{\epsilon^+}\bigl( \hat{f}_n(c) < 1/2 \bigr) \Bigr)\, .
    \end{align*}

Recall that $\hat{f}_n(c)$ is $D_n$--measurable.
  By Pinsker's inequality
and the additivity of the Kullback-Leibler divergence for product measures,
\begin{align*}
& \mathbf{P}^{\otimes n}_{-\epsilon^+}\bigl( \hat{f}_n(c) \geq 1/2 \bigr) +
        \mathbf{P}^{\otimes n}_{\epsilon^+}\bigl( \hat{f}_n(c) < 1/2 \bigr) \\
\geq \ &
\mathbf{P}^{\otimes n}_{\epsilon^+}\bigl( \hat{f}_n(c) \geq 1/2 \bigr) +
        \mathbf{P}^{\otimes n}_{\epsilon^+}\bigl( \hat{f}_n(c) < 1/2 \bigr)
- \sqrt{\frac{1}{2} \, \operatorname{KL} \bigl( \mathbf{P}^{\otimes n}_{-\epsilon^+} \,\Vert\, \mathbf{P}^{\otimes n}_{\epsilon^+} \bigr)}
    = 1 - \sqrt{\frac{n}{2} \, \operatorname{KL} (\mathbf{P}_{-\epsilon^+} \,\Vert\, \mathbf{P}_{\epsilon^+})}\,.
\end{align*}
We last show that
$\operatorname{KL} \bigl( \mathbf{P}_{-\epsilon^+} \,\Vert \, \mathbf{P}_{\epsilon^+} \bigr) \leq 6\,(\epsilon^+)^2$;
    together with the inequalities above, we then get
    \[
    \max_{\sigma\in\{-1,+1\}}\mathbf{P}^{\otimes n}_{\sigma\epsilon^+}\!\left( \max_{g\in\cG} \cE_{\sigma\epsilon^+,g}\bigl(\hat f_n\bigr) \geq \frac{2 \eps^+}{3}
    \right) \geq \frac{1}{2} - \sqrt{3n\,(\epsilon^+)^2}\,;
     \]
    the proof for the value $n$ fixed is concluded by taking $\epsilon^+ = 1/(8\beta\sqrt{n}) \in (0,\,1/8]$,
    by noting that, for the groups $g \in \cG$, the equalities~\eqref{eq:prob34} ensure that
    \[
    \frac{2 \eps^+}{3} = \frac{1}{12\beta\sqrt{n}} \geq \frac{1}{14\beta\sqrt{n \,\mathbf{P}(X \in g)}}\,,
    \]
    and by substituting $\sqrt{3/64} \leq 0.22$ to get the claimed lower bound on the probabilities.

    It remains to verify the claimed Kullback-Leibler bound. By definition of the Kullback-Leibler divergence and of the distributions
    $\mathbf{P}_{-\epsilon^+}$ and $\mathbf{P}_{\epsilon^+}$,
    \begin{multline*}
        \operatorname{KL} \bigl( \mathbf{P}_{-\epsilon^+} \,\Vert\, \mathbf{P}_{\epsilon^+} \bigr)
    = \sum_{x\in\{a,b,c\}} \sum_{y\in\{0,1\}} p_{-\epsilon^+}(x,y) \log\frac{p_{-\epsilon^+}(x,y)}{p_{\epsilon^+}(x,y)}
    = \sum_{y\in\{0,1\}} p_{-\epsilon^+}(c,y) \log\frac{p_{-\epsilon^+}(c,y)}{p_{\epsilon^+}(c,y)} \\[.2cm]
    = \frac{1}{2}\left(\frac{1}{2}+\epsilon^+\right)
    \ln \frac{(1/2+\epsilon^+)/2}{(1/2-\epsilon^+)/2}
    + \frac{1}{2}\left(\frac{1}{2}-\epsilon^+\right)
    \ln \frac{(1/2-\epsilon^+)/2}{(1/2+\epsilon^+)/2}
    = \epsilon^+ \log \frac{1 + 2\epsilon^+}{1 - 2\epsilon^+}
    \leq 6\, (\epsilon^+)^2\,,
    \end{multline*}
    where we used, for the inequality,
    $\log\bigl( ({1+x})/({1-x}) \bigr) \leq 3x$ for $x\in [0,1/4]$.
\end{proof}

\subsection{Lower bound for ERM over group-stable region}\label{subsec:lower_ERM}
The purpose of this section is to highlight the key difficulty of the problem. Namely, we show here that intuitively and conceptually simple algorithm, consisting in performing ERM minimization on every group-stable region must incur $\sqrt{|\cG|}$ multiplicative factor in the rate even if $|\cF| = 2$. Meanwhile, the approach that we target here and that was proposed by~\cite{pmlr-v162-tosh22a} admits merely $\sqrt{\log|\cG|}$ inflation in this case. We actually, in the subsequent subsection show that $\sqrt{\log|\cG|}$ is unavoidable in general.

Let us describe this baseline strategy. Recall the $\cG$--equivalence relation of Definition~\ref{def:eqv}, $x, x' \in \cX$ are $\cG$--equivalent when they belong to exactly the same groups, and the associated equivalence classes, called group-stable regions.
Form a finite partition $\mathfrak{R}$ of $\cX$. Given a training sample $D_n$, the strategy performs an independent ERM step on every group-stable region $\cR \in \mathfrak{R}$, it computes
\[\widehat{f}_{\cR} \in \arg\min_{f \in \cF} \widehat{\risk}_{\cR}(f)\,,\]
where $\widehat{\risk}_{\cR}(f)$ denotes the empirical risk of $f$ over the training points of $D_n$ falling in $\cR$. The final predictor $\widehat{f}_n : \cX \mapsto [0,1]$ is then obtained by gluing these ERMs together.

In the construction below, $\mathfrak{R}$ consists of the singletons of $\cX$ and $\cF$ consists of two constant functions, so that $\widehat{f}_n$ is nothing but the majority-vote classifier over the singletons.

\begin{theorem}
	Assume the setup of standard binary classification and that $|\cX| = \infty$. For any $2\leq m\leq n$, there is a group structure $\cG$ with $\card{\cG}=m$, a classifier class $\cF=\{f_0,f_1\}$, two distributions $\mathbf{P}_{+},\mathbf{P}_{-}$ over $\cX \times \cY$ such that for a classifier $\widehat{f}_n$ coinciding with the ERM over every group-stable region, we have:
	\[
	\forall g\in \cG,\quad
	\max_{\sigma \in \{\pm\}} \mathbf{P}_{\sigma}^{\otimes n}\left(\cE_{\sigma,g}(\widehat{f}_n )\geq \frac{1}{256}\sqrt{\frac{\card{\cG}}{n\mathbf{P}_\sigma(X\in g)}}\right)\geq \frac{1}{3},
	\]  where $\mathbf{P}_\sigma^{\otimes n}$ denotes the
	probability with respect to the training
	sample $((X_i,Y_i))_{i = 1}^n$ i.i.d.\ according to $\mathbf{P}_\sigma$, and $\cE_{\sigma,g}$ denotes the excess risk over group $g$ under $\mathbf{P}_\sigma$.
\end{theorem}
\begin{proof}
	Recall that $\cY=\{0,1\}$. The exact nature of the objects in $\cX$ does not matter and, since $|\cX|=\infty$, we can take $\cX=\{1,2,\ldots,m\}$ (the distributions below are supported on this subset, and the groups and classifiers are extended arbitrarily outside of it) and the group structure $\cG=\{g_1,g_2,\ldots,g_m\}$ with, for all $1 \leq j \leq m$, $g_j=\cX \setminus \{j\}$.
	The group-stable regions are exactly the $m$ singletons $\{1\},\{2\},\ldots,\{m\}$.

	Let $\epsilon\in(0,1/4]$ to be fixed later. For $\sigma \in\{\pm\}$, define the probability distribution
	\[
	\mathbf{P}_\sigma(X=x,Y=y)=\frac{1}{m}\biggl(\frac{1}{2}+\sigma\epsilon(2y-1)\biggr),
	\]
	which is simply described as follows: under $\mathbf{P}_\sigma$, the marginal distribution of $X$ is uniform, and $\mathbf{P}_{\sigma}(y=1|X=x)=\frac{1}{2}+\sigma\epsilon$. This immediately implies that for $1 \leq j \leq m$, the best benchmark predictor $f^\star_{g_j}$, minimizing the risk under $\mathbf{P}_{\sigma}$ over $\cF=\{f_0,f_1\}$, where $f_0 \equiv 0$ and $f_1 \equiv 1$ are the two constant classifiers, is given by $f^\star_{g_j} = f_{\frac{1+\sigma}{2}} \equiv \frac{1+\sigma}{2}$.

	Running ERM for binary loss over all group-stable regions results in the majority voting predictor $\widehat{f}_n$, that is
	\[
	\forall 1 \leq j \leq m,\quad \widehat{f}_n(j)=\ind{\sum_{i=1}^n \ind{X_i=j}(2Y_i-1)\geq 0}\,,
	\]

	Denoting by $\risk_{\sigma,g}$ the risk on group $g$ under $\mathbf{P}_{\sigma}$, the excess risk of $\widehat{f}_n$ for any group $g_j$, $1 \leq j \leq m$, is thus
	\begin{align}\label{eq:risk_to_f_hat}
		& \cE_{\sigma,g_j}(\widehat{f}_n) =\risk_{\sigma,g_j}(\widehat{f}_n)-\risk_{\sigma,g_j}(f_{\frac{1+\sigma}{2}}) \nonumber \\
		& = \frac{1}{\mathbf{P}_{\sigma}(X \in g_j)}  \Exp_{\sigma}\bigg[\ind{\widehat{f}_n(X) \neq Y, X \in g_j} - \ind{\frac{1+\sigma}{2} \neq Y, X \in g_j} \bigg] \nonumber \\
		& = \frac{m}{m-1} \sum_{k\neq j} \frac{1}{m} \ind{\widehat{f}_n(k)\neq \frac{1+\sigma}{2}}\left( \mathbf{P}_{\sigma}\left(Y = \frac{1+\sigma}{2} \, \bigg| \, X=k\right) - \mathbf{P}_{\sigma}\left(Y = \frac{1-\sigma}{2} \, \bigg| \, X=k\right)  \right) \nonumber \\
		&=\frac{2\epsilon}{m-1}\sum_{k\neq j}\ind{\widehat{f}_n(k)\neq \frac{1+\sigma}{2}}
	\end{align}

	Let us now define an auxiliary distribution. For a context $1 \leq k \leq m$, we consider the distribution $\mathbf{Q}_{\sigma,k}$ defined on $[m] \setminus \{ k \} \cup \{(k,0),(k,1)\} $ such that
	\[
	\mathbf{Q}_{\sigma,k}(j) = 1-1/m \mbox{ for } j \neq k, \quad \mathbf{Q}_{\sigma,k}((k,0)) = \frac{1/2-\sigma\eps}{m}, \mbox{ and } \mathbf{Q}_{\sigma,k}((k,1)) = \frac{1/2+\sigma\eps}{m} \, .
	\]

	Note that by definition of $\mathbf{P}_\sigma$ and $\mathbf{Q}_{\sigma,k}$, we have, for all $1 \leq k \leq m$,
	\[
	\mathbf{P}_\sigma^{\otimes n}(\widehat{f}_n(k)=1) = \mathbf{Q}_{\sigma,k}^{\otimes n}(\widehat{f}_n(k)=1),
	\] so that $\mathbf{P}_\sigma$ and $\mathbf{Q}_{\sigma,k}$ agree on an event of interest, while $\mathbf{Q}_{+,k}$ and $\mathbf{Q}_{-,k}$ are closer to one another than $\mathbf{P}_{+}$ and $\mathbf{P}_{-}$ are. This will be useful in the sequel.

	By Pinsker's inequality and the above remark, we have:
	\begin{align}\label{eq:two-point}
		\mathbf{P}^{\otimes n}_{-}(\widehat{f}_n(k)=1)+\mathbf{P}^{\otimes n}_{+}(\widehat{f}_n(k)=0)& = 1- (\mathbf{Q}^{\otimes n}_{+,k}(\widehat{f}_n(k)=1)- \mathbf{Q}^{\otimes n}_{-,k}(\widehat{f}_n(k)=1)) \nonumber \\
		& \geq 1-\sqrt{\frac{1}{2} \operatorname{KL}(\mathbf{Q}^{\otimes n}_{-,k}\|\mathbf{Q}^{\otimes n}_{+,k})} \nonumber \\
		& =1-\sqrt{\frac{n}{2}\operatorname{KL}(\mathbf{Q}_{-,k}\|\mathbf{Q}_{+,k})} \, .
	\end{align}
	$\operatorname{KL}(\mathbf{Q}_{-,k}\|\mathbf{Q}_{+,k})$ can be easily bounded as follows:
	\begin{align*}
	\operatorname{KL}(\mathbf{Q}_{-,k}\|\mathbf{Q}_{+,k})& =\frac{1/2+\eps}{m} \log \frac{1/2+\eps}{1/2-\eps} + \frac{1/2-\eps}{m} \log \frac{1/2-\eps}{1/2+\eps} \\
	& =  \frac{2\epsilon}{m}\log\frac{1+2\epsilon}{1-2\epsilon} = \frac{2\epsilon}{m}\log\left(1+\frac{4\eps}{1-2\epsilon} \right)\leq \frac{8\epsilon^2}{m(1-2\epsilon)}\leq \frac{16\epsilon^2}{m},
	\end{align*}
	where we used $\log(1+x)\leq x$ for the first inequality, and $\epsilon\leq 1/4$ for the second inequality.

	Substituting the above bound into~\eqref{eq:two-point}, we get:
	\begin{equation}\label{eq:lower_bound_sum_of_probas}
		\mathbf{P}^{\otimes n}_{-}(\widehat{f}_n(k)=1)+\mathbf{P}^{\otimes n}_{+}(\widehat{f}_n(k)=0) \geq 1-\epsilon\sqrt{\frac{8n}{m}}\geq 1-3\epsilon\sqrt{\frac{n}{m}}\, .
	\end{equation}

	For any group $g_j$, $1 \leq j \leq m$, \eqref{eq:lower_bound_sum_of_probas} together with \eqref{eq:risk_to_f_hat} gives
	\begin{multline*}
		\mathbf{E}^{\otimes n}_{+}\bigl[\cE_{+,g_j}(\widehat{f}_n)\bigr]+\mathbf{E}^{\otimes n}_{-}\bigl[\cE_{-,g_j}(\widehat{f}_n)\bigr]=\frac{2\epsilon}{m-1}\sum_{k\neq j}\left(\mathbf{P}^{\otimes n}_{-}(\widehat{f}_n(k)=1)+\mathbf{P}^{\otimes n}_{+}(\widehat{f}_n(k)=0) \right)\\
		\geq \frac{2\epsilon}{m-1}\sum_{k\neq j}\left(1-3\epsilon\sqrt{\frac{n}{m}}\right)=2\epsilon\left(1-3\epsilon\sqrt{\frac{n}{m}}\right)
	\end{multline*}
	Choosing $\epsilon=\frac{1}{16}\sqrt{\frac{m}{n}}\in(0,1/4]$, entails
	\[
	\max_{\sigma\in\{\pm\}} \mathbf{E}^{\otimes n}_{\sigma}\bigl[\cE_{\sigma,g_j}(\widehat{f}_n)\bigr]\geq \epsilon \left(1-3\epsilon\sqrt{\frac{n}{m}}\right) \geq \frac{13\epsilon}{16} \, .
	\]
	Moreover, by \eqref{eq:risk_to_f_hat}, the excess risk is at most $2 \epsilon$, so we have for $\sigma\in\{\pm\}$,
	\[
	\mathbf{E}^{\otimes n}_{\sigma}\bigl[\cE_{\sigma,g_j}(\widehat{f}_n)\bigr]\leq \frac{\epsilon}{8}+2\epsilon \mathbf{P}^{\otimes n}_{\sigma}\left(\cE_{\sigma,g_j}(\widehat{f}_n)\geq\frac{\epsilon}{8}\right) \, .
	\]
	Taking maximum on both sides, and using the previous inequality, we get
	\[
	\frac{\epsilon}{8}+2\epsilon \max_{\sigma\in\{\pm\}}\mathbf{P}^{\otimes n}_{\sigma}\left(\cE_{\sigma,g_j}(\widehat{f}_n)\geq\frac{\epsilon}{8}\right)\geq \frac{13\epsilon}{16},
	\] which finally implies
	\[
	\max_{\sigma\in\{\pm\}} \mathbf{P}^{\otimes n}_{\sigma}\left(\cE_{\sigma,g_j}(\widehat{f}_n)\geq\frac{\epsilon}{8}\right)\geq \frac{11}{32}\geq \frac{1}{3} \, .
	\]
	The proof is concluded since, using $\mathbf{P}(X\in g_j)\geq 1/2$, we get
	\[
	\frac{1}{256}\sqrt{\frac{\card{\cG}}{n\mathbf{P}_\sigma(X\in g_j)}} = \frac{ \epsilon}{16}\sqrt{\frac{1}{\mathbf{P}_\sigma(X\in g_j)}} \leq  \frac{\sqrt{2} \epsilon}{16} \leq \frac{\epsilon}{8} \, .
	\]
\end{proof}

\subsection{Lower bound for all multi-group predictors}\label{subsec:lower_log}

The next lower bound shows a rather intuitive fact that $\sqrt{\log|\cG|}$ inflation of the rate is unavoidable even when $|\cF| = 2$. We, as before, assume the standard binary classification settings. To the best of our knowledge, this fact was not explicitly proven in the prior literature, though it is pretty clear that it should hold. We include the statement and its proof mostly for the sake of the completeness of the picture.

\begin{theorem}
\label{th:31}
    Assume the setup of standard binary classification and   that $|\cX| = \infty$.
    For every integer $10 \leq m \leq n$ such that $n \geq \log(2m)$, there is a benchmark class $\cF = \{f_0, f_1\}$ a group class $\cG$ with $|\cG| = 2m$, and $m$ distributions $\mathbf{P}_1,\ldots,\mathbf{P}_m$ over $\cX \times \cY$ such that, for every data-driven and possibly randomized classifier $\hat{f}$,
    \[
    \max_{j = 1, \ldots, m}\mathbf{P}^{\otimes n}_{j}\left(\exists g \in \cG \quad\text{s.t.}\quad\cE_{\mathbf{P}_j,g}(\hat{f})\geq \frac{1}{64}\sqrt{\frac{\log (2|\cG|)}{n\,\mathbf{P}_j(X\in g)}}\right)\geq \frac14\,,
    \]
    where $\mathbf{P}^{\otimes n}_{j}$ denotes the probability with respect to an i.i.d. sample distributed according to $\mathbf{P}_j$ and $\cE_{\mathbf{P}_j,g}(\hat{f})$ is the excess risk under distribution $\mathbf{P}_j$ for the group $g \in \cG$.
\end{theorem}
    \begin{proof}
        The exact nature of the objects in $\cX$ does not matter and, since $|\cX|=\infty$, we can take $\cX = \{-1, 1\}^m$ (the distributions below are supported on this subset, and the groups and classifiers are extended arbitrarily outside of it). So, every $x = (x_1, \ldots, x_m) \in \cX$ is a binary vector.
        We now construct $2m$ different groups: for $j = 1,\ldots, m$ and $\sigma \in \{-1, +1\}$ let
        \[
        g_{j, \sigma}\eqdef \{(x_1, \ldots, x_m) \in \{-1, 1\}^m\,:\, x_j = \sigma\}\,.
        \]
        The group structure is then defined as $\cG = \{g_{j, \sigma} \,:\, j = 1, \ldots, m, \sigma \in \{-1, +1\}\}$, and the benchmark class consists of the two constant classifiers, $\cF = \{f_0, f_1\}$ with $f_0 \equiv 0$ and $f_1 \equiv 1$. We now build $m$ different joint distributions $\mathbf{P}_{j}$ of the couple $(X, Y) \in \{-1, +1\}^m \times \{0, 1\}$. All joint distributions $\mathbf{P}_{j}$ share the same marginal $\mu$ over $\cX$, which is simply the uniform distribution on it. They only differ in their conditional distribution of $Y$ given $X$. More precisely, for any $x = (x_1, \ldots, x_m) \in \cX$
        \begin{align}
            \mathbf{P}_j(Y = 1 \mid X = x) = \frac{1}{2} + \eps x_j\,,
        \end{align}
        for some $\eps < 1/2$ to be specified. Note that $\mathbf{P}_j(X \in g) = 1/2$ for all $j = 1, \ldots, m$ and $g \in \cG$.

        We now compute the excess risk on the two groups associated with the coordinate $j$ biased by $\mathbf{P}_j$. We claim that for any $f : \cX \mapsto [0, 1]$,
        \begin{align}
            \label{eq:excess-two-groups}
            \cE_{\mathbf{P}_j,g_{j,1}}(f) + \cE_{\mathbf{P}_j,g_{j,-1}}(f)= 2\eps\left(1-\int (2f(x)-1)x_j \mu(dx)\right)\,.
        \end{align}
        Indeed, under $\mathbf{P}_j$, one has $\mathbf{P}_j(Y = 1 \mid X_j = \sigma) = \frac{1}{2} + \eps\sigma$, so that the best benchmark predictor on $g_{j,\sigma}$ is the constant one equal to $\ind{\sigma = 1}$, with risk $\frac12 - \eps$. In addition, a direct computation of the risk of any $f$ gives
        \[
        \risk_{\mathbf{P}_j,g_{j,\sigma}}(f) = \frac{1}{2}+\eps\sigma-2\eps\sigma\int f(x)\,\mu\bigl(dx\,\bigm|\,x_j = \sigma\bigr)\,,
        \]
        where $\mu\bigl(\,\cdot\,\bigm|\,x_j = \sigma\bigr)$ denotes the uniform distribution on $g_{j,\sigma}$. The identity
        \[
        \int (2f(x)-1)x_j\, \mu(dx) = \int f(x)\,\mu\bigl(dx\,\bigm|\,x_j = 1\bigr) - \int f(x)\,\mu\bigl(dx\,\bigm|\,x_j = -1\bigr)\,,
        \]
        yields \eqref{eq:excess-two-groups}.

        Introduce the following shorthand notation $L_j(f) = \int (2f(x)-1)x_j\mu(dx)$. We now perform the classical reduction to a multiple testing problem~\citep[see e.g.,][Section 2.2]{tsybakov2009introduction}. For any data-driven and potentially randomized classifier $\hat{f}$, define
        \begin{align*}
            \hat{j} \in \argmax_{j = 1, \ldots, m} L_j(\hat{f})\,.
        \end{align*}
        Now, whenever $\hat{j} \neq j$, since $L_{\hat j}(\hat f) \geq L_j(\hat f)$, we have by \eqref{eq:excess-two-groups}
        \begin{align*}
             \frac{\cE_{\mathbf{P}_j,g_{j,1}}(\hat f) + \cE_{\mathbf{P}_j,g_{j,-1}}(\hat f)}{2} = \eps\left(1-L_{j}(\hat f)\right)
             \geq \eps\left(1 - \frac{L_{j}(\hat f) + L_{\hat{j}}(\hat f)}{2}\right)\,.
        \end{align*}
        Furthermore, we have by the Cauchy-Schwarz inequality, using that $(2\hat f(x)-1)^2 \leq 1$ and that the coordinates of $X$ are independent and centered under $\mu$,
        \begin{align*}
            L_{j}(\hat f) + L_{\hat{j}}(\hat f) &= \int (2\hat f(x)-1)(x_j + x_{\hat{j}}) \mu(dx)\\
            &\leq \sqrt{\int (2\hat f(x)-1)^2 \mu(dx)}\sqrt{\int (x_j + x_{\hat{j}})^2 \mu(dx)}
            \leq \sqrt{2}\,,
        \end{align*}
        where the last inequality also uses $\hat j \neq j$. That is,
        \begin{align*}
            \left\{\hat{j} \neq j \right\} \subseteq \left\{\max_{g \in \cG}\cE_{\mathbf{P}_j,g}(\hat{f}) \geq \eps(1 - \sqrt{2}/2) \right\}\,.
        \end{align*}
        Thus, setting $\mathbf{P}^{\otimes n}_j$ the joint distribution of $(X_i, Y_i)_{i = 1}^n$ i.i.d. from $\mathbf{P}_j$, we have from the above, for any data-driven classifier $\hat{f}$,
        \begin{align*}
           \max_{j = 1, \ldots, m} \mathbf{P}^{\otimes n}_j\left(\max_{g \in \cG}\cE_{\mathbf{P}_j,g}(\hat{f}) \geq \eps(1 - \sqrt{2}/2)\right)
           \geq
          \max_{j = 1, \ldots, m}\mathbf{P}^{\otimes n}_j(\hat{j} \neq j)\,.
        \end{align*}

        We conclude using standard Fano's inequality. Let $\mathbf{P}_0$ be such that under it $X$ is independent of $Y$ and both are distributed uniformly over their respective domains. Let $\mathbf{P}^{\otimes n}_0$ be the joint distribution of $n$ i.i.d. copies from $\mathbf{P}_0$. One has
        \begin{align*}
            \operatorname{KL}(\mathbf{P}^{\otimes n}_j, \mathbf{P}^{\otimes n}_0) \leq 4n\eps^2\,.
        \end{align*}
        Thus, Fano's inequality gives
        \begin{align*}
          \max_{j = 1, \ldots, m} \mathbf{P}^{\otimes n}_j(\hat{j} \neq j) \geq \frac{1}{m}\sum_{j=1}^m \mathbf{P}^{\otimes n}_j(\hat j \neq j)
          \geq 1 - \frac{4n\epsilon^2 + \log(2)}{\log m}\,.
        \end{align*}
        Choosing $\epsilon=\frac18 \sqrt{\tfrac{\log(2m)}{n}} \leq 1/8$, and using that $\log(2m)\leq 2\log(m)$ for $m\geq 2$ and $\log(2)\leq \frac{\log(m)}{3}$ for $m \geq 10$, we conclude that
        \begin{align*}
           \max_{j = 1, \ldots, m} \mathbf{P}^{\otimes n}_j\left(\max_{g \in \cG}\cE_{\mathbf{P}_j,g}(\hat{f}) \geq \epsilon(1 - \sqrt{2}/2)\right)
           \geq 1 - \frac{1}{8} - \frac{1}{3} = \frac{13}{24} \geq \frac{1}{4}\,.
        \end{align*}
        Since $1-\sqrt{2}/2\geq 1/4$, $\mathbf{P}_j(X\in g)=1/2$ for all $g\in\cG$, and $\log(2|\cG|)=\log(4m)\leq 2\log(2m)$, we get
        \begin{align*}
            \epsilon\bigl(1-\sqrt{2}/2\bigr) \geq \frac{\epsilon}{4}
            = \frac{1}{32}\sqrt{\frac{\log(2m)}{n}}
            \geq \frac{1}{64}\sqrt{\frac{2\log(4m)}{n}}
            = \frac{1}{64}\sqrt{\frac{\log(2|\cG|)}{n\,\mathbf{P}_j(X\in g)}}\,,
        \end{align*}
        which is the claimed rate.
\end{proof}

\clearpage
\section{Intermediate-rate results}
\label{sec:interm-rates}

In this section, we show that intermediate
rates are possible on the excess risk, between the $1/\sqrt{n}$ slow rate of Section~\ref{sec:reduc}
and the fast $1/n$ rate of Section~\ref{sec:fast rate}.
To do so, we generalize the Bernstein condition
of Definition~\ref{Condition:group bernstein} as follows, with $\alpha_g=1$ corresponding to
the statement of the latter and to fast rates,
and no condition, i.e., $\alpha_g=0$, corresponding to slow rates.
We then show that in classification, a Mammen-Tsybakov-type margin condition ensures that the generalized condition is satisfied.

\subsection{Generalized Bernstein condition}

\begin{definition}\label{Con:General Bernstein}
    Distribution $\Prob^{(X,Y)}$ and group $g\in \cG$ satisfy the general Bernstein condition with coefficient $b_g>0$ and exponent $\alpha_g\in [0,1]$ if, for any weight vector $\bq\in \Delta(\cF\times \cG)$, we have, with the notation of Section~\ref{sec:reduc}:
    \begin{multline*}
    \Exp\!\left[\ind{X\in g} \Biggl(\sum_{(f,g')\in \cF\times \cG_X}q_{(f,g')|\cF\times \cG_X}\,\ell\bigl(f(X),Y\bigr)-
        \ell\bigl(f^\star_g(X),Y\bigr) \Biggr)^{\! \! 2}  \right]\\
        \leq b_g \,\bigl(\Prob(X\in g)\bigr)^{\! 1-\alpha_g}\,\,\Exp\!\left[\ind{X\in g} \Biggl(\sum_{(f,g')\in \cF\times \cG_X}q_{(f,g')| \cF \times \cG_X}\,\ell\bigl(f(X),Y\bigr)-\ell\bigl(f^\star_g(X),Y\bigr)\Biggr)\right]^{\! \alpha_g}.
    \end{multline*}
\end{definition}

\begin{theoremsketch}\label{tm:general fast rate}
In the setting and under the assumptions
of Corollary~\ref{cor:main fast rate},
except for the replacement
of the Bernstein condition by the one of Definition~\ref{Con:General Bernstein},
with high probability, the excess risk
    \[
    \cE_g(\hat{\bq}_n)
    \quad \mbox{is at most of order} \quad
    \biggl(\frac{b_g }{n\, \Prob(X\in g)}\biggr)^{\!1/(2-\alpha_g)}\,,
    \]
up to (multiplicative and additive) logarithmic terms.
\end{theoremsketch}

\begin{proof}\textbf{sketch}~~We use the notation of the proofs of Theorem~\ref{tm:fast rate fixed point} and Corollary~\ref{cor:main fast rate} explain only where and how their arguments should be modified.
The key inequality, similar to \eqref{eq:filtered bernstein}, is that, now,
\[
    \E\bigl[W_{t,g}^2\,|\,D_{t-1}\bigr]
    \leq b_g \, \Prob(X\in g)^{1-\alpha_g}  \, \bigl(\E[W_{t,g}\,|\,D_{t-1}]\bigr)^{\alpha_g}\,.
\]
First, by Hölder's inequality with $p = 1/\alpha_g$ and $q = 1/(1-\alpha_g)$ when $\alpha_g \not\in \{0,1\}$, and nothing to be proved
when $\alpha_g \in \{0,1\}$,
\begin{align*}
\sum_{t=1}^n \E\bigl[W_{t,g}^2\,|\,D_{t-1}\bigr] & \leq b_g \, \Prob(X\in g)^{1-\alpha_g}\sum_{t=1}^n \bigl( \E[W_{t,g}\,|\,D_{t-1}] \bigr)^{\alpha_g} \\ & \leq b_g \, \bigl(n\, \Prob(X\in g)\bigr)^{1-\alpha_g}\Biggl( \underbrace{\sum_{t=1}^n\E[W_{t,g}\,|\,D_{t-1}]}_{ = n \, \Prob(X \in g) \, \cE_g(\widehat \bq_n)}\Biggr)^{\!\alpha_g}
= b_g\, n \,
\Prob(X\in g) \, \bigl( \cE_g(\widehat \bq_n) \bigr)^{\alpha_g} \,.
\end{align*}
Second, the same arguments as for~\eqref{eq:squared term in main fast rate} thus
show that with probability at least~$1-\delta$,
\[
\sum_{t=1}^n W_{t,g}^2 \leq 2b_g \, n\,\Prob(X\in g)\,\,\bigl(\cE_g(\hat{\bq}_n) \bigr)^{\alpha_g} + 8\log\biggl(\frac{\log_2 n}{\delta}\biggr)\,.
\]
We then substitute the rewriting of regret
offered in~\eqref{eq:regret in main result of fast rate 0} and the application of the Bernstein-Freedman inequality in~\eqref{eq:EWFreed} together with the discussion that follows it; for the convenience of the reader, we restate both here, the second holding with
probability at least $1-\delta$:
\begin{align*}
\regret_{n,(f^*_g,g)} & = \sum_{t=1}^n W_{t,g} \leq C_n\sqrt{\sum_{t=1}^n W_{t,g}^2}+C_n^2\,, \\
n\,\Prob(X\in g)\,\,\cE_g(\hat{\bq}_n) & =
\sum_{t=1}^n \E[W_{t,g} \mid D_{t-1}] \\
& \leq \sum_{t=1}^n W_{t,g} +
    4 \sqrt{\sum_{t=1}^n \E\bigl[W_{t,g}^2\mid D_{t-1}\bigr]\log\left(\frac{\log_2 n}{\delta}\right)}+4\log\left(\frac{\log_2 n}{\delta}\right).
\end{align*}
Combining all previous inequalities,
we get an inequality of the form
\[
n\,\Prob(X\in g)\,\,\cE_g(\hat{\bq}_n)
\leq L \sqrt{b_g \, n\,\Prob(X\in g)\,\,\bigl(\cE_g(\hat{\bq}_n) \bigr)^{\alpha_g}}
+ L'\,,
\]
where $L$ and $L'$ are terms that
depend on $C$ and on the logarithms
of other quantities, including the $\ln \ln n$ such that $C_n \leq C \ln \ln n$.

The proof is concluded by dividing both members by $n\,\Prob(X\in g)$ and by solving for $\cE_g(\hat{\bq}_n)$, using that
$x \leq a\,x^{\alpha/2}+y$, for $x,y,a \geq 0$ and $\alpha \in [0,1]$ entails $x \leq \xi_\alpha\bigl( a^{2/(2-\alpha)}+y \bigr)$,
for some constant $\xi_\alpha$ only depending on $\alpha$. We get
\[
\cE_g(\hat{\bq}_n)
\leq \xi_\alpha \left( \Biggl( \frac{L\sqrt{b_g}}{\sqrt{n\,\Prob(X\in g)}} \Biggr)^{2/(2-\alpha_g)}
+ \frac{L'}{n\,\Prob(X\in g)} \right),
\]
which is exactly the result claimed.
\end{proof}

\subsection{Mammen-Tsybakov-type margin condition
in classification}

We consider the setting of binary classification
as in Section~\ref{sec:Massart}.
The following condition relaxes the Bayes-Massart
condition of Definition~\ref{Con:Bayes--Massart}
(which corresponds to $\alpha_g=1$).

\begin{definition}\label{Con:Tsybakov}
     We say that group $g\in\cG$ satisfies the Bayes-Mammen-Tsybakov margin condition
    with exponent $\alpha_g \in [0,1)$, constant $C_g>0$, and threshold $u_g \in (0,1]$ if, first,
    \[
    \forall u\in[0,u_g],\qquad \Prob\bigl( |2\eta(X)-1|\leq u \bigm| X\in g \bigr)
    \leq C_g \, u^{\alpha_g/(1-\alpha_g)}\,,
    \]
    and second, there exists a classifier in $\cF$ that is a group-wise Bayes classifier for $g$.
\end{definition}

\begin{theorem}
\label{tm:Tsybakov-implies-generalized-Bernstein}
    If group $g\in\cG$ satisfies the Bayes-Mammen-Tsybakov margin condition
    of Definition~\ref{Con:Tsybakov}, then it satisfies the
    general Bernstein condition of Definition~\ref{Con:General Bernstein} with the same exponent $\alpha_g$.
\end{theorem}

\begin{proof}
We follow the proof of Theorem~\ref{tm:Bayes--Massart implies Bernstein}
up to its last two lines, where $|2\eta-1|$ is lower bounded. In particular, we extract that, under no specific condition,
for all distributions $\bq$, there exists a random variable $\rho_{\bq,g}$, bounded by $\ind{X \in g}$,  such that
\begin{align*}
    \E\!\left[ \ind{X\in g} \Biggl( \sum_{(f,g')\in\cF\times\cG_X} q_{(f,g')|\cF\times\cG_X} \ell\bigl(f(X),Y\bigr)-
    \ell\bigl(f^\star_g(X),Y\bigr) \Biggr)^{\! 2} \right]
    & \leq \E[\rho_{\bq,g}] \qquad\quad \mbox{and} \\
    \E\!\left[ \ind{X\in g} \Biggl( \sum_{(f,g')\in\cF\times\cG_X} q_{(f,g')|\cF\times\cG_X} \ell\bigl(f(X),Y\bigr)-
    \ell\bigl(f^\star_g(X),Y\bigr) \Biggr) \right]
    & = \E\Bigl[\bigl|2\eta(X)-1\bigr|\rho_{\bq,g}\Bigr]\,,
\end{align*}
so it suffices to show that
there exists a constant $b_g > 0$ independent of $\bq$ such that
\begin{align*}
b_g \,
\bigl(\Prob(X\in g)\bigr)^{\! 1-\alpha_g} \,
\biggl( \E\Bigl[\bigl|2\eta(X)-1\bigr|\rho_{\bq,g}\Bigr] \biggr)^{\alpha_g}
& \geq \E[\rho_{\bq,g}]\,, \\
\mbox{or equivalently,}
\qquad \qquad
b_g \, \biggl( \E\Bigl[\bigl|2\eta(X)-1\bigr|\rho_{\bq,g} \Bigm| X \in g \Bigr] \biggr)^{\alpha_g}
& \geq \E[\rho_{\bq,g} \mid X \in g]\,,
\end{align*}
which we do next, for $b_g = C_g + 1/u_g$.
The end of the proof follows a standard methodology.
For a threshold $\tau \leq u_g$
defined below, we have,
by total probabilities and by substituting the bound of Definition~\ref{Con:Tsybakov}, keeping in mind
that $\rho_{\bq,g} \leq \ind{X \in g}$,
\begin{align*}
\lefteqn{\E[\rho_{\bq,g} \mid X\in g]} \\
& = \E\Bigl[ \rho_{\bq,g} \,\, \indb{|2\eta(X)-1| > \tau} \Bigm| X\in g \Bigr]
  + \E\Bigl[ \rho_{\bq,g} \,\, \indb{|2\eta(X)-1| \leq \tau} \Bigm| X\in g\Bigr] \\
    & \leq \frac{1}{\tau} \,\, \E\Bigl[ \rho_{\bq,g} \,\, \bigl|2\eta(X)-1\bigr| \Bigm| X\in g \Bigr] +
    C_g \, \tau^{\alpha_g/(1-\alpha_g)} \\
& \leq \left( \frac{1}{u_g} + C_g \right) \biggl( \E\Bigl[ \rho_{\bq,g} \,\, \bigl|2\eta(X)-1\bigr| \Bigm| X\in g \Bigr] \biggr)^{\alpha_g}
\end{align*}
where we substituted
\quad $\smash{\displaystyle{\tau = u_g \, \biggl( \E\Bigl[ \rho_{\bq,g} \,\, \bigl|2\eta(X)-1\bigr| \Bigm|
X\in g \Bigr] \biggr)^{1-\alpha_g} \in [0,u_g]}}$.
\end{proof}

\clearpage
\section{Multi-group binary classification for infinite class}
\label{sec:infinite}

In this appendix, we consider binary classification with $\cY=\{0,1\}$ and \[\ell(f(X),Y)=\ind{f(X)\neq Y}\,,\] while allowing the group class $\cG$ and the benchmark class $\cF$ both to be infinite but with finite VC dimension.

\begin{assumption}
Throughout this appendix, we assume that every infimum and supremum over $\cG$ and $\cF$ is attained.
If this is not the case, then we rather consider approximate optimizers, up to some slack $\zeta$, which should then
be added to excess-risk bounds.
\end{assumption}

Some background theory is recalled in Appendix~\ref{app:infinite-classes} and may be useful to read before proceeding.

\subsection{Extension of the reduction of Section~\ref{sec:reduc} for infinite classes}

Let us first explain how to solve the infinite class problem by the reduction technique we explain in Section~\ref{sec:reduc}.

\paragraph{Solve the infinite class problem by splitting the samples.}\label{pg:infinite class strategy}
Given $2n$ i.i.d. observations, the strategy is to split them into two groups $D^{(1)},D^{(2)}$, each of size $n$, whose roles are as follows.

\begin{enumerate}[leftmargin=0.5cm]
    \item \emph{Block $D^{(1)}$: find group representatives and corresponding ERMs.} We use the first block of data to find a finite subset $\widehat{\cG} \subset \cG$ of groups with following properties: for any $g \in \cG$ there exists $g' \in \widehat{\cG}$ so that
    \begin{align}\label{eq:def_g_hat}
        \ind{X_i \in g} = \ind{X_i \in g'} \qquad\forall i = 1, \ldots, n\,.
    \end{align}
    Clearly, since data from block $D^{(1)}$ is finite, such a finite $\widehat{\cG} \subset \cG$ meeting the above properties exists. Among all such $\widehat{\cG}$, select one with smallest cardinality. We relate each group $g\in\cG$ to a fixed $\hat{g} = \hat{g}(g) \in \widehat{\cG}$ satisfying the above criterion \eqref{eq:def_g_hat}.

    We will also use the data from the first block $D^{(1)}$ to learn a group-wise ERM for every group in $\widehat{\cG}$.

    For every group $g \in \widehat{\cG}$, define for all $f \in \cF$,
    \[
     N_g^{(1)}=\sum_{i=1}^n\ind{X_i^{(1)}\in g}\quad\text{and}\quad
     \widehat\risk_g^{(1)}(f) = \frac{{\sum_{i=1}^n\ind{X_i^{(1)}\in g}
              \ell(f(X_i^{(1)}),Y_i^{(1)})}}{N_g^{(1)}}\,.
    \]
    Note that $N_g^{(1)}> 0$ due to minimality of $\widehat{\cG}$. For $g \in \widehat{\cG}$ define
    \[
     \widehat f_{g}\in\mathop{\rm argmin}_{f\in\cF}
           \widehat\risk_{g}^{(1)}(f)\,,
    \]
     Construct the finite benchmark class as $\widehat{\cF}\eqdef\{\widehat{f}_g \,:\, g\in \widehat{\cG}\}$.
     \item \emph{Block $D^{(2)}$: apply the finite sleeping-experts reduction.} Run the reduction introduced in Section~\ref{sec:reduc} for finite benchmark class
    $\widehat{\cF}$ and finite group class $\widehat{\cG}$ with samples in $D^{(2)}$. We denote the output randomized predictor as $\hat{\bq}_n$.
\end{enumerate}

The main theorem is an analogue of Corollary~\ref{cor:simple-reduc}, where the finiteness assumption is lifted.

\begin{theorem}\label{tm:Infinite class MGC}
Assume $n\geq2$ and that the benchmark class $\cF$ (resp. the collection of groups $\cG$) has finite VC dimension $d_{\cF} < \infty$ (resp. $d_{\cG} < \infty$).
Let $\widehat \bq_n$ be the randomized predictor produced by the construction
described in Section~\ref{pg:infinite class strategy}. If the sleeping experts algorithm satisfies the regret bound of Lemma~\ref{lm:regr-bd} with some constant $C$, then, for all $\delta\in(0,1)$, with probability at least $1-\delta$,
    \[
    \forall g\in \cG,\quad
    \cE_g(\hat{\bq}_n) \leq KC\sqrt{\frac{(d_{\cF}+d_{\cG})\log\bigl(n/\delta\bigr)}{n\, \Prob(X\in g)}},
    \]
    where $K$ is a universal constant.
\end{theorem}

\begin{proof}\textbf{of Theorem~\ref{tm:Infinite class MGC}.}
Throughout the proof, $K$ denotes a universal constant whose value may change from line to line and which may depend on the constant $C$ of Lemma~\ref{lm:regr-bd}.
The proof proceeds via a suitable decomposition that allows to make use of the introduced data-split.

\paragraph{Step 1: Decomposition of excess risk.}
Fix a group $g \in \cG$. By construction of $\widehat{\cG}$ and the definition of $\hat{g}$ in \eqref{eq:def_g_hat}, we have
\begin{align*}
    \ind{X_i^{(1)} \in g} = \ind{X_i^{(1)} \in \hat{g}} \qquad\text{for all }
 i = 1, \ldots, n\,.
 \end{align*}

 Consider the following decomposition: for $f^\star_{g} \in \argmin_{f \in \cF}\risk_g(f)$,
\begin{equation}
    \label{eq:infinite risk decom}
\begin{aligned}
    &\cE_g(\hat{\bq}_n) = \risk_g(\hat{\bq}_n) \pm \risk_{\hat{g}}(\hat{\bq}_n) \pm \risk_{\hat{g}}(\widehat{f}_{{\hat{g}}}) \pm \risk_{\hat g}(f^\star_{\hat g}) - \risk_g(f^\star_g)\\
    &\leq
    2\sup_{f \in \cF}\left|\risk_{{\hat{g}}}(f) - \risk_{{g}}(f)\right|
    {+}
     \bigl(\risk_{\hat{g}}(\hat{\bq}_n)-\risk_{\hat{g}}(\widehat{f}_{{\hat{g}}})\bigr)
     {+}
     \bigl(\risk_{\hat{g}}(\widehat{f}_{{\hat{g}}})-\risk_{{\hat{g}}}(f^\star_{\hat g})\bigr)\,,
\end{aligned}
\end{equation}
where the above inequality follows from the fact that $\risk_g(\cdot)$ is extended linearly when applied to a randomized estimator and from \[\inf_{f \in \cF} \risk_{g}(f) - \inf_{f \in \cF} \risk_{\hat g}(f) \leq \sup_{f \in \cF}\left|\risk_{{\hat{g}}}(f) - \risk_{{g}}(f)\right| \, . \]

The three terms are controlled separately. Intuitively, the first one is small since $g$ and $\hat{g}$ are indistinguishable on $D^{(1)}$; the second one is controlled conditionally on $D^{(1)}$ by Corollary~\ref{cor:simple-reduc} in the finite-class case; the third one corresponds to a classical excess risk-bound.

\paragraph{Step 2: Control of each risk-term.}
For the first term in~\eqref{eq:infinite risk decom}, note that for any events $A, A', B$:
\begin{align*}
    \left|\frac{\Prob(A \cap B)}{\Prob(A)} - \frac{\Prob(A' \cap B)}{\Prob(A')}\right|
    &\leq
    \left|\frac{\Prob(A \cap B)}{\Prob(A)} - \frac{\Prob(A' \cap B)}{\Prob(A)}\right| + \left|\frac{\Prob(A' \cap B)}{\Prob(A')} - \frac{\Prob(A' \cap B)}{\Prob(A)}\right| \\
    & \leq \frac{\Prob((A \cap B) \triangle (A' \cap B))}{\Prob(A)} + \Prob(A' \cap B) \frac{\left|\Prob(A) - \Prob(A')\right|}{\Prob(A)\Prob(A')}\\
    &\leq
    \frac{2 \Prob(A \triangle A')}{\Prob(A)}\,.
\end{align*}
Thus, we have for any $f \in \cF$,
\begin{align*}
    \left|\risk_g(f) - \risk_{\hat g}(f)\right|
    &=
    \left|\frac{\Prob(X \in g,\, f(X) \neq Y)}{\Prob(X \in g)} - \frac{\Prob(X \in \hat g, f(X) \neq Y)}{\Prob(X \in \hat g)}\right|
    \leq
    2\frac{\Prob(X\in g\triangle \hat{g})}{\Prob(X\in g)}\,.
\end{align*}
Appealing to Lemma~\ref{lm:zero empirical} applied to class $\cG\triangle \cG$ together with a control on the VC dimension of  $\cG\triangle \cG$ by that of $\cG$ (Lemma~\ref{lm:VC of symmetric difference}), we obtain that event $E_1$ on which
\begin{align}\label{eq:first term control}
\forall g\in {\cG},\quad \frac{\Prob(X\in g\triangle \hat{g})}{\Prob(X\in g)}\leq \frac{4}{n\Prob(X\in g)}\left(10ed_{\cG}\log\left(1+\frac{2en}{10d_{\cG}}\right)+\log(4/\delta)\right)
\end{align} holds with probability at least $1-\delta$.

For the second term in~\eqref{eq:infinite risk decom}, we define the event $E_2$ on which the following holds: $\forall g\in \cG$
\begin{align*}
    \risk_{\hat{g}}(\hat{\bq}_n) \leq \risk_{\hat{g}}(\widehat{f}_{{\hat{g}}})+ 3C \, \frac{\sqrt{\log\bigl(2|\widehat{\cF}|\card{\widehat{\cG}}/\delta\bigr) + \log\log n}}{\sqrt{n\, \Prob(X \in \hat g)}} + 3C \frac{\log\bigl(2|\widehat{\cF}|\card{\widehat{\cG}}/\delta\bigr) + \sqrt{\log\log n}}{n\, \Prob(X \in \hat g)} \,.
\end{align*}
The result on finite class of Corollary~\ref{cor:simple-reduc} guarantees that $\Prob(E_2\,|\, D^{(1)})\geq 1-\delta$ almost surely. Moreover, $|\widehat{\cF}| \leq |\widehat{\cG}|$ and Sauer-Shelah lemma implies that $\log|\widehat{\cG}| \leq e d_{\cG} \log(1 + n / d_{\cG})$, which finally gives that under $E_2$, $\forall g\in {\cG}$
\begin{equation}
    \label{eq:second term control}
\begin{aligned}
\risk_{\hat{g}}(\hat{\bq}_n)\leq\risk_{\hat{g}}(\widehat{f}_{{\hat{g}}}) &+
KC \sqrt{\frac{d_{\cG} \log\bigl(2(1 {+} n / d_{\cG})/\delta\bigr)}{n\, \Prob(X \in \hat g)}}
+
KC \frac{d_{\cG} \log\bigl(2(1 {+} n / d_{\cG})/\delta\bigr)}{n\, \Prob(X \in \hat g)}\,.
\end{aligned}
\end{equation}

For the third term in~\eqref{eq:infinite risk decom}, we consider the event $E_3$ on which the following holds:
\begin{align}\label{eq:third term control}
\forall g\in {\cG}, \quad \risk_{\hat{g}}(\widehat{f}_{{\hat{g}}}) \leq \risk_{\hat{g}}(f^*_{\hat{g}}) +
2K\sqrt{\frac{(d_{\cF}+d_{\cG})\log\!\bigl(e/\Prob(X\in {\hat{g}})\bigr)
		+\log(4/\delta)}{n\,\Prob(X\in {\hat{g}})}}
\end{align}
The standard empirical risk minimization decomposition
\begin{align*}
\risk_{\hat{g}}(\widehat{f}_{{\hat{g}}})-\risk_{\hat{g}}(f^*_{\hat{g}}) & = \risk_{{\hat{g}}}(\widehat{f}_{{\hat{g}}})-\widehat{\risk}_{\hat{g}}(\widehat{f}_{{\hat{g}}}) + \widehat{\risk}_{\hat{g}}(\widehat{f}_{\hat{g}})-\risk_{\hat{g}}(f^*_{\hat{g}}) \\ & \leq  \risk_{\hat{g}}(\widehat{f}_{\hat{g}})-\widehat{\risk}_{\hat{g}}(\widehat{f}_{\hat{g}}) + \widehat{\risk}_{\hat{g}}(f^*_{\hat{g}})-\risk_{\hat{g}}(f^*_{\hat{g}}),
\end{align*} together with Lemma~\ref{lm:group-conditional-concentration}, implies that $E_3$ holds with probability at least $1-\delta$.

\paragraph{Concluding the proof.} By a union bound, $E_1\cap E_2\cap E_3$ has probability at least $1-3\delta$. Two issues remain: \emph{i)} both~\eqref{eq:second term control} and~\eqref{eq:third term control} involve $\Prob(X \in \hat{g})$ instead of $\Prob(X \in g)$ \emph{ii)}~\eqref{eq:third term control} involves division by $\Prob(X \in \hat{g})$ under the logarithm, which we would like to transform into $\log(n)$-type of behaviour. Both issues are handled by case switching argument.

Choose a small universal $c > 0$. If for some $g\in\cG$
\[
 \frac{(d_{\cF} + d_{\cG})\log(n/\delta)}{n\Prob(X \in g)} > c\,,
\]
the claimed bound is trivial for this $g$ since $\cE_g(\hat{\bq}_n)\leq 1$ (enlarge $K$ if necessary). Otherwise, for sufficiently small $c > 0$,~\eqref{eq:first term control} ensures \[\frac{\Prob(X \in g)}{2}\leq\Prob(X\in\hat g)\leq \frac{3\Prob(X \in g)}{2}\,.\] Thus, every occurrence of $\Prob(X\in\hat g)$ can be replaced by $\Prob(X\in g)$.
Also, $ \tfrac{(d_{\cF} + d_{\cG})\log(n/\delta)}{n\Prob(X \in g)} \leq c$ implies $\log(e/\Prob(X\in\hat g))\leq K\log(n/\delta)$. This concludes the argument and the proof.
\end{proof}

\subsection{Reminder: concentration tools for infinite classes}
\label{app:infinite-classes}

In this final subsection, we recall and derive some basic concentration inequalities which are instrumental in the proof of Theorem~\ref{tm:Infinite class MGC}. Recall that we consider binary classification with feature space $\cX$, label space $\cY = \{0, 1\}$ and loss $ \ell(\widehat y,y) =  \ind{\widehat y\neq y}$; the classifier class $\cF$ and the group class $\cG$ both have finite VC dimensions. We denote $d_{\cF} := \operatorname{VC}(\cF)$ and $d_{\cG} := \operatorname{VC}(\cG)$.


We define the symmetric-difference class
\begin{align*}
    \cG \triangle \cG \eqdef \{g \triangle g' \, : \, g, g' \in \cG\}\,,
\end{align*}
where $g \triangle g'$ stands for the symmetric difference of two sets. We define the group-loss class as
\begin{align*}
    (\cG \times \cF)_\ell \eqdef \{(x, y) \mapsto \ind{x \in g,\, f(x) \neq y} \,:\, g \in \cG, f \in \cF\}\,.
\end{align*}
The next result gives a control on the VC-dimension of $\cG \triangle \cG$ and $(\cG \times \cF)_\ell$ based on  $d_{\cG}$ and $d_{\cF}$.
\begin{lemma}\label{lm:VC of symmetric difference}
    For any $\cG,\cF$ with finite VC dimension $d_{\cG}, d_{\cF}$ respectively, we have:
    \[
    \operatorname{VC}(\cG \triangle \cG)\leq 10d_{\cG} \qquad \mbox{and} \qquad\operatorname{VC}((\cG \times \cF)_\ell)\leq 5 \bigl(d_{\cG} + d_{\cF})\,.
    \]
\end{lemma}

\begin{proof}\textbf{of Lemma~\ref{lm:VC of symmetric difference}}

    \paragraph{Bound on $\operatorname{VC}(\cG \triangle \cG)$.}
    Assume by contradiction that $\operatorname{VC}(\cG \triangle \cG) > 10d_{\cG}$: there exist $n = 10d_{\cG} + 1$ points in $\cX$ that are shattered by $\cG \triangle \cG$. That is, using elementary set-theoretic identities and Sauer-Shelah lemma we deduce
    \begin{align*}
        2^{10d_{\cG} + 1} &= |\{G \cap \{x_1, \ldots, x_n\} \,:\, G \in \cG \triangle \cG\}|\\
            &= |\{(g \triangle g') \cap \{x_1, \ldots, x_n\} \,:\, g, g' \in \cG\}|\\
            &= |\{(g \cap \{x_1, \ldots, x_n\}) \triangle (g'\cap \{x_1, \ldots, x_n\})  \,:\, g, g' \in \cG\}|\\
            &\leq |\{(g \cap \{x_1, \ldots, x_n\}) \,:\, g \in \cG\}|^2 \leq \left(\frac{e (10d_{\cG} + 1)}{d_{\cG}}\right)^{2d_{\cG}} \leq (11e)^{2d_{\cG}},
    \end{align*} where we used the convention $0^0=1$ in the last bound.
    Note that $2^{10d_{\cG} + 1} = 2 (32)^{2d_{\cG}}$, yielding a contradiction.
    \paragraph{Bound on $\operatorname{VC}((\cG \times \cF)_\ell)$. } Assume by contradiction that $\operatorname{VC}(\cG \times \cF)_{\ell} > 5(d_{\cG}+d_{\cF})$: there exist $n = 5(d_{\cG}+d_{\cF}) + 1$ points in $\cX\times \cY$ that are shattered by $\cG \times \cF$. Recall that any binary classifier on $\cX\times \cY$ can be identified as a subset of $\cX\times \cY$. From Sauer-Shelah lemma we deduce
    \begin{align*}
    2^{5(d_{\cG}+d_{\cF}) + 1} &=
\left|\left\{\{i\in[n]:x_i\in g,\ f(x_i)\neq y_i\}
:\ g\in\cG,\ f\in\cF\right\}\right|\\
&=
\left|\left\{\{i\in[n]:x_i\in g\}
\cap\{i\in[n]:f(x_i)\neq y_i\}:\ g\in\cG,\ f\in\cF\right\}\right|
\\
&\leq
\left|\left\{\{i\in[n]:x_i\in g\}:g\in\cG\right\}\right|
\times
\left|\left\{\{i\in[n]:f(x_i)\neq y_i\}:f\in\cF\right\}\right|
\\
&=
\left|\left\{\{i\in[n]:x_i\in g\}:g\in\cG\right\}\right|
\times
\left|\left\{\{i\in[n]:f(x_i)=1\}:f\in\cF\right\}\right|
\\
    & \leq \left(\frac{e (5(d_{\cG}+d_{\cF}) + 1)}{d_{\cG}}\right)^{d_{\cG}} \times \left(\frac{e (5(d_{\cG}+d_{\cF}) + 1)}{d_{\cF}}\right)^{d_{\cF}}\\
    & \leq \frac{(e(5(d_{\cG}+d_{\cF}) + 1))^{d_{\cG}+d_{\cF}}}{\left(\frac{ (d_{\cG}+d_{\cF})}{2}\right)^{d_{\cG}+d_{\cF}}} \leq (11e)^{d_{\cG}+d_{\cF}},
    \end{align*} where we used $a^a b^b \geq ((a+b)/2)^{a+b}$ which follows by AM-GM inequality, and the convention $0^0=1$ in the last bound. Here again, $2^{5(d_{\cG}+d_{\cF}) + 1} = 2(32)^{d_{\cG}+d_{\cF}}$ yields a contradiction.
\end{proof}

In the following, we let $((X_i, Y_i))_{i = 1}^n,(X,Y)$ be an i.i.d. sample from $\Prob$, and use $N_g\eqdef\sum_{i=1}^n \ind{X_i\in g}$ to denote empirical mass of $g \in \cG$. We have following concentration result from~\citet[Eq.~(1.1)]{panchenko2002extensions}.

\begin{lemma}\label{lm:zero empirical}
 We have with probability at least $1-\delta$:
    \[
 \forall g\in\cG,\qquad
    \Prob(X\in g)-\frac{N_g}{n}
    \leq
    2\sqrt{
        \frac{\Prob(X\in g)}{n}
        \left(
            ed_{\cG}\log\biggl(1+\frac{2en}{d_{\cG}}\biggr)
            +\log\frac{4}{\delta}
        \right)
    }.
    \]
\end{lemma}

The next lemma presents a concentration result for conditional risk.  Let us first recall and specify the notation for true and empirical risks for the zero-one loss: for any $(f, g) \in \cF \times \cG$, set
\begin{align*}
&\risk_g(f) = \Prob(f(X) \neq Y \mid X \in g)\\
&\widehat{\risk}_g(f) = \begin{cases}
\frac{1}{N_g}\sum_{i=1}^n \ind{X_i\in g} \ind{f(X_i)\neq Y_i} & \mbox{if } N_g > 0 \\
0 & \mbox{if } N_g = 0 \, .
\end{cases}
\end{align*}

The next result is essentially a straightforward corollary of \citet{panchenko2003symmetrization}. We only provide the proof for the key adaptations needed to arrive to the statement.
\begin{lemma}
    \label{lm:group-conditional-concentration}
    There exists a universal constant $K$ such that, for every
$\delta\in(0,1)$, with probability at least $1-\delta$, simultaneously
for all $(g,f)\in\cG\times\cF$ with $\Prob(X\in g)>0$, it holds that
\[
 |\widehat\risk_g(f)-\risk_g(f)|
 \leq K\sqrt{\frac{
 (d_{\cF}+d_{\cG})\log\!\bigl(e/\Prob(X\in g)\bigr)
       +\log(4/\delta)}{n\,\Prob(X\in g)}}.
\]
\end{lemma}

\begin{proof}\textbf{of Lemma~\ref{lm:group-conditional-concentration}.}
For some universal, sufficiently large $C > 0$ and $(f, g) \in \cF \times \cG$, define the following event
\begin{align*}
    E_{(f, g)} = \bigg\{\max\biggl\{
  &\left|\frac{N_g}{n}-\Prob(X\in g)\right|,\,
  \left|\frac{N_g}{n}\widehat\risk_g(f)-\Prob(X\in g)\risk_g(f)\right|
 \biggr\}\\
 &\leq C\sqrt{\frac{
\bigl(\Prob(X\in g)+N_g/n\bigr)
\bigl[(d_{\cF}+d_{\cG})\log(e/\Prob(X\in g))
+\log(4/\delta)\bigr]}{n}}\bigg\}\,.
\end{align*}
Finally, let $E = \bigcap_{(f, g) \in \cF \times \cG} E_{(f, g)}$. We claim that on $E$, the resulting bound holds, then we show that $E$ holds with probability at least $1 - \delta$.

\paragraph{Step 1: $E$ implies our bound}
Assume that $E$ holds. Thus, in particular, for any $g \in \cG$
\begin{align*}
 \frac{N_g}{n}
 &\leq
 \Prob(X\in g)
 +C\sqrt{\frac{
 (\Prob(X\in g)+N_g/n)
[(d_{\cF}+d_{\cG})\log(e/\Prob(X\in g))+\log(4/\delta)]}{n}}\\
 &\leq
 \Prob(X\in g)
 +\frac12\left(\Prob(X\in g)+\frac{N_g}{n}\right)
 +\frac{C^2}{2}\frac{(d_{\cF}+d_{\cG})\log(e/\Prob(X\in g))+\log(4/\delta)}{n}.
\end{align*}
This gives after re-arranging
\[
 \frac{N_g}{n}\leq3\Prob(X\in g)
+C^2\frac{(d_{\cF}+d_{\cG})\log(e/\Prob(X\in g))+\log(4/\delta)}{n}\,.
\]

Now, using the definition of $E$ as well as the above bound on $\frac{N_g}{n}$ and the sub-additivity of the square root, we can write for every $(f, g) \in \cF \times \cG$ that
\begin{align*}
 |\widehat\risk_g&(f)-\risk_g(f)|
 \leq
 \frac{1}{\Prob(X\in g)}
 \biggl(
  \left|\frac{N_g}{n}\widehat\risk_g(f)
  -\Prob(X\in g)\risk_g(f)\right|
  +\left|\frac{N_g}{n}-\Prob(X\in g)\right|
 \biggr)\\
 &\leq 4C\sqrt{\frac{(d_{\cF}+d_{\cG})\log(e/\Prob(X\in g))+\log(4/\delta)}
 {n\Prob(X\in g)}}
+
2C^2\frac{(d_{\cF}+d_{\cG})\log(e/\Prob(X\in g))+\log(4/\delta)}
 {n\Prob(X\in g)}.
\end{align*}
Thus, on $E$, the claim holds for sufficiently large universal $C$.

\paragraph{Step 2: $E$ holds with high probability}
This result follows from Lemma~\ref{lem:panchenko}, which is a direct consequence of~\citet[Corollary~3]{panchenko2003symmetrization}. Indeed, one controls the probability of $E$, applying Lemma~\ref{lem:panchenko} to $\cG$ and $(\cG \times \cF)_{\ell}$ and using Lemma~\ref{lm:VC of symmetric difference} to control the VC dimension of the latter.
\end{proof}

\begin{lemma}
    \label{lem:panchenko}
    There exists a universal $K > 0$ such that for any $\cF$ with $\mathrm{VC}(\cF) = d < \infty$ and $\delta \in (0, 1)$, we have with probability at least $1 - \delta$ that
   \begin{align*}
       \forall f \in \cF \qquad \left|\frac{1}{n}\sum_{i = 1}^nf(X_i) - \Exp[f(X_1)]\right| \leq K \left(\sqrt{\frac{dV_n(f)\log(1 + 1/V_n(f))}{n}} + \sqrt{\frac{V_n(f)\log(1/\delta)}{n}}\right)\,,
   \end{align*}
   where
   \begin{align*}
       V_n(f) \eqdef 4\cdot\min\left\{1,\, \Exp[f(X_1)] + \frac{1}{n}\sum_{i = 1}^n f(X_i)\right\}\,.
   \end{align*}
\end{lemma}
\begin{proof}
    The result follows directly from~\citet[Corollary~3]{panchenko2003symmetrization}, and uses Haussler’s entropy bound for VC-classes and consequently bound the Dudley-like integral. We also use an obvious upper bound on the variance quantity appearing in the initial reference.
\end{proof}

\clearpage
\section{Background and literature review on online adversarial learning with sleeping experts}\label{app:Sleeping-experts liteerature review}

This appendix provides some additional details on the setting of and classic results for
online adversarial learning with sleeping experts, as described in Section~\ref{sec:aux-sleeping}.

\paragraph{On the setting.}
First, we disclose the slight restriction alluded at in Footnote~\ref{fn:sleeping}:
for the sake of simplicity, we assumed that the strategies do not access $\vec{I}_t$
for picking their distributions. This is because typical strategies,
and in particular, all strategies discussed below, maintain probability weights $\bp_t$
over all experts and actually do not use $\vec{I}_t$ to update $\bp_t$, but only
to normalize it into the distribution $\bp_{t|\cA_t}$ to be actually used at round $t$.
However, in the original descriptions of the settings of sleeping experts,
$\vec{I}_t$ is always revealed to the strategies at the beginning of round $t\geq 1$
and could, in principle, be exploited more heavily to make decisions.
To simplify exposition in the main body, we rather assumed that both $\vec{I}_t$ and $\vec{\ell}_t$
are revealed only at the end of round~$t$.

The so-called time-selection-function model extends the setting of sleeping experts by
allowing activity vectors that are not in a one-to-one correspondence
with experts; see, for instance, \citet{blum2007external}. We do not need this extension
for the present article.

\paragraph{Early bounds of order $\sqrt{T}$.}
\citet{Blu97} and \citet{freund1997using} respectively introduce and formalize the framework of specialized experts.
\citet{CV09} noted a generic reduction
from the setting of
sleeping experts to the standard setting of prediction with expert advice,
by imputation of losses: it suffices to consider that all inactive experts
suffer the same loss as the algorithm.
These references and techniques provide
$\sqrt{T}$ bounds on the sleeping regret, using the standard regret guarantees for prediction with expert advice (see, e.g.,
\citealp{littlestone1994weighted}, \citealp{freund1997decision}, \citealp{cesa2006prediction}).
However, a natural target in online adversarial learning with sleeping experts
is to bound the sleeping regret $R_{T,a}$ by something of the order of
the root of the number of non-zero terms in its definition, $\sqrt{\displaystyle{\sum_{t \leq T} I_{t,a}}}$,
as in Lemma~\ref{lm:regr-bd}.

\paragraph{Bounds of the form of Lemma~\ref{lm:regr-bd}.}
The landmark article by
\citet[Section 6]{blum2007external}
considers a version of the exponentially weighted average forecaster (a.k.a.\ the Hedge algorithm) with multiple fixed parameters $\beta_a = \e^{-\eta_a}$, where $\eta_a > 0$.
They obtain (see also \citealp[Appendix~B.2]{pmlr-v35-gaillard14}) regret bounds of the following form. For all $a \in \cA$,
\[
\regret_{T,a} \mbox{ is of order} \quad
\frac{\log\bigl(|\cA|\bigr)}{\eta_a} + \eta_a \sum_{t=1}^T I_{t,a} \ell_{t,a}\,,
\qquad \mbox{suggesting taking
$\eta_a$ of order}
\smash{\sqrt{\frac{\log\bigl(|\cA|\bigr)}{\displaystyle{\sum_{t \leq T} I_{t,a}}}}}
\medskip
\]
to bound each $\regret_{T,a}$
by a quantity of order \quad
$\smash{\displaystyle{\sqrt{\log\bigl(|\cA|\bigr)
\sum_{t \leq T} I_{t,a} \ell_{t,a}}}}$\,,
entailing the bound of Lemma~\ref{lm:regr-bd}. \medskip

This of course raises the question of
an adaptive choice of the learning rates
$\eta_{t,a}$ over time.
\citet[Section~6]{blum2007external} do not
discuss this issue in detail and simply refer
to~\citet{auer2002adaptive} for this online tuning. However, the techniques developed by
the latter work for adaptive choices of a single learning parameter.
\citet[Appendix~B.2]{pmlr-v35-gaillard14}
revisits this issue and suggests applying the techniques developed in the main body
to tune online per-expert learning parameters.
In particular,
\citet[Remark~17, referring to the proof of
Theorem~3]{pmlr-v35-gaillard14}
indicate that the desired per-expert
regret bound stated above holds, up to
an additional multiplicative factor of
order $\log \log T$, but omits the immediate details thereof.
For the sake of completeness, we provide
a simplified derivation below, not shooting
for the $\ell_{t,a}$ terms in the root,
i.e., corresponding to the bound of
Lemma~\ref{lm:regr-bd}. Our only contribution here is the modest improvement of $\log \log T$ to $\sqrt{\log \log T}$.

\paragraph{Bounds of the form of Lemma~\ref{lm:second-order regret}.}
The main sleeping-expert result by
\citet[Section~4]{pmlr-v35-gaillard14}
is actually a bound of the form
contemplated in Lemma~\ref{lm:second-order regret}, involving the squared differences between the losses of the algorithm
and the loss of the expert.
It turns out that \citet[Section~4]{pmlr-v35-gaillard14} state this bound without any specific application: the present
article provides, perhaps for the first time,
such an application. We again very slightly improve their bound, replacing $\log \log T$ by $\sqrt{\log \log T}$. We do not claim any significant novelty in either case of improvements as the key ideas and steps were already present in~\citet{pmlr-v35-gaillard14}.

\subsection{Derivation of a bound of the form of Lemma~\ref{lm:regr-bd}}
We recall that we provide no conceptually new result here but only implement, with modest improvements, \citet[Remark~17]{pmlr-v35-gaillard14}, itself implementing
the online tuning alluded at by
\citet[Section~6]{blum2007external}.
We however provide a concise proof and the $\sqrt{\log\log(T)}$ improvement:
the algorithm below generalizes
the one by \citet[Section~6]{blum2007external},
which corresponds to learning rates
constant over time.
It was introduced and analyzed by \citet{pmlr-v35-gaillard14},
under the name MLC-Hedge
(for Hedge with multiple learning rates for experts reporting confidences).
We keep the same algorithm and the main difference in our analysis is
the introduction of non-decreasing factors $D_t$
in the considered tunings, see Lemma~\ref{lm:Adaptvie Hedge regr_bis}.
These factors $D_t$ do not necessarily entail non-increasing sequences
of learning rates $\eta_{t, a}$ for each expert~$a$, which was a key
assumption in the analysis by \citet{pmlr-v35-gaillard14}.

\paragraph{Algorithm MLC-Hedge.}
Initialize $w_{0,a} = 1/|\cA|$.
Consider sequences $(\eta_{t,a})_{t \geq 1}$
of positive learning rates to be specified later. At round $t \geq 1$,
play with distributions
$\bp_t$ and $\bp_{t|\cA_t}$
with components
\[
p_{t,a} = \frac{\bigl( 1 - \e^{-\eta_{t,a}}\bigr) w_{t-1,a}}{\displaystyle{\sum_{a' \in \cA}
\bigl( 1 - \e^{-\eta_{t,a'}}\bigr) w_{t-1,a'}}} \qquad \mbox{and thus} \qquad
p_{t,a|\cA_t} = \frac{I_{t,a} \bigl( 1 - \e^{-\eta_{t,a}}\bigr) w_{t-1,a}}{\displaystyle{\sum_{a' \in \cA}
I_{t,a'} \bigl( 1 - \e^{-\eta_{t,a'}}\bigr) w_{t-1,a'}}}\,,
\]
then suffer the loss $\hl_t = \displaystyle{\sum_{a' \in \cA} p_{t,a'|\cA_t} \ell_{t,a'}}$, and update the weights as
\[
\tw_{t,a} = w_{t-1,a} \, \exp \Bigl( I_{t,a}\, \eta_{t,a} \bigl( \e^{-\eta_{t,a}} \hl_t - \ell_{t,a} \bigr) \Bigr) \qquad \mbox{and} \qquad
w_{t,a} = \bigl( \tw_{t,a} \bigr)^{\eta_{t+1,a}/\eta_{t,a}}\,.
\]
The analysis is decomposed into a general performance guarantee
given by Lemma~\ref{lm:Adaptvie Hedge regr_bis}
and a specific instantiation provided in Corollary~\ref{cor:root_loglog}.

\begin{lemma}\label{lm:Adaptvie Hedge regr_bis}
The MLC-Hedge strategy, run with learning rates
$\eta_{t, a} = D_t / \lambda_{t, a}$ obtained from non-decreasing sequences $D_t$ and $\lambda_{t, a}$, guarantees,
for $T \geq 1$:
\[
\forall a \in \cA, \qquad \regret_{T,a}
\leq \frac{\lambda_{T, a}\log \widetilde{W}_T}{D_T}
+ \frac{\lambda_{1, a}\log \bigl( |\cA| \bigr)}{D_1}
+ D_T\sum_{t=1}^T \frac{I_{t,a}}{\lambda_{t, a}}\,,
\]
where the $\widetilde{W}_t$ satisfy $\widetilde{W}_1 \in [0,1]$ and, for $t \geq 2$,
\begin{align*}
 \widetilde{W}_t \leq \left(\widetilde{W}_{t-1} + \sum_{a \in \cA}\left(1 - \frac{\lambda_{t-1, a}}{\lambda_{t, a}}\right)\right)^{D_t/D_{t-1}}\,.
\end{align*}
\end{lemma}

\begin{proof}
\citet[beginnings of the proofs of Theorems~3 and~16]{pmlr-v35-gaillard14}
note that by definitions of weights,
for all $t \geq 2$ and $a \in \cA$,
\[
\frac{\log w_{t,a}}{\eta_{t+1,a}}
= \frac{\log \tw_{t,a}}{\eta_{t,a}}
= \frac{\log w_{t-1,a}}{\eta_{t,a}}
+ I_{t,a} \bigl( \e^{-\eta_{t,a}} \hl_t - \ell_{t,a} \bigr)\,.
\]
Thus, by telescoping,
\[
\regret_{T,a} = \sum_{t=1}^T I_{t,a} (\hl_t - \ell_{t,a})
= \frac{\log \tw_{T,a}}{\eta_{T,a}}
+ \frac{\log \bigl( |\cA| \bigr)}{\eta_{1,a}}
+ \sum_{t=1}^T I_{t,a} \bigl( \underbrace{1 - \e^{-\eta_{t,a}}}_{\leq \eta_{t,a}} \bigr) \underbrace{\hl_t}_{\leq 1}\, .
\]
The proof continues by bounding any
$\tw_{T,a}$ by some $\widetilde{W}_T$ recursively defined, which we do now.

First, \citet[proof of Claim~12]{blum2007external},
see also
\citet[middle of the proof of Theorem~16]{pmlr-v35-gaillard14}, use the inequalities
$\e^{-\eta x} \leq 1 - (1-\e^{-\eta})x$ and
$\e^{\eta x} \leq 1 + (\e^\eta - 1) x$ for
$\eta > 0$ and $x \in [0,1]$ to get that
for all $t \geq 1$ and $a' \in \cA$,
\begin{align*}
\exp \Bigl( I_{t,a'}\, \eta_{t,a'} \bigl( \e^{-\eta_{t,a'}} \hl_t - \ell_{t,a'} \bigr) \Bigr)
& \leq \bigl( 1 + I_{t,a'}(\e^{\eta_{t,a'}}-1) \e^{-\eta_{t,a'}} \hl_t \bigr) \bigl(
1 - I_{t,a'}(1-\e^{-\eta_{t,a'}})\ell_{t,a'} \bigr) \\
& \leq 1 + I_{t,a'} (1-\e^{-\eta_{t,a'}}) \bigl( \hl_t - \ell_{t,a'} \bigr)\,,
\end{align*}
so that by definition of $\hl_t$
and given that $I_{t,a'} (1-\e^{-\eta_{t,a'}}) w_{t-1,a'}$ is proportional
to $p_{t,a'|\cA_t}$,
\[
\sum_{a' \in \cA} \tw_{t,a'}
\leq \sum_{a' \in \cA} w_{t-1,a'}
+ \sum_{a' \in \cA} w_{t-1,a'} \, I_{t,a'} (1-\e^{-\eta_{t,a'}}) \bigl( \hl_t - \ell_{t,a'} \bigr)
= \sum_{a' \in \cA} w_{t-1,a'} \,.
\]

Second, \citet[middle of the proof of Theorem~3]{pmlr-v35-gaillard14} use the convexity inequality
$x \leq x^\alpha +(\alpha-1)/\e$ for $x > 0$ and $\alpha \geq 1$,
and we rather resort below to
the concavity inequality $x^{\alpha} \leq 1 + \alpha(x-1) \leq x + (1-\alpha)$ for $\alpha \in (0,1]$
and $x > 0$.
Doing so, and since $D_t/D_{t-1} \geq 1$ and $\lambda_{t-1, a}/\lambda_{t, a} \leq 1$,
we have that, for $t \geq 2$,
\begin{multline*}
  \sum_{a' \in \cA} \tw_{t, a'}
  \leq
  \sum_{a' \in \cA} w_{t-1,a'}
  =
  \sum_{a' \in \cA} \bigl(\tw_{t-1,a'} \bigr)^{\eta_{t, a'}/\eta_{t-1,a'}}
  =
  \sum_{a' \in \cA} \bigl(\tw_{t-1,a'} \bigr)^{(D_t/D_{t-1}) (\lambda_{t-1, a'}/\lambda_{t, a'})} \\
  \leq
  \left(\sum_{a' \in \cA} \bigl(\tw_{t-1,a'} \bigr)^{\lambda_{t-1, a'}/\lambda_{t, a'}}\right)^{D_t/D_{t-1}}
  \leq
  \left(\sum_{a' \in \cA}\tw_{t-1,a'} + \sum_{a \in \cA}\left(1 - \frac{\lambda_{t-1, a'}}{\lambda_{t, a'}}\right)\right)^{D_t/D_{t-1}}\, ,
\end{multline*}
where we also used $\displaystyle{\Biggl(\sum_{a' \in \cA} x_{a'} \Biggr)^{\! b} \geq \sum_{a' \in \cA} x_{a'}^b}$ for $x_{a'} \geq 0$ and $b \geq 1$
in the second inequality.

This second step is somewhat new but the idea of potentially growing sequence of learning rates, has already appeared in parameter free online linear optimization literature in~\citet{mcmahan2012no} and pushed further in convex optimization context in~\citet{chzhen2023parameter}.

Collecting all elements, we proved so far
\[
\regret_{T,a} \leq \frac{\log \widetilde{W}_{T}}{\eta_{T,a}}
+ \frac{\log \bigl( |\cA| \bigr)}{\eta_{1,a}}
+ \sum_{t=1}^T I_{t,a} \eta_{t,a}\,, \qquad \mbox{where} \qquad
\widetilde{W}_T \eqdef \sum_{a' \in \cA} \tw_{T, a'}
\]
indeed satisfies the stated induction inequalities.
We conclude the proof by substituting the value of $\eta_{t,a}$ and the fact that $D_t$ is non-decreasing, and by
noting that from the first step above,
\[
0 \leq \widetilde{W}_1 = \sum_{a \in \cA} \tw_{1, a} \leq \sum_{a \in \cA} w_{0, a} = 1 \, .
\]
\vspace{-1cm}

\end{proof}

\begin{corollary}
  \label{cor:root_loglog}
    The choices
    \begin{align*}
      D_t = \sqrt{\log\!\left(|\cA| + \frac{1}{2}\sum_{a' \in \cA}\log\Biggl(1 + \sum_{t' \leq t-1} I_{t',a'}\Biggr)\right)}\qquad\text{and}\qquad\lambda_{t, a} = \sqrt{1 + \sum_{t' \leq t-1} I_{t',a}}
    \end{align*} lead to the following regret bound, for $T \geq 3$:
    \begin{align*}
      \forall a \in \cA, \qquad \regret_{T,a} \leq 5\,K_T \sqrt{\sum_{t \leq T} I_{t,a}} + \sqrt{\log \bigl( |\cA| \bigr)} \,,
    \end{align*}
    where $K_T = \sd{\sqrt{\log\bigl(|\cA|\bigr) + \log\bigl(1 + 0.5\log(1 + T)\bigr)}}$.
    This bound is indeed of the form stated in Lemma~\ref{lm:regr-bd}.
\end{corollary}

\begin{proof}
  Note that both $D_t$ and $\lambda_{t, a}$ are non decreasing. We bound the three terms of Lemma~\ref{lm:Adaptvie Hedge regr_bis}. We go from the simplest to the most difficult one. \medskip

  \noindent\emph{The second term.}
  One clearly has
  \begin{align*}
      \frac{\lambda_{1, a}\log \bigl( |\cA| \bigr)}{D_1} = \sqrt{\log \bigl( |\cA| \bigr)}\,.
  \end{align*}

  \noindent\emph{The third term.} Observing that $D_T \leq K_T$
  and using the classical bound of, e.g., \citet[Lemma~3.5]{auer2002adaptive} for the inequality below,
  \[
    \sum_{t=1}^T \frac{I_{t,a}}{\lambda_{t,a}} =
    \sum_{t=1}^T \frac{I_{t,a}}{\sqrt{1+\displaystyle{\sum_{t' \leq t-1} I_{t',a}}}}
    = \sum_{t=1}^T \frac{I_{t,a}}{\sqrt{\displaystyle{\sum_{t' \leq t} I_{t',a}}}}
    \leq 2 \sqrt{\sum_{t=1}^T I_{t,a}}\,,
  \]
  we deduce that the third term is bounded by
  \begin{align*}
    D_T\sum_{t=1}^T \frac{I_{t,a}}{\lambda_{t, a}} \leq 2\, K_T \sqrt{\sum_{t=1}^T I_{t,a}}\,.
  \end{align*}

  \noindent\emph{The first term.} This is the term that requires some work.
  We will show that
  \begin{equation}
  \label{eq:logwidetildeW}
      \log \widetilde{W}_T \leq 2D_T^2\,,
  \end{equation}
  which will imply that the first term is bounded by
  \begin{align*}
    2\,D_T\lambda_{T, a} \leq 2\,K_T \sqrt{1 + \sum_{t \leq T-1} I_{t,a}} \leq
    2\,K_T \sqrt{2 \sum_{t=1}^T I_{t,a}} \,,
  \end{align*}
  as we can assume, with no loss of generality, that at least one $I_{t,a} = 1$ (otherwise, the regret is null).

  Before showing the bound~\eqref{eq:logwidetildeW}, let us record two useful observations that only concern the behaviour of sequences $\lambda_{t, a}$ and $D_t$ and their relations.
  First of all, with our choices of $\lambda_{t, a}$ and $D_t$, we have, thanks to the simple inequality $1-u \leq \log(1/u)$ for $u \in (0,1)$: for $T \geq 2$,
  \begin{align}
    \label{eq:loglog1}
    \sum_{a \in \cA}\left(1 - \frac{\lambda_{T-1, a}}{\lambda_{T, a}}\right)
    \leq \frac{1}{2}\sum_{a \in \cA} \log \frac{1 + \sum_{t \leq T-1} I_{t,a}}{1 + \sum_{t \leq T-2} I_{t,a}}
    = \exp(D_T^2) - \exp(D_{T-1}^2)\,.
  \end{align}
  Secondly, using the exact definition of $D_T$, the fact that $I_{T,a} \leq 1$, the basic inequality $\log(2 + x) \leq \log(1 + x) + \log(2)$ for $x \geq 0$ and the observation that $\exp(D_{T-1}^2) \geq |\cA|$, we get
  \begin{align*}
    \exp(D_T^2)
    &\leq |\cA| + \frac{1}{2}\sum_{a \in \cA}\log\left(1 + \sum_{t \leq T-2} I_{t,a} + 1\right)\\
    &\leq |\cA| + \frac{1}{2}\sum_{a \in \cA}\left(\log\left(1 + \sum_{t \leq T-2} I_{t,a}\right) + \log(2)\right)\\
    &= \exp(D_{T-1}^2) + |\cA|\frac{\log(2)}{2}\\
    &\leq  \exp(D_{T-1}^2)\left(1 + \frac{\log(2)}{2}\right)\,.
  \end{align*}
  In particular, the above implies that for any $T \geq 2$, we have
  \begin{align}
    \label{eq:loglog2}
    \log(2) \leq \log(|\cA|) \leq D_{T-1}^2 \leq D_{T}^2 \leq D_{T-1}^2 + \log\left(1 + \frac{\log(2)}{2}\right)\,.
  \end{align}

  We are now ready to prove that $\log \widetilde{W}_T \leq 2D_T^2$, which goes by induction. For $T = 1$, $\widetilde{W}_1  \in [0,1]$ thus $\log \widetilde{W}_1 \leq 0 \leq 2D_1^2$.

  Now, assume $\log \widetilde{W}_{T-1} \leq 2D_{T-1}^2$. Using the recursive inequality for $\widetilde{W}_T$, our first observation in~\eqref{eq:loglog1} and the fact that $\log(x + y) \leq \log(x) + y/x$, we have
  \begin{align*}
    \log \widetilde{W}_T
    & \leq \frac{D_T}{D_{T-1}} \log \left(\exp(2D_{T-1}^2) +  \sum_{a \in \cA}\left(1 - \frac{\lambda_{T-1, a}}{\lambda_{T, a}}\right)\right)\\
    &\leq
    \frac{D_T}{D_{T-1}} \Bigl( 2D_{T-1}^2 + \left(\exp(D_T^2) - \exp(D_{T-1}^2)\right) \exp(-2D_{T-1}^2)\Bigr)\,.
  \end{align*}
  To conclude, we claim that
  \begin{align*}
    \exp(D_T^2 -2D_{T-1}^2) \leq \exp(-D_{T-1}^2) +  2 D_{T-1}(D_T - D_{T-1})\,,
  \end{align*}
  which, substituted in the above, exactly yields $\log \widetilde{W}_T \leq 2D_T^2$.
  This is done by a deterministic inequality that we derive in Lemma~\ref{lem:for_recursion} in conjunction with the observation in~\eqref{eq:loglog2}.
\end{proof}

\subsection{Reminder of a bound of the form of Lemma~\ref{lm:second-order regret}}

The following algorithm and corresponding regret guarantee are merely modest improvements of~\citet[Algorithm~2, Corollary~4, and Section~4]{pmlr-v35-gaillard14}
to the case of online adversarial learning with sleeping experts.
The algorithm is called Adapt-ML-Prod for sleeping experts.

\paragraph{Algorithm Adapt-ML-Prod.}
Initialize $w_{0,a}=1/\card{\cA}$. Consider sequences $(\eta_{t,a})_{t\geq 1}$ of positive learning rates to be specified later. At round $t\geq 1$, play with the distributions $\bp_t$ and $\bp_{t|\cA_t}$ with components
\[
p_{t,a}=\frac{\eta_{t,a}w_{t-1,a}}{\displaystyle{\sum_{a'\in \cA}\eta_{t,a'}w_{t-1,a'}}} \qquad \mbox{thus} \qquad p_{t,a|\cA_t}=\frac{I_{t,a}\eta_{t,a}w_{t-1,a}}{\displaystyle{\sum_{a'\in \cA}I_{t,a'}\eta_{t,a'}w_{t-1,a'}}}\,,
\]
suffer the loss $\hl_t = \displaystyle{\sum_{a' \in \cA} p_{t,a'|\cA_t} \ell_{t,a'}}$, and update the weights as
\[
\tw_{t,a} = w_{t-1,a}\bigl(1 + \eta_{t,a}\, I_{t,a}(\hl_t - \ell_{t,a})\bigr)
\qquad\mbox{and}\qquad
w_{t,a} = \bigl(\tw_{t,a}\bigr)^{\eta_{t+1,a}/\eta_{t,a}}\,,
\]
The following risk guarantee is an adaptation of~\citet[Corollary~4]{pmlr-v35-gaillard14}, to account for the non-monotonic $\eta_{t,a}$.

\begin{lemma}\label{lem:ml-prod}
The strategy above guarantees, when
\[\eta_{t, a} = \min\left\{\frac{1}{2},\, \frac{D_t}{\lambda_{t, a}}\right\}\,,\] with non-decreasing $D_t$ and $\lambda_{t, a}$,
\[
\forall a \in \cA, \qquad \regret_{T,a} = \sum_{t=1}^T I_{t,a} (\hl_t - \ell_{t,a})
\leq \frac{\log \widetilde{W}_T}{\eta_{T,a}}
+ \frac{\log \bigl( |\cA| \bigr)}{\eta_{1,a}}
+ D_T\sum_{t=1}^T \frac{I_{t,a}(\hl_t - \ell_{t,a})^2}{\lambda_{t, a}}\,,
\]
where $\widetilde{W}_T$ satisfies, for all $a \in \cA$,
\begin{align*}
  \widetilde{W}_T \leq \left(\widetilde{W}_{T-1} +  \sum_{a \in \cA}\Bigl(1 - \frac{\lambda_{T-1, a}}{\lambda_{T, a}}\Bigr) + |\cA|\frac{D_T - D_{T-1}}{D_T}\right)^{\frac{D_T}{D_{T-1}}}
  \quad\text{and}\quad \widetilde{W}_1 \leq 1\,.
\end{align*}
\end{lemma}

\begin{proof}
As in the first inequality of~\citet[Theorem~3]{pmlr-v35-gaillard14}, which does not yet rely on any monotonicity but uses the fact that $\eta_{t,a} \in (0,1/2]$, we have for any $t \geq 1$,
\[
  \log w_{t,a} \geq \frac{\eta_{t+1,a}}{\eta_{1,a}}\log w_{0,a} + \eta_{t+1,a}\sum_{s=1}^{t}\Bigl(I_{s,a}\bigl(\hl_s - \ell_{s,a}\bigr) - \eta_{s,a}\, I_{s,a}\bigl(\hl_s - \ell_{s,a}\bigr)^{2}\Bigr)\,.
\]
Re-arranging and using $\log w_{T,a} = \tfrac{\eta_{T+1,a}}{\eta_{T,a}} \log \tw_{T,a}$, we get
\[
  \regret_{T,a} = \sum_{t=1}^T I_{t,a} \bigl(\hl_t - \ell_{t,a}\bigr)
  \leq \frac{\log \tw_{T,a}}{\eta_{T,a}}
  + \frac{\log \bigl( |\cA| \bigr)}{\eta_{1,a}}
  + \sum_{t=1}^T \eta_{t,a}\, I_{t,a}\bigl(\hl_t - \ell_{t,a}\bigr)^{2}\,.
\]
This is essentially the ML-Prod counterpart of the starting point that we have had for the MLC-Hedge algorithm. As with the former, we deal with the $\tw_{T,a}$.
Set
\[
  \widetilde{W}_t \eqdef \sum_{a' \in \cA} \tw_{t,a'} = \sum_{a' \in \cA} w_{t-1,a'}\Bigl(1 + \eta_{t,a'}\, I_{t,a'}\bigl(\hl_t - \ell_{t,a'}\bigr)\Bigr) = \sum_{a' \in \cA}w_{t-1,a'} \quad\text{in particular}\quad \widetilde{W}_1 = 1\,.
\]
Since $\frac{D_T}{D_{T-1}} \geq 1$, we may apply the inequality $(\sum_{a} x_a)^b \geq \sum_{a} x_a^b$ for $b \geq 1$ deployed in Lemma~\ref{lm:Adaptvie Hedge regr_bis}, in the form
\begin{align*}
  \widetilde{W}_T = \sum_{a \in \cA} \bigl(\tw_{T-1, a}\bigr)^{\frac{\eta_{T, a}}{\eta_{T-1,a}}}
  = \sum_{a \in \cA} \left(\tw_{T-1, a}^{\,\frac{\eta_{T, a}}{\eta_{T-1,a}}\cdot\frac{D_{T-1}}{D_T}}\right)^{\frac{D_T}{D_{T-1}}}
  \leq \left(\sum_{a \in \cA} \bigl(\tw_{T-1, a}\bigr)^{\frac{\eta_{T, a}}{\eta_{T-1,a}}\cdot\frac{D_{T-1}}{D_T}}\right)^{\frac{D_{T}}{D_{T-1}}}\,.
\end{align*}
We then intend to use the concavity inequality $x^{\alpha} \leq 1 + \alpha(x-1) \leq x + (1-\alpha)$ for $\alpha \in (0,1]$ and $x > 0$. To this end it suffices to show that the exponent of $\tw_{T-1, a}$ is at most one.
Indeed, monotonicity of $\lambda_{T, a}$ and $D_T$ implies
\begin{align*}
  \eta_{T, a} D_{T-1}
= \min\left\{\frac{D_{T-1}}{2},\, \frac{D_{T-1} D_T}{\lambda_{T, a}}\right\}
\leq \min\left\{\frac{D_T}{2},\, \frac{D_{T-1} D_T}{\lambda_{T-1, a}}\right\}
= \eta_{T-1, a} D_T\,.
\end{align*}
Thus, we can use the recalled concavity inequality, to get
\begin{align*}
  \bigl(\tw_{T-1, a}\bigr)^{\frac{\eta_{T, a}}{\eta_{T-1,a}}\cdot\frac{D_{T-1}}{D_T}} \leq \tw_{T-1, a} + \left(1 - \frac{\eta_{T, a}}{\eta_{T-1,a}}\cdot\frac{D_{T-1}}{D_T}\right)\,.
\end{align*}
To conclude the bound claimed, it suffices to show that
\begin{align*}
   \left(1 - \frac{\eta_{T, a}}{\eta_{T-1,a}}\cdot\frac{D_{T-1}}{D_T}\right) \leq \left(1 - \frac{\lambda_{T-1, a}}{\lambda_{T,a}}\right) +  \left(1 - \frac{D_{T-1}}{D_T}\right)\,.
\end{align*}
This is the consequence of the following observation $\eta_{T, a}\,\lambda_{T, a} = \min\{\lambda_{T, a}/2,\, D_T\}$ is non-decreasing in $T$, so that $\frac{\eta_{T, a}}{\eta_{T-1, a}} \geq \frac{\lambda_{T-1, a}}{\lambda_{T, a}}$. Therefore, using $xy \geq x + y - 1$ for $x, y \leq 1$,
\begin{align*}
  \frac{\eta_{T, a}}{\eta_{T-1,a}}\cdot\frac{D_{T-1}}{D_T}
  \geq \frac{D_{T-1}}{D_T}\cdot\frac{\lambda_{T-1, a}}{\lambda_{T, a}}
  \geq \frac{D_{T-1}}{D_T} + \frac{\lambda_{T-1, a}}{\lambda_{T, a}} - 1,
\end{align*}
which is the claimed inequality after rearranging. This concludes the claimed bound.
\end{proof}

\begin{corollary}
\label{cor:root_loglog_Gaillard}
    Let
    \begin{align*}
      &D_t = \sqrt{\log\!\left(|\cA| + \frac{1}{2}\sum_{a' \in \cA}\log\Biggl(1 + \sum_{s \leq t-1} I_{s,a'}\bigl(\hl_s - \ell_{s,a'}\bigr)^{2}\Biggr)\right)}\,,\\
      &\lambda_{t, a} = \sqrt{1 + \sum_{s \leq t-1} I_{s,a}\bigl(\hl_s - \ell_{s,a}\bigr)^{2}}\,.
    \end{align*}
    Then, Adapt-ML-Prod guarantees, for any $T \geq 3$,
    \begin{align*}
      \forall a \in \cA, \qquad \regret_{T,a} \leq 4\,K_T \sqrt{1+\sum_{t=1}^T I_{t,a}\bigl(\hl_t - \ell_{t,a}\bigr)^{2}}
      + 4\,K_T^2 + 2 \log \bigl( |\cA| \bigr) \,, \smallskip
    \end{align*}
    where $K_T = \sd{\sqrt{\log\bigl(|\cA|\bigr) + \log\bigl(1 + 0.5\log(1 + T)\bigr)}}$.
    This bound is indeed of the form stated in Lemma~\ref{lm:second-order regret}.
\end{corollary}

This should be compared with \citet[Corollary~4]{pmlr-v35-gaillard14} where an extra $\sqrt{\log\log T}$ factor is present. \medskip

\begin{proof}
  We bound the three terms of Lemma~\ref{lem:ml-prod}, using that clearly, $D_T \leq K_T$. \medskip

  \noindent\emph{The second term.} The second term equals $2 \log |\cA|$. \medskip

  \noindent\emph{The third term.} Using $\eta_{t,a} \leq D_t/\lambda_{t,a} \leq D_T/\lambda_{t,a}$ and \citet[Lemma~14]{pmlr-v35-gaillard14} with $f(x) = 1/\sqrt{x}$,
  \begin{multline*}
    \sum_{t=1}^T \eta_{t,a}\, I_{t,a}\bigl(\hl_t - \ell_{t,a}\bigr)^{2}
    \leq D_T \sum_{t=1}^T \frac{I_{t,a}\bigl(\hl_t - \ell_{t,a}\bigr)^{2}}{\sqrt{1 + \sum_{s \leq t-1} I_{s,a}\bigl(\hl_s - \ell_{s,a}\bigr)^{2}}} \\
    \leq 2 D_T \sqrt{1 + \sum_{s \leq T} I_{s,a}\bigl(\hl_s - \ell_{s,a}\bigr)^{2}}\,.
  \end{multline*}

  \noindent\emph{The first term.}
  We will, as in Corollary~\ref{cor:root_loglog}, by induction on $T$, show that $\log \widetilde{W}_T \leq 2 D_T^2$.
  First, observe that both~\eqref{eq:loglog1} and~\eqref{eq:loglog2} hold verbatim as they have only relied on boundedness of $I_{t, a}$ between $[0, 1]$. So, in this proof, we refer to~\eqref{eq:loglog1} and~\eqref{eq:loglog2}, assuming the form and not the exact expressions.

  The proof goes by very similar induction, with extra correction that is still handled by Lemma~\ref{lem:for_recursion}.
  Equation~\eqref{eq:loglog1} and the inductive assumption gets us
 \begin{align*}
  \log \widetilde{W}_{T}
  \leq
  \frac{D_{T}}{D_{T-1}}\log\left(\exp(2D_{T-1}^2) + \exp(D_T^2) - \exp(D_{T-1}^2) + |\cA|\left(1 - \frac{D_{T-1}}{D_T}\right)\right)\,.
 \end{align*}
 Using as before $\log(x + y) \leq \log(x) + y/x$ and also $|\cA| \leq \exp(D_{T-1}^2)$ we deduce
 \begin{align*}
  \log \widetilde{W}_{T}
  \leq D_T\left(2 D_{T-1} + \frac{\exp(D_T^2) - \exp(D_{T-1}^2)}{D_{T-1}}\exp(-2D_{T-1}^2) + \frac{D_T - D_{T-1}}{D_{T-1}D_T}\, \exp(- D_{T-1}^2)\right)\,.
 \end{align*}
 To conclude, we claim
 \begin{align*}
  \exp(D_T^2 - 2 D_{T-1}^2) \leq \exp(- D_{T-1}^2)\, \frac{D_{T-1}}{D_T} + 2 D_{T-1} \bigl(D_T - D_{T-1}\bigr)\,,
 \end{align*}
 which, combined with the penultimate display, yields $\log \widetilde{W}_T \leq 2 D_T^2$. This is the precise statement of Lemma~\ref{lem:for_recursion} that needs to be used with~\eqref{eq:loglog2}. To conclude bounding the first term, we note that
 \begin{align*}
  \frac{\log\widetilde{W}_T}{\eta_{T, a}} = \max\left\{2\log\widetilde{W}_T,\, \frac{\lambda_{T, a}}{D_T}\log\widetilde{W}_T\right\} \leq 4D_T^2 + 2\lambda_{T,a}D_T\,.
 \end{align*}
 \vspace{-1cm}

\end{proof}

\subsection{A simple and useful inequality}

Here, we collect and prove some basic inequalities that are used in this section.

\begin{lemma}
  \label{lem:for_recursion}
For all $x \geq \log 2$ and $y \in [x, x + \log(1 + \tfrac{\log 2}{2})]$,
 \begin{align*}
    \exp(y - 2x) \leq \exp(-x)\sqrt{\frac{x}{y}} + 2 \sqrt{x}\,\bigl(\sqrt{y} - \sqrt{x}\bigr) \leq \exp(-x) + 2 \sqrt{x}\,\bigl(\sqrt{y} - \sqrt{x}\bigr)\,.
 \end{align*}
\end{lemma}
\begin{proof}
The second inequality is trivial as $y \geq x$.
For the first one, by convexity of the exponential, we have
  \begin{align*}
    \exp(y - 2x)
    &\leq \frac{x + \log\bigl(1 + \tfrac{\log 2}{2}\bigr) - y}{\log\bigl(1 + \tfrac{\log 2}{2}\bigr)}\,\exp(-x) + \frac{y - x}{\log\bigl(1 + \tfrac{\log 2}{2}\bigr)}\,\exp\Bigl(\log\bigl(1 + \tfrac{\log 2}{2}\bigr) - x\Bigr)\\
    &= \exp(-x) + \frac{\log(2)/2}{\log\bigl(1 + \tfrac{\log 2}{2}\bigr)}\,\exp(-x)\,\bigl(\sqrt{y} + \sqrt{x}\bigr)\bigl(\sqrt{y} - \sqrt{x}\bigr)\,.
 \end{align*}
 We now claim that under the stated assumptions we have
 \begin{align*}
  \frac{\log(2)/2}{\log\bigl(1 + \tfrac{\log 2}{2}\bigr)}\,\exp(-x)\,\bigl(\sqrt{y} + \sqrt{x}\bigr) \leq 2\sqrt{x} - \frac{\exp(-x)}{\sqrt{y}}\,.
 \end{align*}
 This yields the desired bound.
 To prove this, define
 \begin{align*}
    f_\delta(x) = \frac{\log(2)/2}{\log\bigl(1 + \tfrac{\log 2}{2}\bigr)}\,\exp(-x)\,\bigl(\sqrt{x + \delta} + \sqrt{x}\bigr)  \quad\text{and}\quad h_\delta(x) = 2\sqrt{x} - \frac{\exp(-x)}{\sqrt{x + \delta}}\,.
 \end{align*}
 We will show that for any $\delta \in [0, \log(1 + \log(2)/2)]$, $x \mapsto f_\delta(x)$ is decreasing on $x \geq \log(2)$ and $x \mapsto h_\delta(x)$ is increasing on the same range.
 The inequality thus follows from
 \begin{align*}
    f_\delta(\log(2))
    &=
    \frac{\log(2)}{4\log\bigl(1 + \tfrac{\log 2}{2}\bigr)}\bigl(\sqrt{\log(2) + \delta} + \sqrt{\log(2)}\bigr)\\
    &\leq
    \frac{\log(2)}{4\log\bigl(1 + \tfrac{\log 2}{2}\bigr)}\bigl(\sqrt{\log(2) + \log(1 + \log(2)/2)} + \sqrt{\log(2)}\bigr)\\
    &\leq
    2\sqrt{\log(2)} - \frac{1}{2\sqrt{\log(2)}}\\
    &\leq 2\sqrt{\log(2)} - \frac{1}{2\sqrt{\log(2) + \delta}} = h_\delta(\log(2))
 \end{align*}
 where the second and last inequalities follow from $\delta \in [0, \log(1 + \log(2)/2)]$ and where the third one is a constant inequality that we checked numerically.
 Thus, it remains to prove monotonicity of $f_{\delta}$ and $h_{\delta}$.

 For the first one, using $\delta \geq 0$ in the first inequality and $x \geq \log 2 > 1/2$ in the last one,
 \begin{align*}
    f_\delta'(x)
    &= \frac{\log(2)/2}{\log\bigl(1 + \tfrac{\log 2}{2}\bigr)}\, \exp(-x)
    \left(\frac{1}{2\sqrt{x + \delta}} + \frac{1}{2\sqrt{x}} - \sqrt{x + \delta} - \sqrt{x}\right)\\
    &\leq \frac{\log(2)/2}{\log\bigl(1 + \tfrac{\log 2}{2}\bigr)}\, \exp(-x)\left(\frac{1}{\sqrt{x}} - 2\sqrt{x}\right) < 0\,.
 \end{align*}
 For the second one,
 \begin{align*}
    h_\delta'(x)
    = \frac{1}{\sqrt{x}} + \frac{\exp(-x)}{\sqrt{x + \delta}} + \frac{\exp(-x)}{2 (x + \delta)^{3/2}} > 0\,.
 \end{align*}
 This concludes the proof.
\end{proof}

\clearpage
\vskip 0.2in
\bibliography{biblio}

\end{document}